\documentclass[]{article}

\usepackage[utf8]{inputenc}
\usepackage[T1]{fontenc}
\usepackage[margin=1in]{geometry}
\usepackage{times}
\usepackage{microtype}
\usepackage{xcolor}
\usepackage{color}
\usepackage{url}
\usepackage{booktabs}
\usepackage{nicefrac}
\usepackage{enumitem}
\setlist[enumerate]{leftmargin=1.5em}
\setlist[itemize]{leftmargin=1.5em}

\usepackage{amsmath,amssymb,amsfonts,amsthm,mathtools,mathrsfs}
\usepackage{bm}
\usepackage{dsfont}
\usepackage{bbold}
\usepackage{physics}
\usepackage{xfrac}
\usepackage{tensor}
\numberwithin{equation}{section}
\allowdisplaybreaks

\usepackage{graphicx}
\usepackage{float}
\usepackage[style=base]{caption}
\usepackage{subcaption}
\usepackage{wrapfig}
\usepackage{adjustbox}
\usepackage{algorithm}
\usepackage{algorithmic}
\usepackage{tikz}
\usepackage{tikz-cd}
\usepackage{pgfplots}
\usepackage[most]{tcolorbox}
\usepackage{thmtools}
\usepackage{mdframed}
\tcbuselibrary{theorems}
\usetikzlibrary{decorations.pathmorphing}
\usetikzlibrary{decorations.pathreplacing, calligraphy}
\usetikzlibrary{arrows.meta, shapes.geometric, calc}
\usepgfplotslibrary{fillbetween, groupplots}
\pgfplotsset{compat=1.18}
\graphicspath{{img/}}

\usepackage[authoryear,round]{natbib}
\usepackage{hyperref}
\hypersetup{
    colorlinks = true,
    citecolor = blue,
    linkcolor = blue,
    urlcolor  = blue,
    breaklinks = true,
    pdftitle={Scaling Laws from Sequential Feature Recovery: A Solvable Hierarchical Model},
    pdfauthor={Hugo Tabanelli, Yatin Dandi, Florent Krzakala}
}
\usepackage[capitalize,nameinlink]{cleveref}

\crefname{theorem}{Theorem}{Theorems}
\crefname{proposition}{Proposition}{Propositions}
\crefname{lemma}{Lemma}{Lemmas}
\crefname{corollary}{Corollary}{Corollaries}
\crefname{claim}{Claim}{Claims}
\crefname{conjecture}{Conjecture}{Conjectures}
\crefname{result}{Result}{Results}
\crefname{hypothesis}{Hypothesis}{Hypotheses}
\crefname{properties}{Properties}{Properties}
\crefname{assumption}{Assumption}{Assumptions}
\crefname{definition}{Definition}{Definitions}
\crefname{remark}{Remark}{Remarks}
\crefname{algorithm}{Algorithm}{Algorithms}

\usepackage{authblk}

\declaretheorem[thmbox=M,name=Theorem,numberwithin=section]{theorem}

\declaretheorem[thmbox=S,name=Lemma,numberwithin=section]{lemma}
\declaretheorem[thmbox=M,name=Corollary,numberwithin=section]{corollary}
\declaretheorem[thmbox=M,name=Claim,numberwithin=section]{claim}

\theoremstyle{definition}
\newtheorem{definition}{Definition}[section]

\newtheorem{remark}{Remark}[section]
\newtheorem{assumption}{Assumption}[section]

\makeatletter
\renewenvironment{proof}[1][\proofname]{\par
  \pushQED{\qed}%
  \normalfont \topsep6\p@\@plus6\p@\relax
  \trivlist
  \item[\hskip\labelsep\textbf{#1.}]\ignorespaces
}{%
  \popQED\endtrivlist\@endpefalse
}
\makeatother

\newcommand{\R}{\mathbb{R}}
\newcommand{\RR}{\mathbb{R}}
\newcommand{\EE}{\mathbb{E}}

\newcommand{\NN}{\mathbb{N}}
\newcommand{\ZZ}{\mathbb{Z}}
\def\op{{\rm op}}
\newcommand{\He}{\mathord{\mathrm{He}}}
\newcommand{\he}{\mathord{\mathrm{he}}}
\newcommand{\Sym}{\mathop{\mathrm{Sym}}}

\newcommand{\para}[1]{\left(#1\right)}
\newcommand{\cF}{\mathcal{F}}

\newcommand{\INPUT}{\REQUIRE}

\renewcommand{\vec}[1]{\bm{#1}}
\newcommand{\bx}{\vec{x}}
\newcommand{\bh}{\vec{h}}

\newcommand{\htt}[1]{}
\newcommand{\bl}[1]{}
\newcommand{\fk}[1]{}
\newcommand{\yd}[1]{}
\newcommand{\aw}[1]{}

\title{Scaling Laws from Sequential Feature Recovery:\\ A Solvable Hierarchical Model}

\author[1]{Arie Wortsman Zurich}
\author[2]{Hugo Tabanelli}
\author[2,3]{Yatin Dandi}
\author[2]{Florent Krzakala}
\author[1]{Bruno Loureiro}

\affil[1]{Département d’Informatique, Ecole Normale Supérieure, PSL \& CNRS}
\affil[2]{Information Learning and Physics Laboratory, \'Ecole Polytechnique F\'ed\'erale de Lausanne (EPFL)}
\affil[3]{Statistical Physics of Computation Laboratory, \'Ecole Polytechnique F\'ed\'erale de Lausanne (EPFL)}

\date{}

\begin{document}

\maketitle

\begin{abstract}
We propose a simple mechanism by which scaling laws emerge from feature learning in multi-layer networks. We study a high-dimensional hierarchical target that is a globally high-degree function, but that can be represented by a combination of latent compositional features whose weights decrease as a power law. We show that a layer-wise spectral algorithm adapted to this compositional structure achieves improved scaling relative to shallow, non-adaptive methods, and recovers the latent directions sequentially: strong features become detectable at small sample sizes, while weaker features require more data. We prove sharp feature-wise recovery thresholds and show that aggregating these transitions yields an explicit power-law decay of the prediction error. Technically, the analysis relies on random matrix methods and a resolvent-based perturbation argument, which gives matching upper and lower bounds for individual eigenvector recovery beyond what standard gap-based perturbation bounds provide. Numerical experiments confirm the predicted sequential recovery, finite-size smoothing of the thresholds, and separation from non-hierarchical kernel baselines. Together, these results show how smooth scaling laws can emerge from a cascade of sharp feature-learning transitions.
\end{abstract}

\section{Introduction}
\label{section:introduction}

Despite the empirical success of neural networks, we still lack a predictive theory answering a deceptively simple question: given a structured learning problem, which features are learned first, and how does their sequential discovery translate into statistical efficiency? This question lies at the intersection of three active lines of research. First, neural scaling laws suggest that the performance of large models follows power laws in data, compute, or model size \citep{kaplan2020scaling,brown2020language,hoffmann2022empirical,bahri2024explaining}. Yet most mathematical theories rely on linearized, kernel, or random-feature models, where the relevant representation is fixed in advance and learning is controlled by the spectrum of this representation \citep{caponnetto2007optimal,bordelon2020spectrum,spigler2020asymptotic,cui2021generalization,cui2023error,defilippis2024dimension}. Second, many works have emphasized that feature learning is not necessarily smooth: training can exhibit plateaus, abrupt drops in risk, and the sequential emergence of features or concepts \citep{saxe2014exact,wei2022emergent,schaeffer2023emergent,ren2025emergence,defilippis2026optimal}. Third, recent theory has begun to isolate the computational advantage of depth in compositional tasks, where deeper architectures can discover intermediate representations inaccessible to shallow methods \citep{cagnetta2024how,garnierbrun2025transformerslearnstructureddata,dandi2025computational,wang2023learning,nichani2024provable,fu2025learning,tabanelli2026deep}.\looseness=-1

This paper asks whether scaling laws can arise not from a fixed spectral bias, but from the progressive uncovering of the relevant features in the data, as can happen in deep neural networks. We investigate a mathematically tractable high-dimensional task that requires recovering hidden features across multiple layers. These latent features are combined through weights with a power-law profile. Statistically detecting an individual feature requires a sample size proportional to the inverse squared feature strength: strong features are learned first, weaker features later, and the prediction error is governed by the tail of the hidden spectrum that has not yet been recovered. Solving the task efficiently requires untangling the compositional structure. This combination of hierarchy and anisotropy leads to different scaling laws for predictors that are adapted, or not adapted, to the task geometry.
\begin{figure}[t]
    \centering
    \begin{minipage}{0.45\textwidth}
        \centering
        \resizebox{\linewidth}{!}{\begin{tikzpicture}[
    % Perspective matching the flow of 1_GzU5uXggvArqjYzxIYZ_IA.jpg
    x={(1cm,0cm)}, 
    y={(0cm,1cm)}, 
    z={(0.4cm,0.3cm)},
    scale=1.2,
    plane/.style={fill=white, opacity=0.8, draw=black, line width=0.5pt},
    fc_line/.style={black!40, thin, dashed},
]

    % % --- Background Box (Feature Learning Area) ---
    % \fill[gray!10] (-1.2,-3) rectangle (5.5,2.5);

    % --- LAYER 1: Input (Size: 1.0) ---
    \begin{scope}[shift={(0,0)}]
        \draw[plane] (0,-1,-1) -- (0,1,-1) -- (0,1,1) -- (0,-1,1) -- cycle;
        
        \foreach \i in {-0.875,-0.625,...,0.625}{
            \foreach \j in {-0.875,-0.625,...,0.625}{
                \pgfmathsetmacro{\grayvalue}{rnd*60+20}
                \fill[black!\grayvalue] (0,\i,\j) -- (0,\i+0.25,\j) -- (0,\i+0.25,\j+0.25) -- (0,\i,\j+0.25) -- cycle;
            }
        }
        \node[below=1.8cm, font=\small\sffamily, align=center] at (-0.25,-0.25) {$x \in \mathbb{R}^d$};
        
        \coordinate (L1-1) at (0,1,1);
        \coordinate (L1-2) at (0,1,-1);
        \coordinate (L1-3) at (0,-1,1);
        \coordinate (L1-4) at (0,-1,-1);
    \end{scope}

    % --- LAYER 2: h1 (Higher dimension d1 - Largest: 1.75) ---
    \begin{scope}[shift={(2.2,0)}]
        \draw[plane] (0,-1.75,-1.75) -- (0,1.75,-1.75) -- (0,1.75,1.75) -- (0,-1.75,1.75) -- cycle;
        
        % Tighter Grid for Layer 2
        \foreach \i in {-1.625,-1.375,...,1.375}{
            \foreach \j in {-1.625,-1.375,...,1.375}{
                \pgfmathsetmacro{\grayvalue}{rnd*60+20}
                \fill[black!\grayvalue] (0,\i,\j) -- (0,\i+0.25,\j) -- (0,\i+0.25,\j+0.25) -- (0,\i,\j+0.25) -- cycle;
            }
        }
        \node[below=2.2cm, font=\small\sffamily] at (-0.5,-0.5) {$A^{(1)} \in \mathbb{R}^{D}$};
        
        \coordinate (L2-1) at (0,1.75,1.75);
        \coordinate (L2-2) at (0,1.75,-1.75);
        \coordinate (L2-3) at (0,-1.75,1.75);
        \coordinate (L2-4) at (0,-1.75,-1.75);

        % Subset "patch" matched to the size of Layer 3
        \coordinate (Patch-1) at (0, 1, 1);
        \coordinate (Patch-2) at (0, 1, -1);
        \coordinate (Patch-3) at (0, -1, 1);
        \coordinate (Patch-4) at (0, -1, -1);
        
        \draw[black, thick, fill=blue!20, opacity=0.4] (Patch-2) -- (Patch-1) -- (Patch-3) -- (Patch-4) -- cycle;
    \end{scope}

    % --- LAYER 3: h2 (Anisotropic - Size: 1.0) ---
    \begin{scope}[shift={(4.4,0)}]
        \draw[plane] (0,-1,-1) -- (0,1,-1) -- (0,1,1) -- (0,-1,1) -- cycle;
        
        \foreach \i [count=\ix] in {-0.875,-0.625,...,0.625}{
            \foreach \j [count=\jx] in {-0.875,-0.625,...,0.625}{
                \pgfmathsetmacro{\weight}{(\ix+\jx-2)/12 * 100} 
                \fill[blue!\weight!red] (0,\i,\j) -- (0,\i+0.25,\j) -- (0,\i+0.25,\j+0.25) -- (0,\i,\j+0.25) -- cycle;
            }
        }
        \node[below=2.1cm, font=\small\sffamily, align=center] at (0,-0.25) {$A^{(2)}\in\mathbb{R}^{d_{1}}$\\ (Anisotropic)};
        
        \coordinate (L3-1) at (0,1,1);
        \coordinate (L3-2) at (0,1,-1);
        \coordinate (L3-3) at (0,-1,1);
        \coordinate (L3-4) at (0,-1,-1);
    \end{scope}

    % --- OUTPUT: y ---
    \node[circle, fill=orange!80!black, draw=black, thick, inner sep=3pt, label={[font=\sffamily]right:$y$}] (output) at (6.5,0,0) {};

    % --- BORDER-TO-BORDER CONNECTIONS ---
    
    \draw[fc_line] (L1-1) -- (L2-1);
    \draw[fc_line] (L1-2) -- (L2-2);
    \draw[fc_line] (L1-3) -- (L2-3);
    \draw[fc_line] (L1-4) -- (L2-4);

    \draw[fc_line] (Patch-1) -- (L3-1);
    \draw[fc_line] (Patch-2) -- (L3-2);
    \draw[fc_line] (Patch-3) -- (L3-3);
    \draw[fc_line] (Patch-4) -- (L3-4);

    \draw[fc_line] (L3-1) -- (output);
    \draw[fc_line] (L3-2) -- (output);
    \draw[fc_line] (L3-3) -- (output);
    \draw[fc_line] (L3-4) -- (output);

    % \node[above=2.25cm, font=\small\sffamily] at (1,0) {$A^{(1)} \in \mathbb{R}^{D}$};

\end{tikzpicture}}
    \end{minipage}
    \hfill
    \begin{minipage}{0.45\textwidth}
        \centering
        \resizebox{\linewidth}{!}{\begin{tikzpicture}
\begin{axis}[
    width=10cm,
    height=6cm,
    axis lines=left,
    xmin=0, xmax=10,
    ymin=0, ymax=1.2,
    xtick=\empty,
    ytick=\empty,
    axis line style={draw=black, opacity=1, thick},
    enlarge x limits={rel=0.05, upper},
    clip=false
]

% Parameters for both shapes (Marchenko-Pastur style)
\pgfmathsetmacro{\bLow}{0.08}
\pgfmathsetmacro{\bHigh}{6.1}
\pgfmathsetmacro{\gLow}{0.05}
\pgfmathsetmacro{\gHigh}{4.1}

% --- Curve 1: Gray Background (Shadow) ---
% \addplot [
%     domain=\gLow:\gHigh, 
%     samples=400, 
%     smooth, 
%     thick, 
%     gray!25,
%     forget plot
% ] { (x > \gLow && x < \gHigh) ? ( (1/x) * sqrt((\gHigh-x)*(x-\gLow)) * 0.28 ) : 0 };

% --- Curve 2: Blue Main Density ---
\addplot [
    name path=bluecurve,
    domain=\bLow:\bHigh, 
    samples=500, 
    very thick, 
    cyan!70!blue
] { (x < \bHigh && x > \bLow) ? ( (1/x) * sqrt((\bHigh-x)*(x-\bLow)) * 0.18 ) : 0 };

\path[name path=axis] (axis cs:0,0) -- (axis cs:\bHigh,0);

\addplot [
    fill=cyan!20, 
    opacity=0.6
] fill between [of=bluecurve and axis];

% --- Orange Arrows (Outliers) ---
\pgfplotsinvokeforeach{6.3, 6.45, 6.7, 7, 7.5, 8.5, 9.75}{
    \draw [orange, -stealth, thick] (axis cs:#1, 0) -- (axis cs:#1, 0.42);
    \fill [orange] (axis cs:#1, 0) circle (1.8pt);
}

% --- Vertical Dashed Line and Label ---
\draw [dashed, thick, gray!80] (axis cs:7.25, 0) -- (axis cs:7.25, 1.0) 
    node [anchor=south, black] {$m_{n}$};

% --- UPDATED: Bracket for Learned Features (Last 3 Spikes Only) ---
\draw [
    thick, 
    decoration={
        brace, 
        raise=0.2cm,
        amplitude=5pt
    }, 
    decorate
] (axis cs:7.4, 0.45) -- (axis cs:9.85, 0.45) 
node [pos=0.5, anchor=south, yshift=0.5cm,align=left] {Learned\\features};

\end{axis}
\end{tikzpicture}}
    \end{minipage}
    \caption{(\textbf{Left}) Illustration of the compositional function defined in \cref{eq:def:target} and studied throughout this paper, where the target is given by an anisotropic combination of high-degree features $\He_{q}(x)$ of the input data $x\in\mathbb{R}^{d}$. (\textbf{Right}) The key conceptual idea in our proof is to show that the relevant features of the target can be efficiently learned by a spectral method (\cref{def:estimator_1}) adapted to the compositional structure of the target. Its spectrum is composed of a bulk (blue), representing the noise, and spikes (orange), representing the signal, which is only resolved up to a scale $m_{n}$ depending on the sample size $n$. MSE (\cref{eq:prediction_MSE}) is then dominated by directions $i>m_{n}$ which are not learned.}
    \label{fig:illustration}
\end{figure}
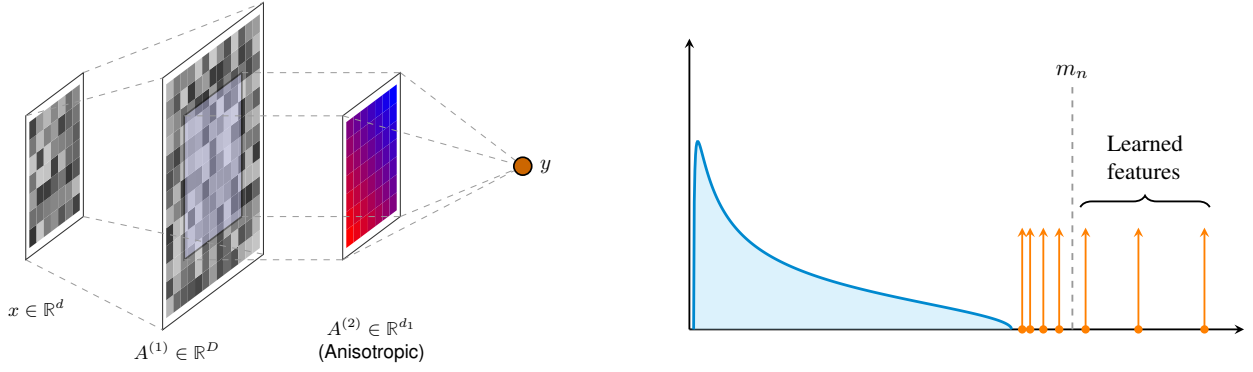

The key technical idea in our analysis is to frame feature learning as a sequence of spectral transitions corresponding to the progressive resolution of the hidden features; see \cref{fig:illustration} (right) for an illustration. We combine recent progress on spectral methods for compositional targets \citep{wang2023learning,nichani2024provable,fu2025learning,tabanelli2026deep} with the scaling-law perspective of power-law feature strengths \citep{defilippis2025scaling,ren2025emergence,defilippis2026optimal}. A large part of the analysis is therefore random-matrix-theoretic: we must control the spectrum of empirical Hermite moment matrices and the alignment of their outlier eigenvectors. The main technical challenge is that standard eigenvector perturbation results typically used to study spectral transitions are not sharp enough for this setting. Bounds such as Davis--Kahan \citep{davis1970rotation} control the worst-case subspace error in terms of spectral gaps, but the relevant gaps between power-law spikes shrink with the feature index. To obtain matching upper and lower bounds for individual feature recovery, we use a resolvent-based perturbation expansion of the empirical eigenvectors in the spirit of \citep{eldridge2018unperturbed,greenbaum2020first}. This allows us to isolate the noise projected outside the signal subspace and to prove that it becomes negligible exactly at the scale $n \asymp d^q/a_i^2$, where $a_{i}$ is the $i$-th feature weight. The same analysis also gives the converse: below this scale, the noise prevents alignment with the corresponding teacher direction.

We complement the theory with numerical experiments validating the sequential recovery of latent directions, the predicted finite-size smoothing of the sharp asymptotic thresholds, and the resulting decay of the mean-squared error. We also compare the hierarchical spectral method with non-hierarchical kernel baselines. These comparisons illustrate the role of compositional structure: while the target is a high-degree function of the input, exploiting its hierarchy allows the learner to recover the relevant latent representation at a lower sample scale.

Our \textbf{main contributions} are as follows:
\begin{itemize}[noitemsep,leftmargin=1em,wide=0pt]
\item We introduce a high-dimensional task combining hierarchical and compositional structure, providing a tractable setting for studying scaling laws in a setting where both depth and feature learning are required for efficient learning.
\item We prove sharp sample-complexity thresholds for the recovery of individual latent directions by a spectral algorithm in the high-dimensional limit. Our result is based on an eigenvector perturbation analysis of the resolvent that goes beyond the standard Davis--Kahan bounds for this problem, a technique we believe can be of independent interest.
\item We show that the emergence of scaling laws in this setting can be understood from the aggregated spectral transitions, with the error controlled by the unlearned spectral tail.
\item We provide experiments confirming the predicted recovery transitions, finite-size effects, and separation from shallow kernel methods.
\end{itemize}

Overall, our results show that power-law learning curves can arise from a simple and interpretable mechanism: hierarchical learners recover latent features one at a time, and a power-law spectrum of feature strengths converts these sharp spectral transitions into smooth scaling laws.

\section{Setting}
\label{section:setting}
We now introduce the class of target functions that will serve as a minimal model to study the computational advantage of depth. At a high level, the key idea is to create a target that is the composition of ``easy'' functions, but is globally a ``hard'' function of the input data $x\in\mathbb{R}^{d}$. More precisely\footnote{
{\bf Notation.} For $M\in\mathbb N$, write $[M]=\{1,\ldots,M\}$. For $k\ge0$, $(\mathbb R^m)^{\otimes k}$ is the space of order-$k$ tensors and $(\mathbb R^m)^{\odot k}$ its symmetric subspace. For $S,T\in(\mathbb R^m)^{\otimes k}$, $\langle S,T\rangle=\sum_{i_1,\ldots,i_k=1}^m S_{i_1,\ldots,i_k}T_{i_1,\ldots,i_k}$ denotes the full contraction. If $S\in(\mathbb R^m)^{\otimes k}$ and $T\in(\mathbb R^m)^{\otimes \ell}$, then $S\otimes_rT$ denotes contraction over $r\le \min(k,\ell)$ indices, $S\widetilde\otimes_rT$ its symmetrization, and $S\odot T=\Sym(S\otimes T)$. 
For $x\in\mathbb R^m$, the normalized tensor Hermite polynomial $\He_k(x)\in(\mathbb R^m)^{\odot k}$ is defined by 
$$\sqrt{k!}\,\He_k(x)=(-1)^k e^{\|x\|^2/2}\nabla_x^{\otimes k}(e^{-\|x\|^2/2})\, , $$ 
where $\nabla_x^{\otimes k}$ is the $k$-fold symmetric tensor of derivatives. For $m=1$, we write $\he_k$ for the corresponding normalized scalar Hermite polynomial. For a multi-index $\beta\in\mathbb Z_{\ge0}^m$, let $|\beta|=\sum_i\beta_i$ and $\He_\beta(x)=\prod_i \he_{\beta_i}(x_i)$. We write $B(m,k)=\binom{m+k-1}{k}$ and denote by $\cF[\He_k(x)]\in\mathbb R^{B(m,k)}$ the flattened degree-$k$ Hermite feature vector, indexed by $\{\beta:|\beta|=k\}$. With this normalization, for symmetric $A\in\mathbb R^{m\times m}$, $\langle A,\He_2(x)\rangle=(x^\top Ax-\Tr(A))/\sqrt{2}$, and the flattening preserves the Frobenius product. 
For vectors and matrices, $\|\cdot\|$, $\|\cdot\|_{\mathrm F}$, and $\|\cdot\|_{\mathrm{op}}$ denote Euclidean, Frobenius, and operator norms. All asymptotic notation is for $d\to\infty$: $O_d,o_d,\Omega_d,\omega_d$ denote bounded, vanishing, bounded-away-from-zero, and diverging ratios, respectively, and $\Theta_d$ means both $O_d$ and $\Omega_d$. We write $a_d\lesssim b_d$ for $a_d=O_d(b_d)$, up to constants independent of $d,n$, and $a_d\asymp b_d$ when both inequalities hold.

}, we consider a supervised learning problem where we observe a dataset $\{(x_\mu, y_\mu)\}_{\mu=1}^n$ with Gaussian inputs $x_\mu \sim \mathcal{N}(0, I_d)$, and labels generated by a compositional target of the form
\begin{align}
\forall \mu \in [n], \qquad y_\mu = f_{\star}(x_\mu).
\end{align}
As discussed above, the key feature of our model is that $f_{\star}$ is not treated as a generic high-dimensional function, but rather as the composition of successive nonlinear transformations across layers. Concretely, we assume that the target admits the following compositional structure:
\begin{align}
\label{eq:def:target}
x_\mu \in \mathbb{R}^d 
\;\longrightarrow\; h^{(1)}_\mu \in \mathbb{R}^{d^\varepsilon}
\;\longrightarrow\; h^{(2)}_\mu \in \mathbb{R}
\;\longrightarrow\; y_\mu = g(h^{(2)}_\mu).
\end{align}
\begin{itemize}[noitemsep,leftmargin=1em]
    \item \textbf{First Layer}: First, the inputs $x_\mu$ are lifted to a larger ambient space: we consider the flattened degree-$q$ Hermite tensors $F_\mu = \mathcal{F}(\He_q(x_\mu))\in \RR^{D}$ with $D:=B(d,q) = \binom{d+ q-1}{q}$. Let $\varepsilon >0$, and define $d_1 = \lfloor d^{\varepsilon} \rfloor \in \NN$ the number of hidden directions. For all $i \in [d_1]$, let $A^{(1)}_{i} \in \RR^{D}$ denote the $d_1$ Gaussian weight vectors of the teacher, with independent entries and variance $\frac{1}{d^{q}}$. The first layer of the hierarchical model is defined by projections on the $d_1$ directions in the ambient space of dimension $D$. Namely, the first layer of the teacher is the vector $h^{(1)}_{\mu} \in \RR^{d_1}$, with entries: 
    \begin{equation}
       \forall (\mu, i) \in [n]\times [d_1] ,\quad h^{(1)}_{\mu,i} = \langle A^{(1)}_{i}, F_\mu \rangle \, , 
    \end{equation} 
    We identify the collection of symmetric tensors $\{A_i^{(1)}\}_{i\in[d_1]}$ with the row-stacked matrix $A^{(1)}\in\RR^{d_1\times D}$ whose $i$-th row is the flattened tensor $\mathcal F(A_i^{(1)})^\top$.
    \item \textbf{Second layer}: Let $\gamma \geq 0$ be the power-law exponent. For each $i \in [d_1]$, we define the weight $\lambda_i = Z_{\gamma} z_i i^{-\gamma}$, where $Z_{\gamma}$ is a normalization factor, and $z_i \sim \mathrm{Rad}(\sfrac{1}{2})$. We also define the diagonal second-layer matrix $A^{(2)}=\mathrm{diag}(\lambda_1,\ldots,\lambda_{d_1})\in\RR^{d_1\times d_1}$. The second layer is
    \begin{equation}
        h_\mu^{(2)}
        =
        \langle A^{(2)},\He_2(h_\mu^{(1)})\rangle
        =
        \frac{1}{\sqrt{2}}\sum_{i=1}^{d_1}\lambda_i\left((h_{\mu,i}^{(1)})^2-1\right).
        \label{eq:def_second_layer}
    \end{equation}
    \item \textbf{Output:} The observed output $y_\mu = g(h^{(2)}_\mu)$ is defined with respect to the second latent feature $h^{(2)}_\mu$ through the non-linearity $g:\RR \rightarrow \RR$. Additionally, we assume that $g$ is centered and has information exponent ${\rm IE}(g)=1$, i.e., $\EE[g(z)]=0$ and $\EE[g'(z)]\neq0$ where $z\sim\mathcal{N}(0,1)$. Thus, the output reads, 
    \begin{align}
        \forall \mu \in [n],\qquad  y_\mu = g(h_{\mu}^{(2)})\,.
    \end{align}
    Note that the value of $Z_{\gamma}$ in $\lambda_i = Z_{\gamma}z_i i^{-\gamma}$ is precisely chosen such that $\mathrm{Var}(y_\mu)=\Theta_d(1)$. Its precise value is $Z_{\gamma} = (\sum_{i=1}^{d_1} i^{-2\gamma})^{-\sfrac{1}{2}}$ and is derived in \Cref{lemma:normalization_C_gamma}. 
\end{itemize}
This compositional model is closest technically to the hierarchical spectral model of \citep{tabanelli2026deep}, and belongs to the broader line of works on hierarchical polynomial targets \citep{wang2023learning,nichani2024provable,fu2025learning}. The important difference is that the second layer in \cref{eq:def_second_layer} has an anisotropic power-law spectrum, which is the key ingredient behind the scaling behavior studied below. The representation, illustrated in \cref{fig:illustration} (left), can be interpreted as follows: the first layer extracts intermediate features $h^{(1)}\in\mathbb R^{d_1}$ from the degree-$q$ Hermite feature space $\mathbb R^D$, the second layer combines these features into a scalar representation $h^{(2)}$, and the output is obtained through a one-dimensional readout $g$.

This compositional structure has important statistical consequences. Although the signal is built from degree-$q$ intermediate features, its leading informative component is a degree-$2q$ function of the input, so shallow orthogonally invariant kernels require $n=\omega_d(d^{2q})$ samples \citep{mei2022generalization}. A depth-exploiting procedure can instead recover the degree-$q$ latent representation at the scale $D=\Theta_d(d^q)$ and then solve a low-dimensional second-layer problem \citep{tabanelli2026deep}. The power-law anisotropy introduced here further makes recovery sequential: strong directions emerge first, weak directions later, yielding the scaling mechanism analyzed below.

\subsection{The spectral algorithm}\label{subsec:spectral_algo}
We now show that exploiting the compositional structure in \cref{eq:def:target} allows efficient recovery of the hidden features. The key idea is to construct a hierarchical spectral algorithm. More precisely, our algorithm consists in estimating the first latent feature $h^{(1)}_\mu$, followed by an estimation of the second latent feature $h^{(2)}_\mu$, and finally learning the function $g$ from the second latent features $h^{(2)}\in\mathbb{R}^{d}$, which boils down to a scalar non-parametric problem. The end-to-end algorithm is presented in Alg.~\ref{alg:agnostic-recovery}.\looseness=-1

\subsubsection{Estimation of the First Layer}

\begin{figure}[t]
    \centering
    \includegraphics[width=\linewidth]{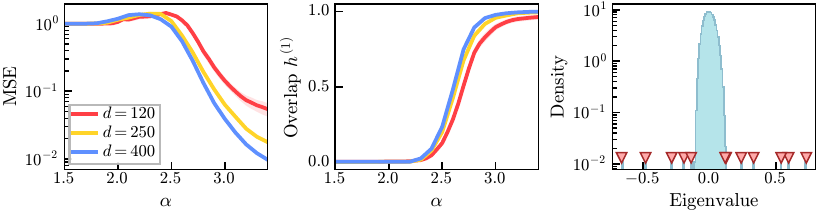}
    \caption{
    \textbf{Feature recovery in the power-law setting.}
    Parameters: $q=2$, $\varepsilon=0.5$, $\gamma=0.4$, $g^\star=\mathrm{id}$.
    \textbf{Left:} Test MSE versus $\alpha=\log(n)/\log(d)$.
    \textbf{Center:} First-layer feature overlap $q_h^{(1)}$ versus $\alpha$
    \textbf{Right:} Spectrum of the first-layer moment matrix for $d=140$ and $\alpha=3.5$, red markers indicate the top $d_1=12$ absolute eigenvalues.
    }
    \label{fig:several-d}
\end{figure}

\begin{definition}
    Let $\{ (x_\mu, y_\mu) \} _{\mu=1}^{n}$ denote a batch of data, and define the following spectral estimator:
    \begin{align}
    \widehat{C} = \frac{1}{n} \sum_{\mu=1}^{n} y_\mu \He_2(F_{ \mu})\in\mathbb{R}^{D\times D}
    \end{align}
    Then, the top $d_{1}$ eigenvectors $\widehat A^{(1)}_{1}, \dots, \widehat A^{(1)}_{d_1}$ of $\widehat{C}$ define a spectral estimator of the hidden directions $A^{(1)}_{1}, \dots, A^{(1)}_{d_1}$. 
\label{def:estimator_1}
\end{definition}
To see why \cref{def:estimator_1} defines an estimator for $A^{(1)}$, consider the decomposition:
\begin{equation}
    \widehat{C} = \EE[\widehat{C}] + (\widehat{C} - \EE[\widehat{C}]) = \mathrm{signal} + \mathrm{noise}
\end{equation}
\textbf{Recovery of $A^{(1)}_i$ by the signal $\EE[\widehat{C}]$:}
The alignment of $\EE[\widehat{C}]$ with the hidden directions $A^{(1)}$ can be understood from a Gaussian equivalence heuristic \citep{tabanelli2026deep}. The degree-$q$ Hermite features $F_{q,\mu}\in\RR^{D}$ behave, in the regime $d\to\infty$ with $n\ll d^{q+1}$, as a Gaussian vector $\tilde{x}_\mu\in\RR^{D}$ of dimension $D\asymp d^q$. Under this equivalence, the label $y_\mu$ depends on $\tilde{x}_\mu$ only through its projection onto the $A^{(1)}$-subspace, i.e.\ through $\tilde{x}^{\parallel}_\mu := A^{(1)}\tilde{x}_\mu$, while the orthogonal component $\tilde{x}^{\perp}_\mu$ is independent of $y_\mu$. Splitting $\widehat{C}$ along this decomposition gives
\begin{equation}
    \widehat{C} \;\simeq\; \underbrace{\nu_1\, A^{(1)\top} A^{(2)} A^{(1)}}_{\text{signal } \EE[\widehat{C}]} \;+\; \underbrace{\tfrac{1}{n}\,\tilde{X}_\perp Y \tilde{X}_\perp^{\top} + o_d(1)}_{\text{noise } \widehat{C}-\EE[\widehat{C}]},
\end{equation}
where $\nu_1=\mathbb{E}[g'(z)]$ is the first Hermite coefficient of $g_\star$ and $\tilde{X}_\perp\in\RR^{B(d,q)\times n}$ collects the perpendicular components. The signal is a rank-$d_1$ matrix supported on the $d_1$-dimensional row space of $A^{(1)}$, so it places $d_1$ spikes inside the $d^q\times d^q$ matrix $\widehat{C}$, consequently the top eigenvectors of $\widehat{C}$ recover $A^{(1)}$ as shown in Fig~\ref{fig:several-d}. The noise, on the other hand, is a sample covariance of vectors that are independent of the labels and live in the $\sim d^q$-dimensional orthogonal complement, so it behaves as an isotropic bulk of operator-norm scale $\sqrt{d^q/n}$, motivating the following heuristic:
\begin{equation}
    \widehat{C} \simeq \EE[\widehat{C}] + \sqrt{\dfrac{d^{q}}{n}}\, W, \quad \mathrm{with}\quad W\sim \mathrm{GOE-like}\,.
    \label{eq:signal_vs_noise_heuristic}
\end{equation}

\textbf{Effective signal-to-noise ratio for the $i$-th feature:} The $i$-th population spike has magnitude $|\lambda_i|=Z_\gamma i^{-\gamma}$, while the empirical noise has size $\sqrt{d^q/n}$; hence the $i$-th direction emerges when $Z_\gamma i^{-\gamma}\gtrsim \sqrt{d^q/n}$, equivalently
\begin{equation}
    n_i \gtrsim \frac{d^q i^{2\gamma}}{Z_\gamma^2}.
    \label{eq:pred_sample_complexity}
\end{equation}

To formalize the recovery of individual hidden directions by the spectral estimator, we use the following notion of weak recovery.

\begin{definition}[Weak Recovery]
We say that a sequence of estimators $(u_N)_{N \in \NN}$, with $u_N \in \RR^{N}$, \emph{weakly recovers} a sequence of unit vectors $(v_N)_{N \in \NN}$, with $v_N \in \RR^{N}$, if there exists a constant $c > 0$, independent of $N$, such that
\[
\liminf_{N \to \infty}\, \mathbb{P}\!\left( \left|\langle u_N, v_N \rangle\right| \geq c \right) = 1,
\]
i.e.\ $|\langle u_N, v_N \rangle| \geq c$ with high probability as $N \to \infty$.
\end{definition}

Given the recovered eigenvectors $\widehat A^{(1)}_1,\ldots,\widehat A^{(1)}_{d_1}$, we define the learned first-layer representation by
\begin{equation}
    \widehat h^{(1)}_{\mu,i}
    =
    \langle \widehat A^{(1)}_i,F_{\mu}\rangle,
    \qquad
    (\mu,i)\in[n]\times[d_1].
    \label{eq:def_hhat_first_layer}
\end{equation}
Thus the first spectral step produces a low-dimensional representation $\widehat h^{(1)}_\mu\in\RR^{d_1}$ of the input.

\subsection{Second-layer and readout estimation}\label{subsec:est_2nd_layer}

Once the first-layer representation has been estimated, the rest of the pipeline operates in the low-dimensional latent space $\RR^{d_1}$.
\begin{equation}
    \widehat A^{(2)}
    =
    \frac1n\sum_{\mu=1}^n y_\mu \He_2\para{\widehat h^{(1)}_\mu}
    \in \RR^{d_1\times d_1}.
    \label{eq:def_A2_hat}
\end{equation}
In contrast with the first layer, this step is performed in dimension $d_1\ll D$ and is not responsible for the high-dimensional spectral transition analyzed below. In the diagonal teacher considered here, the population version of $\widehat A^{(2)}$ is aligned with the second-layer matrix $A^{(2)}$ defined in \Cref{eq:def_second_layer}. As in the hierarchical spectral pipeline of \citep{tabanelli2026deep}, the statistical bottleneck is the recovery of the first-layer directions in the degree-$q$ Hermite feature space. Once these directions have been recovered, estimating the second layer only involves the learned representation in dimension $d_1$, and therefore does not generate the feature-wise high-dimensional transitions studied in \Cref{sec:main_th}. We then define the learned second latent feature by
\begin{equation}
    \widehat h^{(2)}_\mu
    =
    \left\langle \widehat A^{(2)},\He_2\para{\widehat h^{(1)}_\mu}\right\rangle.
    \label{eq:def_hhat_second_layer}
\end{equation}
It remains to fit the one-dimensional readout $g$ from the learned scalar features $\widehat h^{(2)}_\mu$. In the experiments, we use ridge regression on a fixed feature map $\phi:\RR\to\RR^p$,
\begin{equation}
    \widehat a
    \in
    \mathrm{argmin}_{a\in\RR^p}
    \frac1n\sum_{\mu=1}^n
    \para{y_\mu-\langle a,\phi(\widehat h^{(2)}_\mu)\rangle}^2
    +
    \rho\|a\|^2.
    \label{eq:ridge_readout}
\end{equation}
The resulting predictor is denoted by $\widehat f$. The full estimation error will be measured by the test or generalization mean-squared error.

% After computing the estimator in \Cref{def:estimator_1}, we proceed with kernel ridge regression with the new features given by the vector $\hat{h}^{(1)}_{\mu} = \langle\hat{A}, \He_2(F_\mu)\rangle$, where $\hat{A}$ is the estimator of $A^{(1)}$ obtained by the spectral algorithm. 

% To make the discussion above precise, we need to quantify recovery. Below, we introduce two metrics we will adopt in the following. 

% \begin{definition}[MSE]
% Given a sample complexity $n$, denote by $I_{n} \subseteq [d_1]$ the indices of all the directions that are retrieved at this sample complexity. We define the MSE as: 
% $$
% \mathrm{MSE}(n) = \| \sum_{i=1}^{d_1} \lambda_{i} A_{i}  - \sum_{i \in I} \lambda_i A_i\|_{2}^2,
% $$
% where $\lambda_{i}$ are defined in \Cref{eq:def_second_layer}. 
% \label{def_MSE}
% \bl{Change this to generalisation error / MSE on the labels.}
% \end{definition}
\begin{definition}[Mean-squared error]
Let $\widehat f$ be the predictor returned by the algorithm trained on the dataset $\{x_\mu, y_\mu\}_{\mu \in [n]}$. We define its mean-squared error, or equivalently its squared-loss generalization error, as
$$
    \mathrm{MSE}(n)
    =
    \mathbb E_{x\sim \mathcal N(0,I_d)}
    \left[
        \left(\widehat f(x)-f_\star(x)\right)^2
    \right],
$$
where the expectation is over an independent test point $x$. Equivalently, if $y=f_\star(x)$ and $\widehat y=\widehat f(x)$, then $\mathrm{MSE}(n)=\mathbb E[(\widehat y-y)^2]$.
\label{def:MSE}
\end{definition}
% The weighting in the MSE reflects the fact that each direction has a different weight in the teachers output. 

By standard results on kernel ridge regression \citep{caponnetto2007optimal}, the error introduced by the ridge regression step in Equation \ref{eq:ridge_readout} is $\simeq \frac{1}{n}$ which we show to be subdominant compared to the error in the estimation of $\widehat h^{(2)}$.
Hence, the error is governed by which first-layer directions have been recovered. Directions with larger $|\lambda_i|$ contribute more strongly to the label, so failing to recover them has a larger effect on the MSE. This leads to the unrecovered-tail estimate below.

The preceding signal-to-noise heuristic gives the scaling picture that the main results make precise. Write $a_i = |\lambda_i|$ for the strength of the $i$-th latent feature, so that $a_i \propto i^{-\gamma}$. The $i$-th direction becomes recoverable when its spike separates from the empirical noise, which as we have discussed occurs at the sample scale $n_i \asymp d^q/a_i^2$,
and remains hidden below this threshold. If $m(n)$ denotes the number of directions recovered at sample size $n$ (see \cref{fig:illustration} for an illustration), the prediction error is then controlled by the unrecovered tail, $\mathrm{MSE}(n) \simeq \sum_{i > m(n)} a_i^2$. In particular, for summable spectra $2\gamma > 1$,
\begin{equation}
\mathrm{MSE}(n) \asymp (\sfrac{n}{d^q})^{-1 + 1/(2\gamma)}.
\label{eq:prediction_MSE}
\end{equation}
Thus, the model predicts smooth power-law generalization as the aggregate effect of many sharp spectral recovery transitions. The next section proves this prediction.

\begin{remark}
Our layer-wise spectral estimator is closely aligned with the learning strategy of \citep{tabanelli2026deep}. In particular, their connection to gradient descent in App.C suggests that the spectral estimator studied here is the one naturally emerging from gradient-based training in this hierarchical setting.
\end{remark}

\subsection{Further Related work}\label{subsec:related_works}

\paragraph{Hierarchical and compositional models.}
Depth is often argued to be effective because it allows to exploit hierarchical or compositional structure in the data. This intuition has motivated depth-separation results and compositional target models from both approximation-theoretic and statistical viewpoints \citep{pmlr-v49-telgarsky16,mhaskar2017and,poggio2017and,daniely2017depth,mossel2016deep}. More recent works study random hierarchy models and high-dimensional hierarchical targets, showing that deep networks or layer-wise procedures can exploit intermediate representations inaccessible to shallow methods \citep{garnierbrun2025transformerslearnstructureddata,cagnetta2024how,dandi2025computational}. Closest to us are the analyses of hierarchical polynomial targets and nonlinear feature learning in three-layer networks \citep{wang2023learning,nichani2024provable,fu2025learning}, as well as the hierarchical spectral method of \citep{tabanelli2026deep}. We depart from these works by adding an anisotropic power-law spectrum over the latent features and by proving matching upper and lower thresholds for individual feature recovery, which allows us to derive an aggregate scaling law from the cascade of transitions.

\paragraph{Multi-index and spectral methods.}
A  related line studies multi-index models, where the target depends on a low-dimensional projection of the input. These have been used to characterize statistical-computational gaps, weak-recovery thresholds, and the limitations of kernel methods \citep{aubin2018committee,barbier2019optimal,BenArous2021,abbe2022merged,bietti2022learning,troiani2024fundamental,damian2024computational}. Spectral methods are particularly relevant in this context, since they provide sharp recovery guarantees for low-dimensional structure in Gaussian models \citep{lu2020phase,mondelli2018fundamental,maillard2022construction,kovavcevic2025spectral,defilippis2025optimal}. Our estimator builds on this spectral viewpoint, but differs from standard multi-index learning in that the latent structure is compositional and the strengths of the recovered directions are anisotropic and power-law distributed.

\paragraph{Scaling laws and power-law spectra.}
A large body of work has studied scaling laws in settings where the representation is fixed, for instance in kernel or random-feature models, where generalization is controlled by the spectrum of the associated feature map \citep{caponnetto2007optimal,bordelon2020spectrum,spigler2020asymptotic,cui2021generalization,maloney2022solvable,cui2023error,bahri2024explaining,paquette2024four,defilippis2024dimension,atanasov2024scaling,bordelon2024dynamical,wortsman2025kernel}. A distinct line of work investigates how increasing the number of trainable parameters affects optimization, initialization, and expressivity \citep{yang2021tuning,bordelon2023depthwise,chizat2024feature,chaintron2026resnets}. More recently, several works on quadratic and shallow neural-network models have shown how scaling laws can arise from feature learning itself \citep{ren2025emergence,benarous2025learning,defilippis2025optimal,defilippis2025scaling,boncoraglio2025single}. Closest to our work are \citep{defilippis2025optimal,defilippis2025scaling}, which obtain related rates and learned-representation spectra, including the sequential emergence of learned directions. The present work shows that analogous rates arise in a genuinely multi-layer, hierarchical setting, suggesting that the mechanism linking power-law spectra, feature recovery, and scaling laws is robust beyond shallow quadratic models.

\paragraph{Gaussian equivalence and polynomial features.} A related technical literature studies Gaussian equivalence and universality phenomena for polynomial feature maps, random feature matrices, and high-dimensional kernel matrices \citep{hu2024asymptotics,COLTXU2025,wen2025does,lu2025equivalence}. Although our proof does not proceed by replacing the Hermite feature vectors with an equivalent Gaussian model, this line of work provides a useful comparison point for understanding when polynomial features behave as if they were Gaussian and when non-Gaussian corrections become relevant. Our analysis instead keeps the Hermite structure explicit and uses Wiener-chaos tools, such as product formulae, integration by parts, hypercontractivity, and contraction estimates \citep{nualart2005central,nourdin2009stein,nourdin2012normal}. These tools allow us to control the empirical Hermite moment matrices and the perturbative eigenvector expansion directly, without invoking a full Gaussian-equivalence reduction.
\section{Main Theorems}\label{sec:main_th}

\begin{figure}[t]
    \centering
    \includegraphics[width=0.92\linewidth]{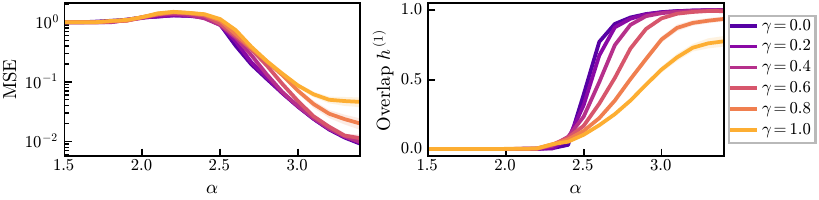}
\caption{
\textbf{Effect of the power-law exponent.}
Parameters: $d=400$, $q=2$, $\varepsilon=0.5$, $g^\star=\mathrm{id}$.
\textbf{Left:} Test MSE versus $\alpha=\log(n)/\log(d)$.
\textbf{Right:} First-layer feature overlap $q_h^{(1)}$ versus $\alpha$.
}
\label{fig:several-gamma}
\end{figure}

We now rigorously prove the predictions from the heuristics described in the
last section under one of the following two regimes.

\begin{assumption}[Readout regimes]
\label{ass:readout-regimes}
We work under one of the following two regimes:
\begin{enumerate}[label=\textnormal{(\roman*)}, itemsep=0pt, topsep=0pt]
    \item Identity readout: $g(t)=t$ and $\gamma>0$, or
    \item Delocalized nonlinear readout: $0<\gamma<1/2$ and $g$ is centered with information exponent one, i.e. $\mathbb E[g'(Z)]\neq0$ for $Z\sim\mathcal N(0,1)$.
\end{enumerate}
\end{assumption}

\subsection{Recovery of the First Layer}
\label{subsection:recovery_first_layer}

Our first set of results concerns the recovery of the different directions $A^{(1)}_{1}, \dots, A^{(1)}_{d_1}$. This will be proven by deriving matching upper and lower bounds to the sample complexity.

\begin{theorem}[Weak Recovery]
    Consider the setting described in \Cref{section:setting}, and assume that one of the two regimes in \Cref{ass:readout-regimes} holds. Assume moreover that $d_1=d^\varepsilon$ under the technical small-$\varepsilon$ growth conditions used in the proof; in particular, $\varepsilon<q/2$, and in the nonlinear regime $0<\gamma<1/2$ we also require $\varepsilon<q/(1-2\gamma)$. 
    \begin{enumerate}
        \setlength{\itemsep}{0pt}
        \setlength{\parsep}{0pt}
        \setlength{\topsep}{0pt}
        \setlength{\partopsep}{0pt}
        \item \textbf{Sufficient Sample Complexity:} The $k$-th eigenvector of $\hat{C}$, denoted by $u_{k}$ satisfies 
    \begin{align}\label{eq:ov_rate}
    \left | \left \langle \dfrac{u_{k}}{\|u_{k}\|_2}, \dfrac{A^{(1)}_{k}}{\|A^{(1)}_{k} \|_2 } \right \rangle  \right | = 1 - O_d\left(
    \frac{d^q\,k^{2\gamma}}{nZ_\gamma^2}
\right). 
    \end{align}
    In particular, if $n = \omega_{d} \left (d^{q} k^{2\gamma}Z_{\gamma}^{-2} \right)$, with high probability the direction $A^{(1)}_{k}$ is recovered by  \Cref{alg:agnostic-recovery}. This gives a $1/n$ rate for the decay of the overlap error. 
    \item \textbf{Necessary Sample Complexity:} Let $\delta>0$ be independent of $d$.  If $n = \Theta \left (d^{q} k^{2\gamma}Z_\gamma^{-2} d^{-\delta}\right)$, then, with probability at least $1 - o_{d}(1)$,  $u_{k}$ is not recovered by \Cref{alg:agnostic-recovery}. 
    \end{enumerate}
    \label{thm:sample_complexity}
\end{theorem}
The complete proof of \Cref{thm:sample_complexity} is discussed in \Cref{appendix:proof_thm_1}, but we highlight below the main idea.
\begin{proof}[Sketch of the Proof:] 
A natural first attempt to prove this result would invoke the classical Davis--Kahan bound \citep{davis1970rotation}. Unfortunately, this does not provide the tight sample complexity due to the lack of a sufficiently large spectral gap in our power-law setting. The main technical challenge is therefore to have a finer control the eigenvectors of $\hat{C}$. In order to avoid using Davis--Kahan, we use the Neumann expansion of the resolvent (see, e.g \citep{eldridge2018unperturbed, greenbaum2020first}) to prove: 
\begin{equation}
    \hat{u}_{k} = u_{k }+\sum_{j=1, j\not = k}^{d_1} \dfrac{u_{j}^{\top} \Delta u_{j}}{\lambda_{k} - \lambda_{j}} u_{j} + \dfrac{1}{\lambda_{k}}P_{\mathrm{Ker}} \Delta u_{k} + o(\| \Delta\|_\op^2),
    \label{eq:main_eigenvector_perturbation_id}
\end{equation}
where $P_{\mathrm{Ker}}$ denotes the projection into $\mathrm{Ker}(\EE[\hat{C}])$, and $u_{i} = A^{(1)}_i$. 

\Cref{eq:main_eigenvector_perturbation_id} is valid as long as we are in the regime where $\|\hat{C} - \EE\left[ \hat{C}\right ]\|_\op $ is small. By using this identity, one then has to prove that all terms apart from $u_{k}$ have negligible norm. The rest of the proof concerns showing that both terms in \Cref{eq:main_eigenvector_perturbation_id} are either small (for the sufficiency result) or large (for the necessary part). 
\end{proof}

% \begin{remark}
%     The reader may note that the proof of \Cref{thm:sample_complexity} is stated for $\gamma >0$. However, the bound for the isotropic case $\gamma = 0$ is easier and follows by the same reasoning, without the need to use \Cref{eq:main_eigenvector_perturbation_id}. 
% \end{remark}
\begin{remark}
\Cref{thm:sample_complexity} is understood under \Cref{ass:readout-regimes}: the identity readout is covered for all $\gamma>0$, while nonlinear readouts of $\mathrm{IE}=1$ are covered only in the delocalized regime $0<\gamma<1/2$. The isotropic case $\gamma=0$ is simpler and follows by the same reasoning, without the need for \Cref{eq:main_eigenvector_perturbation_id}.
\end{remark}

Note that \Cref{thm:sample_complexity} tells us that the sample complexity derived heuristically in \Cref{eq:signal_vs_noise_heuristic} for recovering the $k$-th direction is indeed correct. At the same time, it also tells us that recovery cannot be achieved with any less data with \Cref{alg:agnostic-recovery}. 

\begin{corollary} Let $n  \in \NN$, with $ n = \omega_{d}(d^{q})$. Let $m_{n}$ denote the number of recovered directions given $n$. Then: 
\begin{align}
m_{n} = \left ( \frac{Z^2_\gamma}{D} n\right )^{\frac{1}{2\gamma}}
\end{align}
\label{cor:number_of_directions}
\end{corollary}

\subsection{Recovery of the Second Layer and Rates for the Generalization Error}
\label{subsection:generalization}
% In this section, we focus on the case where the non-linearity $g$ is Lipschitz. After computing $\widehat C^{(1)}$, \Cref{alg:agnostic-recovery} recovers a subset of the first-layer directions. Let $m_n$ denote the number of recovered directions at sample size $n$. From these $m_{n}$ recovered directions, the algorithm forms the learned first-layer features $\widehat h^{(1)}_{\mu,i}=\langle \widehat A_i^{(1)},F_\mu\rangle$ and then the corresponding centered quadratic features $(\widehat h^{(1)}_{\mu,i})^2-1$. If $g$ is a polynomial of fixed degree, fitting the second layer on these learned features is a low-dimensional linear regression problem. Since $m_n\le d_1\ll D=\Theta_d(d^q)$, this step is not the high-dimensional bottleneck, the dominant difficulty is the spectral recovery of the first-layer directions characterized in \Cref{thm:sample_complexity}. The following theorem provides the rates for the MSE in \Cref{def:MSE} obtained from \Cref{alg:agnostic-recovery}. 
Under \Cref{ass:readout-regimes} the second-layer/readout fit is low-dimensional once $m_n$ first-layer directions have been recovered. Thus the high-dimensional bottleneck is the spectral recovery in \Cref{thm:sample_complexity}, and the MSE is governed by the unrecovered spectral tail.

\begin{theorem}[Rates for the Generalization Error]
Consider the setting described in \Cref{section:setting}, and assume that one of the two regimes in \Cref{ass:readout-regimes} holds. Then, for $n = \omega_{d}(\frac{d^{q}}{Z_\gamma})$: 
\begin{equation}
    \mathrm{MSE}(n) = \begin{cases}
       \Theta_{d}(1) - \left (\dfrac{n}{d_1 d^{q}} \right )^{\frac{1}{2\gamma} -1}& \text{ if } 0 < \gamma < \frac{1}{2},\, d^{q}\ll n \ll d^{q}d_1, \\ 
      n^{-1 + \frac{1}{2\gamma}}, & \text{if } \gamma > \frac{1}{2} \text{ under regime \textnormal{(i)}}, \, d^{q}\ll n. 
    \end{cases}
\end{equation}
\label{thm:rates}
\end{theorem}

\begin{remark}
    The case in \Cref{thm:rates} where $\gamma > \frac{1}{2}$ is only proved in the case where $g$ is a linear function. Numerical evidence suggest this remains true for $\gamma > \frac{1}{2}$ in the non-linear case, but rigorously showing it is considerably more challenging. We conjecture that this rates still hold if $\gamma > \frac{1}{2}$ and $g$ is non-linear, and leave the proof for future work. 
\end{remark}

The proof of \Cref{thm:rates} builds on \Cref{thm:sample_complexity} and \Cref{cor:number_of_directions} and is presented in \Cref{section:derivation_rates}. The above rates match optimal rates from \cite{defilippis2026optimal} with the input dimension replaced by the effective dimension $D=\Theta(d^q)$. 

\begin{remark} 
In order to derive rates, it is necessary to show that the spectral estimator learn directions in a sequential, sharp way. That is, the algorithm either learns or not a particular direction. For this reason, proving only  the sufficiency part of \Cref{thm:sample_complexity} is not enough, and we also need a refutation result like the second part of \Cref{thm:sample_complexity}. 
\end{remark}

\section{Numerical Experiments}
\label{section:experiments}

\begin{figure}[t]
    \centering
    \includegraphics[width=\linewidth]{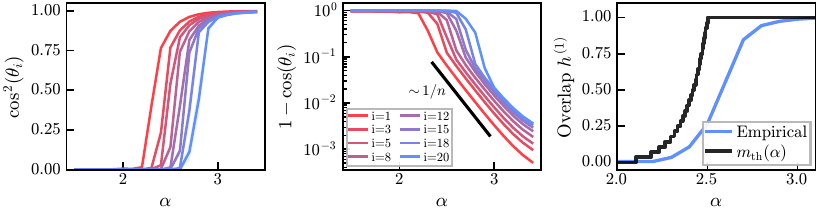}
    \caption{
    \textbf{Direction-wise recovery.}
    Parameters: $q=2$, $\varepsilon=0.5$, $\gamma=0.4$, $g^\star=\mathrm{id}$.
    \textbf{Left:} Direction-wise alignments $\cos^2(\theta_i)$ versus $\alpha$ with $d=400$.
    \textbf{Center:} Direction-wise errors $1-\cos(\theta_i)$ versus $\alpha$, with a $1/n$ guide and $d=400$.
    \textbf{Right:} Empirical overlap $q_h^{(1)}$ compared with the theoretical curve $m_{\mathrm{th}}(\alpha)$ with $d=800$.
    }
    \label{fig:direction-wise-theory}
\end{figure}

In this section, we confront the feature-wise recovery predictions of \Cref{thm:sample_complexity} with numerical experiments. Additional numerical details are given in Appendix~\ref{app:numerics2}.

\textbf{Overall recovery transition.} Fig.~\ref{fig:several-d} shows the overall transition of the estimator as the sample exponent $\alpha=\log(n)/\log(d)$ increases. For small $\alpha$, the sample size is below the spectral recovery scale: the MSE is large, the learned-feature overlap is small, and the spectrum is dominated by the bulk. Around the transition predicted by the signal-plus-noise picture in \cref{eq:signal_vs_noise_heuristic}, outlier eigenvalues separate from the bulk in the right panel; at the same scale, the MSE drops and the overlap increases. For larger $d$, the post-transition overlap gets closer to one, while the remaining gap at smaller $d$ reflects finite-size smoothing and incomplete recovery of the weakest directions.

\textbf{Effect of the power-law exponent.} Fig.~\ref{fig:several-gamma} shows how the recovery window depends on the anisotropy exponent $\gamma$. By \Cref{eq:pred_sample_complexity}, the recovery threshold of direction $i$ is controlled by its weight $|\lambda_i|=Z_\gamma i^{-\gamma}$. Thus the first threshold, corresponding to the strongest direction, decreases with $\gamma$, while the last thresholds increase with $\gamma$. Larger $\gamma$ therefore makes the teacher spectrum more anisotropic: the strongest directions are recovered earlier, while the weak tail is pushed to larger sample sizes. This spreads the recovery process over a wider interval of $\alpha$, producing the gradual MSE decay observed in the figure.

\textbf{Direction-wise comparison with theory.} Fig.~\ref{fig:direction-wise-theory} tests the sharper, direction-wise prediction of \Cref{thm:sample_complexity}. Since global MSE curves mix many directions and finite-size effects, we instead track direction-wise overlaps with the recovered eigenspace. If $u_i=A_i^{(1)}/\|A_i^{(1)}\|$ and $\widehat U$ is an orthonormal basis of the eigenspace recovered from $\widehat C$, we define $\cos^2(\theta_i)=\|\widehat U^\top u_i\|^2$. The left panel shows that directions are recovered in the order dictated by their weights: larger $|\lambda_i|$ directions turn on first, while weaker directions appear later, around the thresholds of \Cref{eq:pred_sample_complexity}. The center panel probes the perturbative regime after recovery: the angular error $1-\cos(\theta_i)$ follows the $1/n$ decay predicted by \Cref{eq:ov_rate}. Finally, the right panel compares the empirical aggregate overlap with the theoretical count $m_{\rm th}(\alpha)$, obtained by counting the directions whose predicted thresholds satisfy $n_i\le d^\alpha$. The staircase is smoothed at finite $d$, but its ordering and scale agree with the theory.

\begin{figure}[t]
    \centering
    \includegraphics[width=0.95\linewidth]{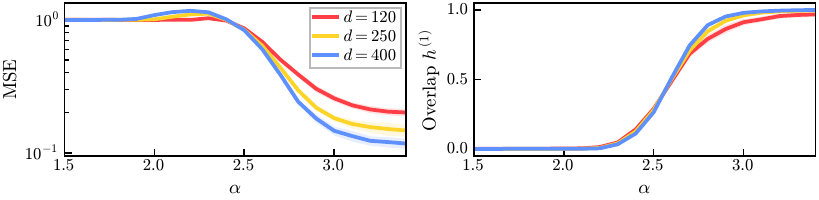}
    \caption{
    \textbf{Non-identity readout.}
    Parameters: $q=2$, $\varepsilon=0.5$, $\gamma=0.4$, and $g^\star=\tanh$.
    \textbf{Left:} Test MSE versus $\alpha=\log(n)/\log(d)$.
    \textbf{Right:} First-layer feature overlap $q_h^{(1)}$ versus $\alpha$.
    }
    \label{fig:non-identity-readout}
\end{figure}

\textbf{Nonlinear readout.} Fig.~\ref{fig:non-identity-readout} repeats the recovery experiment with $g^\star=\tanh$. The qualitative picture is unchanged: prediction error decreases at the same scale at which the first-layer overlap grows. Finite-size effects are more visible at lower dimensions, but the dominant bottleneck remains the recovery of the latent first-layer representation, rather than the final low-dimensional readout.

\section{Discussion and Future Directions}
We have introduced a model in which scaling laws arise from sequential feature recovery. The central message is that hierarchy and anisotropy work together: depth exposes the relevant intermediate representation, while the power-law structure of the signal spreads the recovery of its components across sample sizes. As a result, a smooth power-law learning curve can emerge from the aggregate effect of many sharp feature-learning transitions. This gives a mechanism by which power laws are generated by representation learning, rather than inherited from a fixed kernel spectrum.

The main limitations of our analysis are also what make this mechanism transparent: the hierarchy is specified in advance, inputs are Gaussian, and learning is performed by a layer-wise procedure. These assumptions enable sharp recovery and non-recovery guarantees, while pointing to natural next questions: extending the mechanism to more generic data, richer nonlinearities, and higher information exponents. More broadly, our results suggest that scaling laws in deep learning may reflect not only spectral bias at a fixed representation, but the progressive organization of representations across depth.

\section*{Acknowledgements}
We would like to thank Pierre Mergny and Lenka Zdeborova for insightful discussions. BL and AW were supported by the French government, managed by the National Research Agency (ANR), under the France 2030 program with the project references ``ANR-23-IACL-0008'' (PR[AI]RIE-PSAI) and ``ANR-25-CE23-5660'' (MAPLE), as well as the Choose France - CNRS AI Rising Talents program. FK acknowledge funding from the Swiss National Science Foundation grants OperaGOST (grant number $200021$ $200390$) and DSGIANGO (grant number $225837$). This work was supported by the Simons Collaboration on the Physics of Learning and Neural Computation via the Simons Foundation grant ($\#1257412$).

\bibliographystyle{plainnat}
\bibliography{References}

\newpage
\appendix
\section{Further numerics}\label{app:numerics2}

All experiments follow the hierarchical spectral procedure in \Cref{alg:agnostic-recovery} and are detailed more precisely in \ref{subsec:spec_estimator}. This appendix collects the training procedure, defines the metrics used in the figures, and specifies the implementation details needed to reproduce the plots.

\begin{figure}[h]
\centering
\begin{minipage}{0.7\linewidth}

\hrule
\captionof{algorithm}{Hierarchical spectral learning (Training procedure)}

\label{alg:agnostic-recovery}
\vspace{-1em}
\hrule

\begin{algorithmic}[1]

\INPUT Data $\{(\bx_\mu,y_\mu)\}_{\mu=1}^n$, max degree $K_{\max}$

\STATE \textbf{First layer recovery:}
\STATE Compute flattened degree$-q$ features and moment matrix 
\[
\begin{aligned}
    \phi_\mu &= \mathcal F[\He_q(x_\mu)] \;\in \R^{D}\\
     \widehat C &=
  \frac{1}{n}\sum_{\mu=1}^n
  y_\mu \, \He_2( \phi_\mu) \; \in \R^{D \times D}
\end{aligned}
\]

\STATE Compute top eigenvectors $\{\widehat A^{(1)}\}_{i \in [d_1]} \in \R^{D}$
\FOR{$\mu=1$ to $n$ and $i=1$ to $d_1$}
    \STATE $\widehat h^{(1)}_{\mu,i} \leftarrow
    \langle \widehat A^{(1)}_i, \He_{q}(x_\mu)\rangle$
\ENDFOR
\STATE \textbf{Second layer recovery:}
\STATE Compute the 2nd order moment matrix 
\[
\begin{aligned}
     \widehat A^{(2)} &=
  \frac{1}{n}\sum_{\mu=1}^n
  y_\mu \, \He_2( \widehat h^{(1)}_\mu)\; \in \R^{d_1\times d_1}
\end{aligned}
\]
\FOR{$\mu=1$ to $n$}
    \STATE 
    $\widehat h^{(2)}_\mu \leftarrow
    \langle \widehat A^{(2)}, \He_{2}(\widehat h^{(1)}_\mu)\rangle\;\in\R$
\ENDFOR

\STATE Perform kernel regression on $\{(\widehat h^{(2)}_\mu,y_\mu)\}_{\mu=1}^n$ with a new batch of data. 
\RETURN $\widehat A^{(1)}, \widehat A^{(2)}$

\end{algorithmic}
\hrule

\end{minipage}
\end{figure}

\subsection{Metrics and evaluation protocol}
\label{app:metrics}

% In the figures, $\mathrm{MSE}$ denotes the test estimate of the generalization error in \Cref{def:MSE}, computed on an independent test set.

% Given an independent test set, let $H^{(1)},\widehat H^{(1)}\in\mathbb R^{n_{\rm test}\times d_1}$ be the true and learned first-layer feature matrices, and let $Q,\widehat Q$ be orthonormal bases for their column spaces. We define $q_h^{(1)}=d_1^{-1}\|Q^\top \widehat Q\|_F^2$. This quantity lies in $[0,1]$ and is invariant to signs, permutations, and rotations of the learned latent coordinates.

% For direction-wise plots, we write $u_i=A_i^{(1)}/\|A_i^{(1)}\|$ and let $\widehat U$ be an orthonormal basis of the recovered eigenspace. We define $\cos^2(\theta_i)=\|\widehat U^\top u_i\|^2$. The center panel of \Cref{fig:direction-wise-theory} reports $1-\cos(\theta_i)$.

We use three complementary diagnostics in the numerical experiments. The MSE measures end-to-end prediction performance, the first-layer feature overlap measures recovery of the latent representation as a subspace, and the direction-wise cosines isolate the individual spectral transitions predicted by \Cref{thm:sample_complexity}. These quantities answer different questions: the MSE mixes all recovered and unrecovered directions through their weights, the aggregate overlap summarizes representation recovery, and the direction-wise overlaps reveal which latent directions have crossed their spectral threshold.

\paragraph{Generalization error.}
The MSE reported in the figures is the empirical test estimate of the generalization error in \Cref{def:MSE}. Given an independent test set $\{(x_\mu^{\rm test},y_\mu^{\rm test})\}_{\mu=1}^{n_{\rm test}}$, we compute
\begin{equation}
    \widehat{\mathrm{MSE}}
    =
    \frac{1}{n_{\rm test}}
    \sum_{\mu=1}^{n_{\rm test}}
    \left(\widehat f(x_\mu^{\rm test})-y_\mu^{\rm test}\right)^2 .
\end{equation}
This is the final performance metric of the algorithm. The overlap quantities below are diagnostic measures tied to the teacher-student setting. They are used to identify whether changes in MSE are caused by first-layer feature recovery or else. 

\paragraph{First-layer feature overlap.}
Let $H^{(1)},\widehat H^{(1)}\in\mathbb R^{n_{\rm test}\times d_1}$ denote the true and learned first-layer feature matrices evaluated on the same independent test set. Let $Q$ and $\widehat Q$ be orthonormal bases for the column spaces of $H^{(1)}$ and $\widehat H^{(1)}$, respectively. We define
\begin{equation}
    q_h^{(1)}
    =
    \frac{1}{d_1}\|Q^\top \widehat Q\|_F^2 .
    \label{eq:first_layer_feature_overlap}
\end{equation}
Equivalently, $q_h^{(1)}$ is the average squared canonical correlation between the true and learned latent feature spaces. It belongs to $[0,1]$, more precisely, it is equal to one when the learned features span the same subspace as the true features, and close to zero when the two subspaces are nearly orthogonal.

This subspace definition is intentional. The learned coordinates are only meaningful up to signs, permutations, and possible finite-size rotations inside the recovered eigenspace. These transformations can be absorbed by the second-layer fit and should not be counted as representation error.

\paragraph{Direction-wise alignment.}
The aggregate overlap $q_h^{(1)}$ does not show which individual directions have been recovered. To test the feature-wise prediction of \Cref{thm:sample_complexity}, we also measure single direction alignments. Let $u_i=A_i^{(1)}/\|A_i^{(1)}\|$ be the normalized $i$-th teacher direction, and let $\widehat U\in\mathbb R^{D\times d_1}$ be an orthonormal basis of the top eigenspace of $\widehat C$ (returned by the first spectral step). We define:
\begin{equation}
    \cos^2(\theta_i)
    =
    \|\widehat U^\top u_i\|^2
    =
    u_i^\top \widehat U\widehat U^\top u_i .
    \label{eq:direction_wise_cosine}
\end{equation}
This is a projector overlap. It is invariant to the sign of $u_i$ and to the choice of basis inside the recovered eigenspace. The value $\cos^2(\theta_i)\simeq 0$ means that direction $i$ is absent from the recovered subspace, while $\cos^2(\theta_i)\simeq 1$ means that it has been recovered. This is the quantity shown in the left panel of \Cref{fig:direction-wise-theory}, where the different curves turn on sequentially according to the spike sizes $|\lambda_i|=Z_\gamma i^{-\gamma}$.

After a direction has been recovered, we are interested not only in whether it is present, but also in how fast its alignment improves with $n$. For this post-transition regime we write $\cos(\theta_i)=\|\widehat U^\top u_i\|$ and plot the angular error $1-\cos(\theta_i)$. This is the quantity shown in the center panel of \Cref{fig:direction-wise-theory}. We use $1-\cos(\theta_i)$ rather than $1-\cos^2(\theta_i)$ because \Cref{thm:sample_complexity} is stated directly in terms of the absolute eigenvector overlap and predicts
\begin{equation}
    1-\cos(\theta_i)
    =
    O_d\left(\frac{d^q i^{2\gamma}}{nZ_\gamma^2}\right)
\end{equation}
after recovery. 

Finally, when the test features are well conditioned, the aggregate feature overlap can be viewed as a smoothed average of the direction-wise overlaps. Heuristically,
$q_h^{(1)}\approx d_1^{-1}\sum_{i=1}^{d_1}\cos^2(\theta_i)=:m_{\rm th}(\alpha)$, so the global overlap curves in \Cref{fig:several-d,fig:several-gamma} summarize the cascade of individual transitions displayed in \Cref{fig:direction-wise-theory}.

\subsection{Implementation and reproducibility}
\label{app:implementation}

All experiments follow the hierarchical spectral pipeline described in Algorithm~\ref{alg:agnostic-recovery}. 
For each value of the sample exponent $\alpha$, we use $n=\lfloor d^\alpha\rfloor$ training samples and evaluate the MSE and overlaps on an independent test set. 
For the identity readout $g^\star=\mathrm{id}$, we use the scalar estimator $\widehat h^{(2)}$ directly. 
For the nonlinear experiment with $g^\star=\tanh$, we fit a polynomial ridge regressor on $\widehat h^{(2)}$. 
This readout is specified by three hyperparameters: the maximal polynomial degree $r$, the ridge parameter $\rho$, and the kernel regularization $\lambda_{\mathrm{poly}}$.
For the final curves reported in \Cref{fig:non-identity-readout}, we use $(r,\rho,\lambda_{\mathrm{poly}})=(3,10^{-5},10^{-4})$, selected on the grid
\[
r\in\{3,5,7,9\}, \qquad
\rho\in\{10^{-7},10^{-6},10^{-5}\}, \qquad
\lambda_{\mathrm{poly}}\in\{10^{-5},10^{-4},10^{-3}\}.
\]
The full set of parameters used in the numerical figures is summarized below.

\begin{itemize}
    \setlength{\itemsep}{2pt}
    \setlength{\parsep}{0pt}
    \setlength{\topsep}{2pt}
    \setlength{\partopsep}{0pt}

    \item \textbf{\Cref{fig:several-d}, left-center.}
    MSE and $q_h^{(1)}$; $d\in\{120,250,400\}$, $\alpha$ on a grid in $[1.5,3.4]$,
    $q=2$, $\varepsilon=0.5$, $\gamma=0.4$, $g^\star=\mathrm{id}$.
    We use a linear readout on $\widehat h^{(2)}$ and average over $10$ seeds.

    \item \textbf{\Cref{fig:several-d}, right.}
    Spectrum of $\widehat C$; $d=140$, $\alpha=3.5$, $q=2$, $\varepsilon=0.5$,
    $\gamma=0.4$, $g^\star=\mathrm{id}$.
    We plot the full spectrum with no readout fit, for one seed.

    \item \textbf{\Cref{fig:several-gamma}.}
    MSE and $q_h^{(1)}$; $d=400$, $\alpha$ on a grid in $[1.5,3.4]$,
    $q=2$, $\varepsilon=0.5$,
    $\gamma\in\{0,0.2,0.4,0.6,0.8,1.0\}$, $g^\star=\mathrm{id}$.
    We use a linear readout on $\widehat h^{(2)}$ and average over $10$ seeds.

    \item \textbf{\Cref{fig:direction-wise-theory}, left-center.}
    Direction-wise quantities $\cos^2(\theta_i)$ and $1-\cos(\theta_i)$;
    $d=400$, $\alpha$ on a grid in $[1.5,3.4]$, $q=2$, $\varepsilon=0.5$,
    $\gamma=0.4$, $g^\star=\mathrm{id}$.
    We track directions
    $i\in\{1,3,5,8,12,15,18,20\}$ and average over $10$ seeds.

    \item \textbf{\Cref{fig:direction-wise-theory}, right.}
    Aggregate overlap $q_h^{(1)}$ and theoretical count $m_{\rm th}(\alpha)$;
    $d=800$, $\alpha$ on a grid in $[2.0,3.2]$, $q=2$, $\varepsilon=0.5$,
    $\gamma=0.4$, $g^\star=\mathrm{id}$.
    We use a linear readout on $\widehat h^{(2)}$ and average over $10$ seeds.

    \item \textbf{\Cref{fig:non-identity-readout}.}
    MSE and $q_h^{(1)}$; $d\in\{120,250,400\}$, $\alpha$ on a grid in $[1.5,3.4]$,
    $q=2$, $\varepsilon=0.5$, $\gamma=0.4$, $g^\star=\tanh$.
    We use polynomial ridge regression on $\widehat h^{(2)}$, with final choice
    $(r,\rho,\lambda_{\mathrm{poly}})=(3,10^{-5},10^{-4})$, and average over
    $10$ seeds.
\end{itemize}

Concerning resources, all experiments were run on single-GPU workers of an internal compute cluster, using up to $32$ GB of host memory per job. 
Depending on the values of $d$, $\alpha$, and the number of seeds, runtimes ranged from a few minutes for the smallest jobs to several hours for the largest ones; for the heaviest sweeps, jobs were typically submitted with a wall-clock budget of up to $12$--$24$ hours.

\section{Preliminary Results}
\label{appendix:preliminaries}
\subsection{Spectral Estimators}\label{subsec:spec_estimator}

We define the successive spectral method in Algorithm~\ref{alg:agnostic-recovery}.

\subsection{Other Preliminary Results}

\begin{lemma}[Computation of $Z_\gamma$]
From the criterion $\mathrm{Var}[(h^{(2)})^2] = \Theta(1)$, it comes:
\begin{equation}
    Z_\gamma \propto \begin{cases}
        1/\sqrt{d_1} & \text{ if } \gamma = 0\\ 
        d_1^{\gamma-\frac{1}{2}}, & \text{ if } \gamma < \frac{1}{2}, \\ 
        1, & \text{ if } \gamma > \frac{1}{2}. 
    \end{cases}
    \label{eq:C_gamma}
\end{equation}
\label{lemma:normalization_C_gamma}
\end{lemma}
\begin{proof}
\begin{align}
    \mathrm{Var}\left[(h^{(2)})^2\right]  & = \mathrm{Var}\left( \langle A^{(2)}, He_2(h^{(1)}) \rangle \right )\\ 
    & =\EE\left[ \langle A^{(2)}, \He_2(\bh^{(1)})\rangle^2 \right] \\
    &= Z_\gamma^2 \sum_{i=1}^{d_1} i_1^{-2\gamma} \underbrace{\EE\left[ \para{\para{h^{(1)}_i}^2 -1}^2\right]}_{=\Theta(1)} \\
    &=  \Theta \left ( Z_\gamma^2 \sum_{i=1}^{d_1} i_1^{-2\gamma} \right ). 
\end{align}
The sum on the RHS has the following asymptotic behavior: 
\begin{equation}
    \sum_{i=1}^{d_1} i_1^{-2\gamma} = \begin{cases}
        d_1 & \text{ if } \gamma = 0 \\ 
        d_1^{1-2\gamma}, &\text{if } 0 \leq \gamma < \frac{1}{2} \\
        \log(d_1), & \text{ if } \gamma = \frac{1}{2} \\
        \Theta(1),  & \text{if } \gamma > \frac{1}{2}. 
    \end{cases}
\end{equation} 
Then, taking $Z_\gamma = \Theta(\sqrt{\sum_{i=1}^{d_1} i^{-2\gamma}})$ concludes the result. 
\end{proof}

\subsection{Auxiliary Concentration Lemmas}
\begin{lemma}[Lemma F.4 in \cite{wen2025does}]
    with probability at least $1 - e^{-cd}$, 
    \begin{equation}
        \| F_\mu \|^{2}_2 \leq C d^{q},
    \end{equation}
    for a universal constant $C$. 
    \label{lemma:bound_norm_hermite}
\end{lemma}

\begin{lemma}[Corollary 5.21 in \cite{BLM13}]
Let $f(x) = \sum_{i=0}^{k} a_i x^i$ be a polynomial of degree $k$ of a real variable and let $X$ be a standard normal random variable. Then for any $q > 2$,
$$
\left(E\left[|f(X)|^q\right]\right)^{1/q} \leq (q-1)^{k/2} \left(E\left[|f(X)|^2\right]\right)^{1/2}.
$$
\label{lemma:gaussian_hypercontractivity}
\end{lemma}

\begin{lemma}[Gaussian Poincaré Inequality, Theorem 3.20 in \cite{BLM13}]Let $X = (X_1, \ldots, X_d)$ be a vector of i.i.d. standard Gaussian random variables (i.e., $X$ is a Gaussian vector with zero mean vector and identity covariance matrix). Let $f : \mathbb{R}^d \to \mathbb{R}$ be any continuously differentiable function. Then
$$
\mathrm{Var}(f(X)) \leq E\left[\|\nabla f(X)\|^2\right] .
$$
\label{lemma:gaussian_poincare}
\end{lemma}

\section{Detection of outliers}

Recall that the algorithm described in \Cref{alg:agnostic-recovery} works by first computing the matrix: 
\begin{equation}
    \hat{C} = \dfrac{1}{n} \sum_{\mu =1}^{n} y_\mu \He_2(\cF[\He_{q}(x_\mu)]). 
\end{equation}
To simplify the our notation, let $F_\mu = \cF [\He_q(x_\mu)] \in \RR^{D}$, where $D = B(d,k) = \binom{d + k - 1}{k-1} = \Theta_{d}(d^{q})$. With this notation, 
\[
\hat{C} = \dfrac{1}{n} \sum_{\mu =1 }^{n} y_\mu \left ( F_\mu F_\mu^T - I_{D}\right ). 
\]
We want to study the eigenvectors of this random matrix. For this, we first write: 
\begin{equation}
    \hat{C} = \mathbb{E} \left [ \hat{C}\right ] + \left ( \hat{C} - \mathbb{E} \left [ \hat{C} \right ]\right ). 
\end{equation}
The reader should think of the first term as the signal and the second one as the noise. We begin this section by studying the signal part. 

\subsection{Studying \texorpdfstring{$\mathbb{E}\left[\hat{C} \right ]$}{E[C-hat (1)]}} 

As noted in \cite{tabanelli2026deep}, (and previously in \cite{wang2023learning,wen2025does}, among others) when computing the expectation $\mathbb{E}\left[\hat{C} \right ]$, the vectors $F_\mu$ behave as if they were isotropic Gaussians in $\RR^{D}$. For this reason, we have: 

\begin{lemma}
Let $\gamma >0$, and denote $\lambda_j = Z_\gamma j^{-\gamma}$ and $u_j = A^{(1)}_j$. Then: 
    \[
    \EE \left [C^{(1)} \right ] = \frac{\nu_{1}}{\sqrt{2}}A^{(1)} D_{\gamma} (A^{(1)})^T + \Delta, 
    \]
    where $\| \Delta \|_\op = o_{d}(1)$,  $A^{(1)} = [u_1, \dots u_{d_1}] \in \mathbb{R}^{D \times d_1}$, $D_\gamma = \mathrm{diag} ( \lambda_1, \dots, \lambda_{d_1}) \in \RR^{d_1 \times d_1}$ and $\nu_{1}$ denotes the first Hermite coefficient of $g$. 
    \label{lemma:expectation_C_1}
\end{lemma}

We postpone the Proof of \Cref{lemma:expectation_C_1} to \Cref{section:E[C]_proof}.

\subsection{Eigenvector Perturbation Formula}

As noted before, it will be useful to write: 
\begin{equation}
    \hat{C} = \EE [\hat{C}] + \left (\hat{C} - \EE \left [ C^{(1)}\right ] \right ) = A^{(1)} D_{\gamma} (A^{(1)})^T + \left (\hat{C} - \EE \left [ C^{(1)}\right ] \right ) + \Delta. 
\end{equation}
Since $\| \Delta\|_\op = o_{d}(1)$, we only care about the first two terms of the decomposition above. 

Now, for $z \in \mathbb{C}$, let
\begin{equation}
    R_{\bar{C}}(z) = (z I _{D} - \EE[\hat{C}])^{-1}, R_{\hat{C}}(z) = (zI_{D} - \hat{C})^{-1}. 
\end{equation}
Then: 
\begin{equation}
    R_{\hat{C}}(z) - R_{\bar{C}}(z) = R_{\hat{C}}(z)(\hat{C} - \EE[\hat{C}]) R_{\bar{C}}(z). 
\end{equation}
Denote $\Delta= \hat{C} - \EE [\hat{C}]$. Then: 
\begin{equation}
    zI_{D} - \hat{C} = z I_{D} - \EE[C] - \Delta = ( z I_{D} - \EE[C] )( I_{D} -R_{\bar{C}}(z)\Delta ).
\end{equation}
Therefore: 
\begin{equation}
    R_{\hat{C}}(z) = ( I_{D} -R_{\bar{C}}(z)\Delta )^{-1} R_{\bar{C}}(z). 
    \label{eq:resolvent_for_hat_C}
\end{equation}
If $\| R_{\bar{C}}\Delta \|_\op \leq 1$, then we can expand  $(I_{D} -R_{\bar{C}}(z)\Delta )^{-1}$ into its Neumann series: 
\begin{equation}
    (I_{D} -R_{\bar{C}}(z)\Delta )^{-1} = \sum_{\ell \geq 0} (R_{\bar{C}}(z)\Delta)^{\ell}. 
\end{equation}
Then, going back to \Cref{eq:resolvent_for_hat_C}, we can write: 
\begin{equation}
     R_{\hat{C}}(z) = R_{\bar{C}}(z) + R_{\bar{C}}(z) \Delta R_{\bar{C}}(z) + o(\| \Delta \|_\op^{2}). 
     \label{eq:first_order_expansion_resolvent}
\end{equation}

Let $u_k, \hat{u}_{k}$ be isolated eigenvectors of $\EE[\hat{C}], \hat{C}$, respectively. Let $\gamma$ be a contour around $\lambda_{k}$ and $\hat{\lambda}_{k}$. Then, we can write the projectors $\Pi_{k} = u_k u_k^T$ and $\hat{\Pi}_{k} = \hat{u}_k \hat{u}_k$ as: 
\begin{equation}
    \Pi_{k} = \dfrac{1}{2\pi i}\oint_{\gamma} R_{\bar{C}}(z) dz, \quad  \hat{\Pi}_{k} = \dfrac{1}{2\pi i}\oint_{\gamma} R_{\hat{C}}(z) dz. 
\end{equation}
Then, we can integrate \Cref{eq:first_order_expansion_resolvent} to get: 
\begin{align}
    \hat{\Pi}_{k}\Pi_{k} &= \Pi_{k} + \dfrac{1}{2 \pi i} \oint_{\gamma} R_{\bar{C}}\Delta R_{\bar{C}} + o(\| \Delta\|_\op^2). 
    \label{eq:series_projection_k}
\end{align}
We can re-write the resolvent $R_{\bar{C}}$ in the following way: 
\begin{equation}
    R_{\bar{C}}(z) = (z I - \EE[\hat{C}])^{-1} = \sum_{j=1}^{d_1} \dfrac{1}{z - \lambda_{j}} u_{j}u_j^T + \dfrac{1}{z} P_{\mathrm{Ker}}, 
\end{equation}
where $P_{\mathrm{Ker}} = P_{\mathrm{Ker}(\EE[\hat{C}])}$ denotes the projection into the kernel of $\EE[\hat{C}]$. Then: 
\begin{align}
    R_{\bar{C}}(z)\Delta R_{\bar{C}}(z)  &= \left ( \sum_{j=1}^{d_1} \dfrac{1}{z - \lambda_{j}} u_{j}u_j^T + \dfrac{1}{z} P_{\mathrm{Ker}}\right ) \Delta \left (\sum_{j=1}^{d_1} \dfrac{1}{z - \lambda_{j}} u_{j}u_j^T + \dfrac{1}{z} P_{\mathrm{Ker}} \right ) \\ 
    & = \sum_{j_1,j_2} \dfrac{1}{(z- \lambda_{j_1}) (z - \lambda_{j_2})} u_{j_1}u_{j_1}^T \Delta u_{j_2} u_{j_2}^T + \sum_{j=1}^{d_1} \dfrac{1}{z-\lambda_{j}} u_{j}u_{j}^T \Delta P_{\mathrm{Ker}} \\
    & + \sum_{j=1}^{d_1} \dfrac{1}{z-\lambda_{j}} u_{j}u_{j}^T \Delta P_{\mathrm{Ker}} + \dfrac{1}{z^2}  P_{\mathrm{Ker}}\Delta P_{\mathrm{Ker}}. 
\end{align}
Integrating, we get: 
\begin{align}
    \dfrac{1}{2\pi i }\oint_{\gamma} R_{\bar{C}}(z)\Delta R_{\bar{C}}(z) &=  \dfrac{1}{2\pi i }\oint_{\gamma}\sum_{j_1,j_2} \dfrac{1}{(z- \lambda_{j_1}) (z - \lambda_{j_2})} u_{j_1}u_{j_1}^T \Delta u_{j_2} u_{j_2}^T + \dfrac{1}{2\pi i }\oint_{\gamma}\sum_{j=1}^{d_1} \dfrac{1}{z-\lambda_{j}} u_{j}u_{j}^T \Delta P_{\mathrm{Ker}} \\
    & + \dfrac{1}{2\pi i }\oint_{\gamma}\sum_{j=1}^{d_1} \dfrac{1}{z-\lambda_{j}} u_{j}u_{j}^T \Delta P_{\mathrm{Ker}} + \underbrace{\dfrac{1}{2\pi i }\oint_{\gamma}\dfrac{1}{z^2}  P_{\mathrm{Ker}}\Delta P_{\mathrm{Ker}}}_{=0} \\ 
    & = \sum_{j=1, j\not = k}^{d_1} \dfrac{u_{j}^T \Delta u_{j}}{\lambda_{k} - \lambda_{j}}\left ( u_{j} u_{k}^T + u_{k} u_{j}^T \right ) + \dfrac{1}{\lambda_{k}} u_{k} u_{k}^T \Delta P_{\mathrm{Ker}} +  \dfrac{1}{\lambda_{k}} P_{\mathrm{Ker}} \Delta u_{k} u_{k}^T. 
\end{align}
Then, replacing this in \Cref{eq:series_projection_k}: 
\begin{align}
      \hat{\Pi}_{k}  &= \Pi_{k} + \sum_{j=1, j\not = k}^{d_1} \dfrac{u_{j}^T \Delta u_{j}}{\lambda_{k} - \lambda_{j}}\left ( u_{j} u_{k}^T + u_{k} u_{j}^T \right ) + \dfrac{1}{\lambda_{k}} u_{k} u_{k}^T \Delta P_{\mathrm{Ker}} +  \dfrac{1}{\lambda_{k}} P_{\mathrm{Ker}} \Delta u_{k} u_{k}^T + o(\| \Delta\|_\op^2). 
     \label{eq:expansion_projection_k_2}
\end{align}
In order to get eigenvectors, we apply this projection to $u_{k}$ and obtain: 
\begin{align}
    \hat{\Pi}_{k}u_{k} & =  u_{k} + \left ( \sum_{j=1, j\not = k}^{d_1} \dfrac{u_{j}^T \Delta u_{j}}{\lambda_{k} - \lambda_{j}}\left ( u_{j} u_{k}^T + u_{k} u_{j}^T \right ) + \dfrac{1}{\lambda_{k}} u_{k} u_{k}^T \Delta P_{\mathrm{Ker}} +  \dfrac{1}{\lambda_{k}} P_{\mathrm{Ker}} \Delta u_{k} u_{k}^T + o(\| \Delta\|_\op^2) \right ) u_{k} \\
    & =  u_{k}  + \sum_{j=1, j\not = k}^{d_1} \dfrac{u_{j}^T \Delta u_{j}}{\lambda_{k} - \lambda_{j}} u_{j} + \dfrac{1}{\lambda_{k}}P_{\mathrm{Ker}} \Delta u_{k} + o(\| \Delta\|_\op^2). 
    \label{eq:projection_hat_u_k}
\end{align}
Note that the second and third terms are orthogonal to $u_{k}$, so 
\begin{equation}
    \|\hat{\Pi}_{k} u_{k}\|^{2} = 1 + \left \|\sum_{j=1, j\not = k}^{d_1} \dfrac{u_{j}^T \Delta u_{j}}{\lambda_{k} - \lambda_{j}} u_{j} + \dfrac{1}{\lambda_{k}}P_{\mathrm{Ker}} \Delta u_{k} \right \|^{2}  + o(\| \Delta \|^{2}_\op) ,
\end{equation}
and since the term in the middle is bounded by $C\| \Delta\|^{2}$, we conclude that: 
\begin{equation}
    \|\hat{\Pi}_{k} u_{k}\|^{2} = 1 + o (\| \Delta\|^{2}). 
\end{equation}
Then, we can normalize \Cref{eq:projection_hat_u_k} and we will have: 
\begin{equation}
    \hat{u}_{k} = u_{k }+\sum_{j=1, j\not = k}^{d_1} \dfrac{u_{j}^T \Delta u_{j}}{\lambda_{k} - \lambda_{j}} u_{j} + \dfrac{1}{\lambda_{k}}P_{\mathrm{Ker}} \Delta u_{k} + o(\| \Delta\|_\op^2). 
    \label{eq:eigenvector_perturbation_identity}
\end{equation}
At last, we replace the eigenvectors $u_{k}$ by $A^{(1)}_{k}$, by using 
\Cref{lemma:expectation_C_1}, plus the fact that $A^{(1)}$ are almost the eigenvectors of $\EE \left [ \hat{C}\right ]$. To see this, note that by applying a covariance concentration bound \cite{V10} for $\hat{A}^{(1)}, \dots, \hat{A}^{(1)}$, the condition $\varepsilon < (k - 2\gamma) $ gives that
\begin{equation}
    \left \| \sum_{i=1}^{d_1} \dfrac{1}{\lambda_i} u_i u^T - \sum_{i=1}^{d_1} \dfrac{1}{\lambda_i} A^{(1)}_i (A^{(1)})^T \right | \leq d_1^{\gamma} \sqrt{\dfrac{d_1}{d^q}} = d^{\varepsilon (\frac{1}{2} +2\gamma) - q},
\end{equation}
which is small by our assumption that  $\varepsilon < \sfrac{q}{(1- 2\gamma)}$ to get: 
\begin{equation}
    \hat{u}_{k} = A^{(1)}_{k }+\sum_{j=1, j\not = k}^{d_1} \dfrac{(A^{(1)}_{k})^T \Delta (A^{(1)}_{j})}{\lambda_{k} - \lambda_{j}} (A^{(1)}_{j}) + \dfrac{1}{\lambda_{k}}P_{\mathrm{Ker}} \Delta u_{k} + o(\| \Delta\|_\op^2). 
\end{equation}

\subsection{Analysis of Outliers - Sufficient Sample Complexity}
\label{appendix:proof_thm_1}
Having \Cref{eq:eigenvector_perturbation_identity}, we can study the eigenvectors of $\hat{C}$. In order for the expansion to be valid, we need $\|\Delta\|_\op = \| \hat{C} - \EE[\hat{C}]\|_\op$ to be small. In the following Lemma, we show that this is indeed the case when $n \gg d^{q}$. 

\begin{lemma}
Consider the estimator $\hat{C}$ in \Cref{alg:agnostic-recovery} computed for the Hermite tensor of degree $q$. Then with high probability
\begin{equation}
    \left \| \hat{C} - \EE[\hat{C}]\right \|_\op \lesssim \sqrt{\dfrac{d^{q}}{n}}.
\end{equation}
\label{lemma:matrix_bernstein_c_hat}
\end{lemma}

\begin{proof}
    The proof proceed the same way as \cite{tabanelli2026deep} and \cite{wen2025does}. By Lemma F.4 in \cite{wen2025does}, with probability at least $1 - e^{-cd}$, 
    \begin{equation}
        \| F_\mu \|^{2}_2 \leq C d^{q},
    \end{equation}
    for a universal constant $C$. By truncating the matrix $\hat{C}$ with indicators $\mathbf{1}_{\| \| F_\mu \|^{2}_2 \leq C d^{q}}$, and applying Bernstein's inequality, we get the desired results. 
\end{proof}

Then, by \Cref{lemma:matrix_bernstein_c_hat}, we can apply the expansion \Cref{eq:eigenvector_perturbation_identity} and we can conclude 

\begin{lemma}
    Let $\hat{u}_{k}$ denote the $k$-th eigenvector of $\hat{C}$, and $u_{k}$ the $k$-th eigenvector of $\EE[\hat{C}]$. Denote $\Delta: = \hat{C} - \EE[\hat{C}]$. Then, 
    \[
      \hat{u}_{k} = A^{(1)}_{k }+\sum_{j=1, j\not = k}^{d_1} \dfrac{(A^{(1)}_{k})^T \Delta (A^{(1)}_{j})}{\lambda_{k} - \lambda_{j}} (A^{(1)}_{j}) + \dfrac{1}{\lambda_{k}}P_{\mathrm{Ker}} \Delta u_{k} + o_d(\| \Delta\|_\op^2) . 
    \]
    where $\| \Delta\|_\op = O \left ( \max\left ( \sqrt{\dfrac{d^{q}\log(d)}{n}}, \dfrac{d_1^{-\gamma}}{Z_\gamma}\sqrt{\dfrac{d_1}{d}} \right )\right )$. 
    \label{lemma:eigenvector_perturbation_identity}
\end{lemma}

\Cref{lemma:eigenvector_perturbation_identity} tells us that we can write the $k$-th eigenvector of $\hat{C}$ as: 
\begin{equation}
    \hat{u}_{k} = A^{(1)}_{k } + \underbrace{\sum_{j=1, j \not = k}^{d_1} \dfrac{(A^{(1)}_{j})^T (\hat{C} - \EE [\hat{C}])A^{(1)}_{k }}{ \lambda_k - \lambda_j } u_j}_{(I)} +\dfrac{1}{\lambda_{k}} P_{\mathrm{Ker}(\EE[\hat{C}])} (\hat{C} - \mathbb{E}[C^{(1)}] ) u_{k} + \Delta, 
    \label{eq:I_II_perturbation}
\end{equation}
for $\| \Delta\|_2 = o_{d}(1)$. Let's focus on $(I)$. We have: 
\begin{align}
    (I) & = \sum_{j=1,j \not = k}^{d_1} \dfrac{u_j^T (\hat{C} - \EE [\hat{C}])u_k}{ \lambda_k - \lambda_j } u_j \\
    & = \sum_{j=1,j \not = k}^{d_1} \dfrac{u_j^T (\dfrac{1}{n}\sum_{\mu = 1}^{n} y_\mu (F_\mu F_\mu^T - I_{D}) - \EE [\hat{C}])u_k}{ \lambda_k - \lambda_j } u_j \\
    & = \sum_{j=1, j \not = k}^{d_1} \frac{\frac{1}{n}\sum_{\mu = 1}^{n} y_\mu h_{\mu,j} h_{\mu, k}}{ \lambda_k - \lambda_j } u_j. 
\end{align}
From Lemma F.4 in \cite{wen2025does}, there exists a universal constant $C$ such that with high probability $\| F_\mu \|_2^2 \leq C d^q$. Then, assuming the eigenvectors of $\EE[\hat{C}]$ are de-localized (in the sense that $\|u_{k}\|_\infty \leq Cd^{-q} $, we will have 
\begin{equation}
    | h_{\mu,j}| = \langle u_j , F_\mu \rangle \leq C. 
\end{equation}
Then: 
\begin{align}
    \| (I) \|_2 &= \left ( \sum_{j \not = k}  \dfrac{1}{(\lambda_k - \lambda_j)^{2}} \left (\frac{1}{n}\sum_{\mu = 1}^{n} y_\mu h_{\mu,j} h_{\mu, k} \right )^{2} \right )^{\frac{1}{2}} \\ 
\end{align}
By \Cref{lemma:bound_square_y_s}, with high probability: 
\begin{equation}
    \left ( \dfrac{Z_\gamma}{n} y_\mu h_{\mu,k} h_{\mu,j} \right )^{2} \lesssim \dfrac{Z_\gamma^2}{d^{q}} \min(k,j)^{-2\gamma}.
\end{equation}
Then: 
\begin{align}
    \| (I) \|_2 & \leq \left (\dfrac{Z_\gamma^2}{d^{q}} \sum_{j \not = k}  \dfrac{\min(k,j)^{-2\gamma}}{(\lambda_k - \lambda_j)^{2}}  \right )^{\frac{1}{2}}.
    \label{eq:bound_I_part_1}
\end{align}
We now focus on the inner sum. Recall that $\lambda_{j} = z_{j} Z_\gamma j^{-\gamma}$, and that $z_j \sim \mathrm{Rad}(\frac{1}{2})$. Then the inner sum is upper bounded by the case where $z_{j} \not = z_{k}$, that is:  
\begin{align}
    \sum_{j \not = k}  \dfrac{\min(k,j)^{-2\gamma}}{(\lambda_k - \lambda_j)^{2}} \leq \dfrac{1}{Z_\gamma^{2}}\sum_{j \not = k}  \dfrac{\min(k,j)^{-2\gamma}}{(k^{-\gamma} - j^{-\gamma})^{2}}. 
\end{align}
By separating the sum according to $j <k$ and $j >k$: 
\begin{align}
     \sum_{j \not = k}  \dfrac{\min(k,j)^{-2\gamma}}{(\lambda_k - \lambda_j)^{2}} &\leq \dfrac{1}{Z_\gamma^{2}}\left ( \sum_{j < k}  \dfrac{j^{-2\gamma}}{(k^{-\gamma} - j^{-\gamma})^{2}}+ \sum_{j > k}  \dfrac{k^{-2\gamma}}{(k^{-\gamma} - j^{-\gamma})^{2}} \right ) \\ 
     & \leq \dfrac{1}{Z_\gamma^{2}}\left ( \sum_{j < k}  \dfrac{1}{\left (\left ( \dfrac{k}{j} \right )^{-\gamma} - 1 \right )^{2}}+ \sum_{j > k}  \dfrac{1}{\left (1 - \dfrac{j}{k}^{-\gamma} \right )^{2}} \right ) 
\end{align}
Note that for $\gamma \in (0,1)$, the function $u \to u^{\gamma}$ is concave. Then, $u^{\gamma} \leq 1+\gamma (u -1) $, and therefore  $1-u^{\gamma} > \gamma(1- u)$. On the other hand, for $\gamma >1$ if $u \in (0,1)$, then $ u ^{\gamma } <u$ and hence $1-u^{\gamma} > 1-u$. Either way, we get the following upper-bound: 
\begin{equation}
    \sum_{j \not = k}  \dfrac{\min(k,j)^{-2\gamma}}{(\lambda_k - \lambda_j)^{2}} \lesssim \dfrac{1}{Z_\gamma^{2}}\left ( \sum_{j < k}  \dfrac{1}{\left (\left ( \dfrac{j}{k} \right )^{\gamma} - 1 \right )^{2}}+ \sum_{j > k}  \dfrac{1}{\left (1 - \left ( \dfrac{k}{j} \right )^{\gamma} \right )^{2}} \right ) 
\end{equation}
Using this fact: 
\begin{align}
    \sum_{j \not = k}  \dfrac{\min(k,j)^{-2\gamma}}{(\lambda_k - \lambda_j)^{2}} &\lesssim \dfrac{1}{Z_\gamma^{2}}\left ( \sum_{j < k}  \dfrac{1}{\left (1-\dfrac{j}{k}  \right )^{2}}+ \sum_{j > k}  \dfrac{1}{\left (1 -  \dfrac{k}{j} \right )^{2}} \right ) \\
    & \lesssim \dfrac{1}{Z_\gamma^{2}} \left ( \sum_{j < k}  \dfrac{k^2}{\left (k-j  \right )^{2}}+ \sum_{j > k}  \dfrac{j^2}{\left (j -  k\right )^{2}} \right ) \\
    & \lesssim  \dfrac{1}{Z_\gamma^{2}} \left ( k^2+ d_1 \right )\lesssim \dfrac{1}{Z_\gamma^{2}}d_1^{2}. 
\end{align}
Then, going back to \Cref{eq:bound_I_part_1}: 
\begin{align}
    \| (I) \|_2 & \leq \left (\dfrac{Z_\gamma^2}{d^{q}} \dfrac{1}{Z_\gamma^{2}}d_1^{2}  \right )^{\frac{1}{2}} \lesssim \left (\dfrac{d_1^{2} }{d^{q}} \right )^{\frac{1}{2}}.
    \label{eq:I_is_small}
\end{align}
Since $d_1 = d^{\varepsilon}$, and $\varepsilon <\frac{q}{2}$, we conclude that $\| I\|_{2} = o_{d}(1)$. From this, we can go back to \Cref{eq:I_II_perturbation}, and applying inner product with $A^{(1)}_{k}$, we will get: 
\begin{align}
 \langle \hat{u}_{k}, A^{(1)}_{k}\rangle  = \| A^{(1)}_{k }\|^{2} + \langle (I), A^{(1)}_k \rangle +  o_{d}(1),    
\end{align}
and by Cauchy-Schwarz and \Cref{eq:I_is_small}: 
\begin{equation}
    \left | \langle (I), A^{(1)}_k \rangle \right | \leq \| \hat{A}^{(1)}_{k}\|_{2} \| (I)\|_{2} = o_{d}(1). 
\end{equation}
where we used the fact that the norm of $\hat{A}^{(1)}_{k}$ concentrates around $1$ by Hanson-Wright Inequality (\cite{vershynin2018high}, Theorem 6.2.2). Then, we conclude: 
\begin{equation}
    \langle \hat{u}_{k}, A^{(1)}_{k}\rangle = 1 + o_{d}(1).
    \label{eq:inner_product_I}
\end{equation}
Denote
\begin{equation}
    \mathrm{Overlap}(a,b) = \dfrac{\langle a, b \rangle }{\| a\| \|b\|}. 
\end{equation}
After normalizing in \Cref{eq:inner_product_I}, we get: 
\begin{align}
   \mathrm{Overlap}(\hat{u}_{k}, A^{(1)}_{k}) & = \dfrac{1}{\sqrt{1 + \| I\|^{2} + \left \| \dfrac{1}{\lambda_{k}} P_{\mathrm{Ker}(\EE[\hat{C}])} (\hat{C} - \mathbb{E}[C^{(1)}] ) u_{k} \right \|^{2} }}
   \label{eq:overlap_after_perturbation}
\end{align}
 By applying Taylor expansion to the function $u \to \dfrac{1}{\sqrt{1 + x^{2}}}$, we get: 
\begin{align}
    1 - \mathrm{Overlap}(\hat{u}_{k}, A^{(1)}_{k}) = 1 - O \left ( \| I\|^{2} + \| II\|^{2}\right ) + O((\| I\|^{2} + \|II\|^{2})^{2}).
\end{align}
By \Cref{eq:I_is_small}, we know that $\| I\|^{2} =o_{d}(1)$, with a bound independent of $n$. Thus, the only thing left to conclude is to show that $\| (II)\|$ goes to zero with $n$, at a rate $\dfrac{1}{n}$, and we can conclude the first part of \Cref{thm:sample_complexity}. Computing $\|(II)\|^{2}$, we get: 
\begin{equation}
    \| (II)\|_2^{2} \leq \dfrac{1}{\lambda_i^{2}}\|\hat{C} - \EE [\hat{C}] \|_\op^{2}.
\end{equation}
By \Cref{lemma:matrix_bernstein_c_hat}, with high probability: 
\begin{equation}
    \|\hat{C} - \EE [\hat{C}] \|_\op \lesssim  \sqrt{\dfrac{d^{q}\mathrm{poly}\log(d)}{n}}. 
\end{equation}
Then: 
\begin{equation}
    \| (II)\|_2^{2}  \lesssim  \dfrac{i^{2\gamma} d^{q}\mathrm{poly}\log(d)}{Z_{\gamma}^{2}n}. 
    \label{eq:II_is_small}
\end{equation}
Then we conclude that, if $n = \Theta(Z_\gamma^{2} i^{2\gamma}d^{k + \delta + \varepsilon} )$, 
\begin{equation}
    \mathrm{Overlap}(\hat{u}_{k}, \hat{A}^{(1)}_{k} \rangle = 1 - o_{d}(1), 
\end{equation}
thus, we recover the direction $\hat{A}^{(1)}_{k}$ with $n = \Theta(Z_\gamma^{2} i^{2\gamma}d^{k + \delta + \varepsilon} )$. As for the rate,  \Cref{eq:II_is_small} tells us that it is controlled by the second term, and decays as $\frac{-1}{n}$, which concludes the first part of \Cref{thm:sample_complexity}.

\subsection{Analysis of Outliers - Necessary Sample Complexity}

\label{appendix:proof_thm_2}
Recall that from \Cref{lemma:eigenvector_perturbation_identity}: 
\begin{equation}
    \hat{u}_{k} = A^{(1)}_{k } + \sum_{j=1, j \not = k}^{d_1} \dfrac{(A^{(1)}_{j })^T (\hat{C} - \EE [\hat{C}])A^{(1)}_{k }}{ \lambda_k - \lambda_j } A^{(1)}_{j} + \dfrac{1}{\lambda_{k}} P_{\mathrm{Ker}(\EE[\hat{C}])} (\hat{C} - \mathbb{E}[C^{(1)}] )u_{k} + o(\| \hat{C} - \mathbb{E}[C^{(1)}] \|^{2}). 
\end{equation}
We will work in the regime where $n \gg d^{k}$, so the residual term is vanishing by \Cref{lemma:matrix_bernstein_c_hat}.In this section, we will prove that the sample complexity we found in the last section is in fact, necessary. For this, we will focus on the last term: 
\begin{equation}
     w = \dfrac{1}{\lambda_{k}} P_{\mathrm{Ker}} (\hat{C} - \mathbb{E}[C^{(1)}] )u_{k}=  \dfrac{1}{\lambda_{k}} P_{\mathrm{Ker}} \hat{C} u_{k}. 
\end{equation}
We will prove that with high probability, $\| w\| = \Theta_{d}(d^{\frac{\delta}{2}})$. For this, we will use to preliminary results. 

\begin{lemma}[Paley-Zigmund Inequality, (\cite{BLM13}, Exercise 2.4)]
Let $Y$ be real, positive random variable, and $\theta \in (0,1)$. Then: 
\begin{equation}
    \mathbb{P} \left ( Y > \theta \EE [Y]\right ) \geq (1-\theta) ^{2} \dfrac{\EE[Y^2]}{\EE[Y]^{2}} . 
\end{equation}
\label{lemma:paley_zigmund}
\end{lemma}
The Paley-Zigmund inequality will give us a lower bound as long as we can bound the moments of a particular random variable. In our case, the random variables will be polynomials of Gaussians, so we will use Gaussian hypercontractivity. 
\begin{lemma}[Gaussian Hypercontractivity, \cite{janson1997gaussian} Theorem 5.8] Let $X$ be a $N$ degree polynomial of $m$ Gaussian random variables. Then  
\[
\EE \left [ |X|^{p}\right ]^{\frac{1}{p}} \leq C(p,N) \EE \left [ |X|^{2}\right]^{\frac{1}{2}},
\]
for all $1 < p < \infty$. 
\label{lemma:gaussian_hypercontractivity_alignment}
\end{lemma}

Having \Cref{lemma:paley_zigmund} and \Cref{lemma:gaussian_hypercontractivity_alignment}, we are ready to proceed with the proof of the second part \Cref{thm:sample_complexity}. First, we identify the random variable to which we will apply Paley-Zigmund inequality, which is $\|w \|$. 

\textbf{Step 1: Defining the key-quantities} By using the definition of $\hat{C}$ we have: 
\begin{align}
    w &=  \dfrac{1}{n} \sum_{\mu =1}^{n}\dfrac{y_\mu}{\lambda_{k}} \langle u_{k}, F_\mu \rangle P_{\mathrm{Ker}} F_{\mu} + \Delta\\
    & = \dfrac{1}{n} \sum_{\mu=1}^{n} Z_{\mu} = \underbrace{\dfrac{1}{n} \sum_{\mu = 1}^{n} (Z_\mu - \EE[Z_\mu])}_{S_n} + \EE[Z_\mu], 
\end{align}
for $\|\Delta\|=o_{d}(1)$. On the other hand: 
\begin{equation}
    \EE \left [ Z_{\mu}\right ] = \dfrac{1}{\lambda_k}\mathrm{P}_\mathrm{Ker}\EE \left [ y_\mu \langle u_k, F_\mu \rangle F_\mu \right ] = \dfrac{1}{\lambda_{k}} P_{\mathrm{Ker}}\EE \left [ \hat{C}\right ] u_k. 
\end{equation}
Applying \Cref{lemma:expectation_C_1}, we have $\EE\left [ \hat{C}\right ] = \nu_1 A^{(1)}D(A^{(1)})^T + \Delta$, for $\| \Delta\|_\op \lesssim \frac{1}{Z_{\gamma}^2}$. Then: 
\begin{equation}
     \|\EE \left [ Z_{\mu}\right ]\| = \left \|\dfrac{1}{\lambda_{k}} P_{\mathrm{Ker}}\left ( \nu_1 A^{(1)}D(A^{(1)})^T  + \Delta \right )u_k\right \| = \left \| \dfrac{1}{\lambda_k} P_\mathrm{Ker} \Delta u_k \right \|\leq \dfrac{\|\Delta\|_\op}{|\lambda_k|} = o_{d}(1), 
     \label{eq:expectaiton_norm_Z_small}
\end{equation}
since for $\gamma < \frac{1}{2}$, we get $\dfrac{Z_{\gamma^2}}{Z_{\gamma} k^{-\gamma}} = \dfrac{k^{-\gamma}}{\sqrt{d^{1-2\gamma}}} = o_{d}(1)$. Then: 
\begin{equation}
    \EE[ \| w\|^2 ] = \EE \left [ \| S_{n}\|^2 \right ] + \EE[Z_\mu]^{2} = \EE \left [ \| S_{n}\|^2 \right ] + o_{d}(1),
    \label{eq:norm_w_squared}
\end{equation}
and also: 
\begin{equation}
    \EE[ \| w\|^4 ] = \EE \left [ \| S_{n}\|^4 \right ]  + o_{d}(1). 
\end{equation}
Then, by applying \ref{lemma:paley_zigmund}:
\begin{align}
    &\mathbb{P}\left ( \| w\|^2 > \theta  \EE\left [ \|w\|^2\right ]\right ) \geq (1-\theta)^2 \dfrac{\EE \left [ \|w\|^4 \right ]}{\EE \left [ \| w\|^2\right ]} \\
    \iff  &\mathbb{P}\left ( \| w\|^2 > \theta  \EE\left [ \|w\|^2\right ]\right ) \geq (1-\theta)^2 \dfrac{\EE \left [ \|S_{n}\|^4 \right ]}{\EE \left [ \| S_{n}\|^2\right ]^2} + o_{d}(1).
    \label{eq:PZ_application_1}
\end{align}
\textbf{Step 2: Controlling the quotient $\frac{\EE \left [ \|S_{n}\|^4 \right ]}{\EE \left [ \| S_{n}\|^2\right ]^2}$} Let's focus on the term $\frac{\EE \left [ \|S_{n}\|^4 \right ] }{\EE \left [ \| S_{n}\|^2\right ]^2}$. Denote $\bar{Z}_\mu = Z_\mu - \EE[Z_\mu]$. We have:  
\begin{align}
    \EE \left [ \| S_{n} \|^{4}\right ] & =  \dfrac{1}{n^4} \sum_{\mu_1, \mu_2,\mu_3,\mu_4 = 1}^{n} \EE \left [\langle \bar{Z}_{\mu_1}, \bar{Z}_{\mu_2}\rangle \langle \bar{Z}_{\mu_3}, \bar{Z}_{\mu_4}\rangle \right].
\end{align}
By independence, we get: 
\begin{align}
    \EE \left [ \| S_{n} \|^{4}\right ] & =  \dfrac{1}{n^3} \EE \left [\|\bar{Z}_\mu \|^4\right] + \dfrac{n-1}{n^3}\EE[\|\bar{Z}_\mu\|^2]^2 + \dfrac{2(n-1)}{n^3} \EE\left [ \langle \bar{Z}_{\mu_1}, \bar{Z}_{\mu_2} \rangle^{2}\right ],
\end{align}
where in the last expectation $\mu_1\not=\mu_2$.  On the other hand:
\begin{align}
    \EE \left [\|S_{n}\|^2\right ] =  \dfrac{1}{n}\EE \| \bar{Z}_{\mu}\|^2. 
\end{align}
Then: 
\begin{align}
     \dfrac{\EE \left [ \| S_{n} \|^{4}\right ]}{\EE \left [ \|S_{n}\|^2\right ]^2} & =  \dfrac{1}{n} \dfrac{\EE \left [\|\bar{Z}_\mu \|^4\right]}{\EE \left [\| \bar{Z}_{\mu}\|^2 \right ]^2} + \dfrac{n-1}{n}\dfrac{\EE[\|\bar{Z}_\mu\|^2]^2}{\EE \left [\| \bar{Z}_{\mu}\|^2 \right ]^2} + \dfrac{2(n-1)}{n} \dfrac{\EE\left [ \langle Z_{\mu_1}, Z_{\mu_2} \rangle^{2}\right ]}{\EE \left [\| \bar{Z}_{\mu}\|^2 \right ]^2} \\
     & = 1-\dfrac{1}{n} +\dfrac{1}{n} \dfrac{\EE \left [\|\bar{Z}_\mu \|^4\right]}{\EE \left [\| \bar{Z}_{\mu}\|^2 \right ]^2} + \dfrac{2(n-1)}{n} \dfrac{\EE\left [ \langle \bar{Z}_{\mu_1}, \bar{Z}_{\mu_2} \rangle^{2}\right ]}{\EE \left [\| \bar{Z}_{\mu}\|^2 \right ]^2}
     \label{eq:exact_computation_4_2_moments}
\end{align}
Since $\| \bar{Z}_\mu \|^2$ is the $L^2$ of a real polynomial, we can apply Gaussian hypercontractivity \Cref{lemma:gaussian_hypercontractivity_alignment} to get:: 
\begin{equation}
    \mathbb{E} \left [ \| \bar{Z}_\mu\|^{4}\right ] \leq  C\EE \left [\|\bar{Z}_\mu\|^{2} \right ]^{2}, 
\label{eq:constant_hypercontractivity}
\end{equation}
so we have: 
\begin{align}
     \dfrac{\EE \left [ \| S_{n} \|^{4}\right ]}{\EE \left [ \|S_{n}\|^2\right ]^2} & = 1-\dfrac{1}{n} + \dfrac{2(n-1)}{n} \dfrac{\EE\left [ \langle \bar{Z}_{\mu_1}, \bar{Z}_{\mu_2} \rangle^{2}\right ]}{\EE \left [\| \bar{Z}_{\mu}\|^2 \right ]^2}.
\end{align}
We now focus on the last term. We first note that: 
\begin{align}
    \EE\left [ \langle \bar{Z}_{\mu_1}, \bar{Z}_{\mu_2} \rangle^{2}\right ] &= \EE\left [ \langle Z_{\mu_1} - \EE[Z_{\mu_1}], Z_{\mu_2} - \EE[Z_{\mu_2}]\rangle^{2}\right ]  \\
    & =\EE_{\mu_2} \EE _{\mu_1}\left [ \left ( \langle Z_{\mu_1} , Z_{\mu_2} - \EE[Z_{\mu_2}]\rangle -  \langle \EE[Z_{\mu_1}], Z_{\mu_2} - \EE[Z_{\mu_2}] \rangle \right )^{2}\right ] \\
    & \leq  \EE_{\mu_2} \EE _{\mu_1}\left [ \left ( \langle Z_{\mu_1} , Z_{\mu_2} - \EE[Z_{\mu_2}]\rangle\right )^{2}\right ] \\
    & \leq \EE\left [ \left ( \langle Z_{\mu_1} , Z_{\mu_2} \rangle\right )^{2}\right ],
\end{align}
where in the last two lines we bounded the conditional expectations (which are in fact conditional variances) by the second moment. To bound $\EE\left [ \left ( \langle Z_{\mu_1} , Z_{\mu_2} \rangle\right )^{2}\right ]$, we proceed by first conditioning in $\mu_2$. We have: 
\begin{align}
    \EE\left [ \langle Z, Z \rangle^{2} |Z_{\mu_2}\right ] & =   \| Z_{\mu_2}\|^2 \EE \left [\left \langle \dfrac{y_{\mu_1}}{\lambda_{k}} \langle u_{k}, F_{\mu_1} \rangle P_{\mathrm{Ker}} F_{\mu_1}, \dfrac{Z_{\mu_2}}{ \|Z_{\mu_2}\|} \right \rangle^2 | Z_{\mu_{2}}\right ] \\
    & = \| Z_{\mu_2}\|^2 \EE \left [\left \langle \dfrac{y_{\mu_1}}{\lambda_{k}} \langle u_{k}, F_{\mu_1} \rangle P_{\mathrm{Ker}} F_{\mu_1}, \dfrac{P_{\mathrm{Ker}} F_{\mu_2}}{ \|P_{\mathrm{Ker}} F_{\mu_2}\|} \right \rangle^2 | Z_{\mu_{2}}\right ].
    \label{eq:mid_step_cross_expectaiton_PZ}
\end{align}
Denote $a = \dfrac{P_{\mathrm{Ker}} F_{\mu_2}}{ \| P_{\mathrm{Ker}} F_{\mu_2}\|}$. Note that for this expectation, $a$ is a fixed norm $1$ vector in $\mathrm{Ker}(\hat{C}^{(1)})$. Then,, by applying Holder's inequality: 
\begin{align}
    \EE \left [\left \langle \dfrac{y_{\mu_1}}{\lambda_{k}} \langle u_{k}, F_{\mu_1} \rangle P_{\mathrm{Ker}} F_{\mu_1}, a\right \rangle^2\right ] & =  \EE \left [\dfrac{y_{\mu_1}^2}{\lambda_{k}^2} \langle u_{k}, F_{\mu_1} \rangle^2 \left \langle P_{\mathrm{Ker}} F_{\mu_1}, a\right \rangle^2\right ] \\
    & \leq \dfrac{1}{\lambda_k^2} \EE[y_{\mu_1}^{6}]^{\frac{1}{3}} \EE \left [ \langle u_{k}, F_{\mu_1} \rangle^6\right ]^{\frac{1}{3}}\EE \left [ \left \langle P_{\mathrm{Ker}} F_{\mu_1}, a\right \rangle^6 \right ]^{\frac{1}{3}}. 
\end{align}
By \Cref{lemma:gaussian_hypercontractivity_alignment}: since $y_\mu$ has finite variance, we have $\EE[y_{\mu_1}^{6}]^{\frac{1}{3}} = \Theta_{d}(1)$ and 
\begin{equation}
\EE \left [ \langle u_{k}, F_{\mu_1} \rangle^6\right ]^{\frac{1}{3}} \leq C \EE \left [ \langle u_{k}, F_{\mu_1} \rangle^2\right ]= C\| u_k\|^2 = C. 
\label{eq:expectation_projection}
\end{equation}
Then: 
\begin{equation}
     \EE \left [\left \langle \dfrac{y_{\mu_1}}{\lambda_{k}} \langle u_{k}, F_{\mu_1} \rangle P_{\mathrm{Ker}} F_{\mu_1}, a\right \rangle^2\right ] \lesssim \dfrac{1}{\lambda_k^2} \EE \left [ \left \langle P_{\mathrm{Ker}} F_{\mu_1}, a\right \rangle^6 \right ]^{\frac{1}{3}}.
\end{equation}
Now, since $P_{\mathrm{Ker}} a = a$, we have: 
\begin{equation}
    \EE \left [ \left \langle P_{\mathrm{Ker}} F_{\mu_1}, a\right \rangle^6 \right ] = \EE \left [ \left \langle F_{\mu_1}, a\right \rangle^6 \right ],
\end{equation}
and by \Cref{eq:expectation_projection} we conclude $\EE \left [ \left \langle P_{\mathrm{Ker}} F_{\mu_1}, a\right \rangle^6 \right ]^{\frac{1}{2}} \lesssim 1$. Replacing in \Cref{eq:mid_step_cross_expectaiton_PZ}:
\begin{equation}
    \EE\left [ \langle Z, Z \rangle^{2} |Z_{\mu_2}\right ] \lesssim  \dfrac{1}{\lambda_k^2}\|Z_{\mu_2}\|^2,
\end{equation}
and therefore: 
\begin{equation}
    \EE\left [ \langle Z, Z \rangle^{2}\right ] \lesssim  \dfrac{1}{\lambda_k^2}\EE \left [ \|Z_{\mu_2}\|^2 \right ].
\end{equation}
Going back to \Cref{eq:exact_computation_4_2_moments}, we conclude: 
\begin{align}
    \dfrac{\EE \left [ \| S_{n} \|^{4}\right ]}{\EE \left [ \|S_{n}\|^2\right ]^2} &= 1-\dfrac{C}{n} + \dfrac{2(n-1)}{n} \dfrac{\EE\left [ \langle \bar{Z}_{\mu_1}, \bar{Z}_{\mu_2} \rangle^{2}\right ]}{\EE \left [\| \bar{Z}_{\mu}\|^2 \right ]^2} \\
    & = 1-\dfrac{C}{n} + \dfrac{2(n-1)}{n} \dfrac{1}{\EE \left [\| \bar{Z}_{\mu}\|^2 \right ]}.
    \label{eq:4_2_moments_part_2}
\end{align}
\textbf{Step 3: Control of the norm: }
The only thing left is to compute $\EE \left [ \| \bar{Z}_\mu\|^2\right ]$. First, by \Cref{eq:expectaiton_norm_Z_small}: 
\begin{equation}
    \EE[\|\bar{Z}_\mu\|^{2}] = \EE[\|Z_\mu\|^{2}] + o_{d}(1). 
\end{equation}
Then, we can just compute $\EE[\|Z_\mu\|^{2}]$. We have: 
\begin{align}
    \EE[\|\bar{Z}_\mu\|^{2}] & =  \dfrac{1}{\lambda_k^2}\mathbb{E} \left [ y_\mu^{2} \langle u_{k}, F_\mu\rangle^{2} \| P_{\mathrm{Ker}}F_{\mu }\|^{2} \right ]. 
\end{align}
Denote $G_{\mu} = y_\mu^2 \langle u_{k}, F_\mu\rangle^2$. Then: 
\begin{equation}
    \EE[\|Z_\mu\|^{2}] = \dfrac{1}{\lambda_k^2}\mathbb{E} \left [G_\mu \| P_{\mathrm{Ker}}F_{\mu }\|^{2} \right ].
\end{equation}
Note that $\EE[G_\mu] = \Theta(1)$, and we can write $P_\mathrm{Ker}  = I_{D} - P_U$, for $P_U$ the projection into the space spanned by $u_1, \dots, u_{d_1}$. Then: 
\begin{align}
    \EE[\|Z_\mu\|^{2}] & = \dfrac{1}{\lambda_k^2}\mathbb{E} \left [G_\mu \| F_{\mu }\|^{2} \right ] - \dfrac{1}{\lambda_k^2}\mathbb{E} \left [G_\mu \| P_UF_{\mu }\|^{2} \right ] 
\end{align}

Now, define the event $\mathcal{A}:= \{ \| F_\mu \|^2\geq \dfrac{D}{2}\}$. Then: 
\begin{align}
    \mathbb{E} \left [G_\mu \| F_{\mu }\|^{2} \right ] &\geq \mathbb{E} \left [G_\mu \| F_{\mu }\|^{2} \mathbf{1}_{\mathcal{A}}\right ] \\
    & \geq \dfrac{D}{2}\mathbb{E} \left [G_\mu \mathbf{1}_{\mathcal{A}}\right ] = \dfrac{D}{2}\left ( \mathbb{E} \left [G_\mu\right ] -  \mathbb{E} \left [G_\mu \mathbf{1}_{A^{c}}\right ]\right ) \\
    & \geq \dfrac{D}{2}\mathbb{E} \left [G_\mu\right ] + o_{d}(1),
\end{align}
where the last line follows from \Cref{lemma:bound_norm_hermite}. Doing the same for the term $\mathbb{E} \left [G_\mu \| P_UF_{\mu }\|^{2} \right ] $, we get: 
\begin{align}
    \EE\left [ \| Z_{\mu}\|^2 \right ] \geq \mathbb{E} \left [G_\mu \| F_{\mu }\|^{2} \right ] \geq \dfrac{D}{\lambda_{k}^2} + \dfrac{d_1}{\lambda_{k}^2}
\end{align}
Then, going back to \Cref{eq:4_2_moments_part_2}:
\begin{equation}
    \dfrac{\EE \left [ \| S_{n} \|^{4}\right ]}{\EE \left [ \|S_{n}\|^2\right ]^2}  = 1-o_{d}(1). 
\end{equation}

Replacing this in our original Paley-Zigmund inequality, \Cref{eq:PZ_application_1}: 
\begin{align}
& \mathbb{P} \left ( \| w\|^2 > \theta  \EE\left [ \|w\|^2\right ]\right ) \geq (1-\theta)^2 \dfrac{\EE \left [ \|S_{n}\|^4 \right ]}{\EE \left [ \| S_{n}\|^2\right ]^2} + o_{d}(1) \\
\implies & \mathbb{P} \left ( \| w\|^2 > \theta  \EE\left [ \|w\|^2\right ]\right ) \geq (1-\theta)^2 +  o_{d}(1).
\end{align}
Then, with probability at least $(1 - \theta )^2 + o_{d}(1)$: 
\begin{equation}
    \| w\|^2 > \theta  \EE\left [ \|w\|^2\right ]
\end{equation}
From \Cref{eq:norm_w_squared}, this implies that with probability at least $(1 - \theta )^2 + o_{d}(1)$: 
\begin{align}
\| w\|^2 \geq \theta  \EE\left [ \|S_n\|^2 \right ]+ o_{d}(1)
\end{align}
and since: 
\begin{align}
    \EE \left [ \|S_{n}\|^2 \right ] &=  \dfrac{1}{n^2}\EE \left [ \| Z_\mu \|^2\right ] + o_{d}(1) \\ 
    & \geq  \dfrac{1}{n} \left ( \dfrac{D}{\lambda_{k}^2} + \dfrac{d_1}{\lambda_{k}^2} \right )+ o_{d}(1) \\
    & = \Theta_{d} \left ( \dfrac{d^{q}}{n \lambda_{k}^2} \right ),
\end{align}
we have that with probability at least $(1 - \theta )^2 + o_{d}(1)$: 
\begin{align}
\| w\|^2 \geq \theta  \Theta_{d} \left ( \dfrac{d^{q}}{n \lambda_{k}^2} \right ) + o_{d}(1).
\end{align}
Replacing $n = \Theta_{d} \left ( \frac{d^{q}i^{2\gamma} d^{-\delta}}{Z_\gamma} \right )$, and taking $\theta =d^{-\frac{\delta}{2}}$, we conclude that with probability at least $1 - o_{d}(1)$: 
\begin{equation}
    \| w\| \geq \theta \sqrt{\dfrac{d^{q}}{\lambda_k^{2}n}} = \Theta_{d}(d^{\frac{\delta}{2}}). 
\end{equation}
To conclude that the overlap is negligible, recall we had: 
\begin{equation}
    \hat{u}_{k} = u_k + \sum_{j=1, j \not = k}^{d_1} \dfrac{u_j^T (\hat{C} - \EE [\hat{C}])u_k}{ \lambda_k - \lambda_j } u_j + w + o(\| \Delta\|_\op),
\end{equation}
so by the same reasoning of \Cref{eq:overlap_after_perturbation}:
\begin{equation}
    \mathrm{Overlap}(\hat{u}_{k}, A^{(1)}_{k})  = \dfrac{1}{\sqrt{1 + \left \| w \right \|^{2}+ o_{d}(1) }}.
\end{equation}
Then, with high probability, since $\| w\|^{2}$ grows to infinity with $d$: 
\begin{equation}
    \mathrm{Overlap}(\hat{u}_{k}, A^{(1)}_{k}) = o_{d}(1), 
\end{equation}
which concludes the second part of \Cref{thm:sample_complexity}.

\section{Derivation of the rates}
\label{section:derivation_rates}
In this section, we build on \Cref{thm:sample_complexity} and prove \Cref{cor:number_of_directions} in order to then derive rates for the MSE between the teacher and the predictor in \Cref{thm:rates}. 
From these two results, we now that \Cref{alg:agnostic-recovery}  either learns a direction $A_{i}^{(1)}$ or not at all. 

By assumption of \Cref{thm:rates}, the function $g$ is a polynomial and at this stage the features $\hat{h}^{2}_\mu$ in \Cref{alg:agnostic-recovery} are one-dimensional and correspond to a predictor of the second layer of the teacher model.  \Cref{alg:agnostic-recovery} now proceeds with an independent sample $\mathcal{D'} = \{ x'_\mu, y'_\mu \}$. We first transform: 
\begin{equation}
    x'_\mu \to \hat{h}^{(2)}_\mathcal{D} (x'_\mu),
\end{equation}
and then doing KRR on $\{ \hat{h}^{(2)}_{\mathcal{D}}(x'_\mu), y_\mu \}$. We used the notation $h^{(2)}_\mathcal{D}$ to highlight the fact that this predictor is independent of the new data sample. Denote $\hat{f}$ denote the predictor obtained by doing KRR on $\{ \hat{h}^{(2)}_{\mathcal{D}}(x'_\mu), y_\mu \}$. Denote by $\hat{L}(\hat{f})$ the empirical risk, and by $L(\hat{f})$ the population risk. Then, by Theorem 1 in \cite{caponnetto2007optimal}, using the fact that the features are one-dimensional: 
\begin{equation}
    \mathbb{P}_{\mathcal{D'}} \left( L(\hat{f}) - \min_{f \in \mathcal{H}}L(f) > \dfrac{\tau}{n} \right ),
    \label{conditioning_probability}
\end{equation}
goes to $0$ with $n$ for a fixed first sample $\mathcal{D}$. Then: 
\begin{align}
    \mathbb{P}_{\mathcal{D},\mathcal{D}'} \left ( L(\hat{f}) - \min_{f \in \mathcal{H}}L(f) > \dfrac{\tau}{n}\right ) & = \EE_{\mathcal{D}} \left (  \mathbb{P}_{\mathcal{D'}} \left( L(\hat{f}) - \min_{f \in \mathcal{H}}L(f) > \dfrac{\tau}{n} \right )\right ),
\end{align}
which goes to $0$ by \Cref{conditioning_probability}. Then, we conclude that with high probability 
\begin{equation}
    L(\hat{f}) - \min_{f \in \mathcal{H}}L(f) = O(\dfrac{1}{n}),
\end{equation}
and 
\begin{equation}
    L(\hat{f}) \leq \min_{f \in \mathcal{H}}L(f) = O(\dfrac{1}{n}).
\end{equation}
Assuming the kernel is universal, we have $g \in \mathcal{H}$. By further assuming that the  KRR regularization parameter $\lambda$ is optimal, we get: 
\begin{align}
    \min_{f \in \mathcal{H}} \left \{ L(f) \right \}  \leq  \EE\left [ (g(h^{2}_\mu) - g(\hat{h}^{2}_\mu) )^{2}\right ] 
\end{align}
Then, we conclude: 
\begin{align}
\textrm{MSE}(n) & \leq \EE \left [ \left ( g(h^{2}_\mu) - g(\hat{h}^{2}_\mu)\right )^{2}\right ] + O \left ( \dfrac{1}{n}\right ). 
\label{eq:non_linear_MSE}
\end{align}
Now, we focus on the linear term of the subtraction. We have: 
\begin{align}
    \mathrm{MSE}_\mathrm{linear} = \EE \left [ \left \| \sum_{j\geq i^{\star}} j^{-\gamma} \left ( (h_j^{(1)})^2 -1) \right )\right \|^2\right ],
    \label{eq:MSE_1}
\end{align}
where $i^\star$ is the number of learned directions. By \Cref{thm:sample_complexity}, we learn the $i-th$ direction if 
\begin{equation}
    n \asymp \dfrac{d^{q}i^{2\gamma}}{Z_{\gamma}^{2}} \implies i \asymp \left (\dfrac{Z_{\gamma}^{2} n}{d^{q}}\right )^{\frac{1}{2\gamma}}.
\end{equation}
Since directions are learned sequentially, the number of learn directions ( or the last direction that was recovered) at sample complexity $n$ is: 
\begin{equation}
    i^{\star} =  \left (\dfrac{Z_{\gamma}^{2} n}{d^{q}}\right )^{\frac{1}{2\gamma}}. 
    \label{eq:expression_i_star}
\end{equation}
This proves \Cref{cor:number_of_directions}. We will now prove \Cref{thm:rates}. By \Cref{eq:MSE_1}: 
\begin{align}
    \mathrm{MSE}_\mathrm{linear}(n) & = \Theta \left ( Z_{\gamma}^{2}\sum_{i \geq i^\star} j^{-2\gamma} \right ).
\end{align}
We now study this sum according to the value of $\gamma$. If $\gamma < \frac{1}{2}$, then $Z_{\gamma} = (d_1^{1-2\gamma})^{-\frac{1}{2}}$, and we get: 
\begin{align}
   Z_{\gamma}^{2}\sum_{i \geq i^\star} j^{-2\gamma} & = Z_{\gamma}^{2} ( Z_{\gamma}^{-2} - \sum_{i \leq i^\star} i^{-2\gamma}) \\
   & = 1 - \dfrac{1}{d_1^{1-2\gamma}} (i^\star )^{1-2\gamma} \\ 
   & = 1 - \dfrac{1}{d_1^{1-2\gamma}} \left (\dfrac{Z_\gamma^2 n}{d^{q}}\right )^{-1+\frac{ 1}{2\gamma}} \\
   & = 1 - \left (\dfrac{n}{d_1 d^{q}}\right )^{( -1 + \frac{1}{2\gamma})},
\end{align}
which gives the rate for the case where $\gamma < \frac{1}{2} $ and $d^{q} \ll n \ll d^{q}d_1$.

On the other hand, for $\gamma > \frac{1}{2}$, $Z_{\gamma} = \Theta_{d}(1)$ and
\begin{equation}
    \sum_{i \geq i^\star} i^{-2\gamma} = \Theta_{d}((i^\star )^{1-2\gamma}). 
\end{equation}
Then: 
\begin{align}
    \mathrm{MSE}_\mathrm{linear}(n) &  = \Theta_{d} \left ( \sum_{i \geq i^\star} i^{-2\gamma}\right ) \\
    & = \Theta_{d}((i^\star )^{1-2\gamma}) \\
    & = \left (\dfrac{n}{d^{q}}\right )^{\frac{1-2\gamma}{2\gamma}} \\
    & = \left (\dfrac{ n}{d^{q}}\right )^{-1 + \frac{1}{2\gamma}},
\end{align}
and we conclude \Cref{thm:rates} for the linear case. For the non-linear case, we note that all the non-linear terms of \Cref{eq:non_linear_MSE} will be sub-leading with respect to the linear part when $\gamma < \frac{1}{2}$, so we conclude  \Cref{thm:rates} for the non-linear case as well. 

\section{Explicit computations with Weiner Chaos}
\label{section:explicit_computation_hermite}

\subsection{Wiener Chaos properties}

This section will only overview the necessary concepts we need from Wiener chaos expansions. This results are based on \cite{peccati2011wiener}, \cite{nourdin2012normal} and \cite{wen2025does}.   \\

For $A\in\RR^{B(d,q)}$, we define 
\begin{align}
    I_q(A)=\langle A,\mathcal F(\He_q(x))\rangle\, ,
\end{align}
where for $\beta \in \ZZ^{d}_{\geq0}$ with $|\beta|=q$
\begin{equation}
    (\mathrm{He}_q(x))_{\beta} = \He_{\beta}(x). 
\end{equation}
This way, the first layer coefficients can be written as: 
\begin{equation}
    h^{(1)}_{\mu,i} = I_{q}(A^{(1)}_{i}). 
\end{equation}

We will extensively use the orthogonality of Hermite polynomials. 

\begin{lemma}[Orthogonality of Different Chaos] Let $q, q' \in \NN$, $A \in (\RR^{d})^{\odot q}, B \in (\RR^{d})^{\odot q'}$. Then: 
$$
\EE \left[ I_{q}(A) I_{q'}(B)\right ] = \mathbf{1}_{q=q'} \langle A, B \rangle. 
$$
\label{lemma:orthogonality_wiener_chaos}
\end{lemma}

The space spanned by random variables in the $k$-th Wiener chaos is denoted by $\mathcal{H}_{k}$. We will also need the orthogonal projection into the $k$-Wiener chaos, which we denote by $J_{k}: L^{2} \to \mathcal{H}_{k}$. \\ 

We define the Malliavin derivative $D : \mathrm{dom}(D) \to (L^2)^d$ by
$$
DF = (\partial_{x_1} F, \ldots, \partial_{x_d} F), \qquad \mathrm{dom}(D) = \left\{ F \in L^2 : \sum_{k=0}^{\infty} k \|J_k(F)\|_{L^2(\mathcal{G})}^2 < \infty \right\}
$$
where, for smooth functions $F$ with compact support, $\partial_{x_j}$ is the usual partial derivative of $F(\boldsymbol{x}) = F(x_1, \ldots, x_d)$ in the variable $x_j$, and this is extended by completion to $\mathrm{dom}(D)$ . We also define the Ornstein-Uhlenbeck infinitesimal generator $L : \mathrm{dom}(L) \to L^2(\mathcal{G})$ by
$$
LF = \sum_{k=0}^{\infty} -k \, J_k(F), \qquad \mathrm{dom}(L) = \left\{ F \in L^2(\mathcal{G}) : \sum_{k=0}^{\infty} k^2 \|J_k(F)\|_{L^2(\mathcal{G})}^2 < \infty \right\}. 
$$
and we define its inverse $L^{-1}F = \sum_{k=1}^{\infty} -\frac{1}{k} J_k(F)$ whenever $J_0(F) = \mathbb{E}F(x) = 0$.

We will need the following rules to compute the different terms that appears. 
\begin{lemma}[Product formula, \cite{nourdin2012normal}]
 For any $k, \ell \ge 1$, $S \in (\mathbb{R}^d)^{\odot k}$, and $T \in (\mathbb{R}^d)^{\odot \ell}$, 
 $$
 I_k(S) I_\ell(T) = \sum_{r=0}^{\min(k,\ell)} r! \binom{k}{r} \binom{\ell}{r} I_{k+\ell-2r}(S \tilde{\otimes}_r T).
 $$
\label{lemma:product_formula}   
\end{lemma}

We will also need the following rule for computing product of derivatives. 
\begin{lemma} For any $k, \ell \geq 1$, $S \in (\RR^{d})^{\odot k}, T \in (\RR^{d})^{\odot \ell }$, 
    \[
    (DI_k(S))^{\mathrm{T}} (DI_\ell(T)) = k\ell \sum_{r=1}^{\min(k,\ell)} (r-1)! \binom{k-1}{r-1} \binom{\ell-1}{r-1} I_{k+\ell-2r} (S \tilde{\otimes}_r T)
    \]
    \label{lemma:computation_derivatives}
\end{lemma}

For some functions, specially polynomials, the following Gaussian Integration by Parts Lemma will be very useful.

\begin{lemma}[Theorem 2.9.1 in \cite{nourdin2012normal}, Gaussian Integration by Parts] Let $F, G \in \mathbb{D}^{1,2}$, and let $g : \mathbb{R} \to \mathbb{R}$ be a $C^1$ function having a bounded derivative. Then
$$
\EE[Fg(G)] = \EE[F]\EE[g(G)] + \EE[g'(G)\langle DG, -DL^{-1}F \rangle_{\mathfrak{H}}]. \quad
$$
\label{lemma:gaussian_ibp}
\end{lemma}

\subsection{Computing Expectations with Malliavin Calculus}
To compute expectations, we will need: 

\begin{remark}
    As the discussion in \cite{nourdin2012normal}, Page 31 notes, the conditions under which \Cref{lemma:gaussian_ibp} hold are not optimal. In particular, it remains true if $g$  is a polynomial. 
\end{remark}
Denote the second layer by $h^{(2)} = \sum_{i=1}^{d_1} \lambda_i (I_{q}(A_i)^{2} - 1)$, and let $g:\RR \to \RR$ be a polynomial. 
In this section, we will study how to compute expectations of the form: 
\begin{equation}
    \EE \left [ g(h^{(2)}) I_{q}(A_j) I_{q}(A_k)\right ],
\end{equation}
for $j \not = k$. We are interested in how this quantity scales with $d$, as sharp as possible. In particular, we will avoid using Gaussian Approximations when is not absolutely necessary. 

\subsection{Linear Case}
To get some intuition, we will first study the case $g (u)= u$. We want to compute:
\begin{equation}
    E_\mathrm{lin} := \EE \left [ S I_{q}(A_j) I_{q}(A_k)\right ]  = \sum_{i=1}^{d_1} \lambda_{i} \EE \left [ (I_{q}(A_i)^{2} -1) I_{q}(A_j) I_{q}(A_k)\right ]. 
\end{equation}
By \Cref{lemma:product_formula}: 
\begin{align}
    I_{q}(A_i)^{2} -1 &= \sum_{r=0}^{q} c_{q,r} I_{2q-2r}(A_{i} \tilde{\otimes}_{r}A_{i}) - 1 \\
    & = \sum_{r=0}^{q-1} c_{q,r} I_{2q-2r}(A_{i} \tilde{\otimes}_{r}A_{i})  + (\| A_{i}\|_{2}^{2}- 1 ),
\end{align}
where $c_{q,r} = r! \binom{q}{r}^{2}$
Analogously
\begin{align}
I_{q}(A_j) I_{q}(A_k) = \sum_{r=0}^{q} c_{q,r} I_{2q-2r}(A_{j} \tilde{\otimes}_{r}A_k) = \sum_{r=0}^{q-1}c_{q,r} I_{2q-2r}(A_{j} \tilde{\otimes}_{r}A_k)  + q!\langle A_{j},A_{k}\rangle . 
\end{align}
Then, by the orthogonality of different chaos: 
\begin{align}
     E_\mathrm{lin} & = \sum_{i=1}^{d_1} \lambda_{i} \EE \left [ (I_{q}(A_i)^{2} -1) I_{q}(A_j) I_{q}(A_k)\right ]  \\
     & = \sum_{i=1}^{d_1} \lambda_{i} \EE \left [ \left (\sum_{r=0}^{q-1} c_{q,r} I_{2q-2r}(A_{i} \tilde{\otimes}_{r}A_{i})  + (\| A_{i}\|_{2}^{2}- 1 )\right ) \left (  \sum_{r=0}^{q-1}c_{q,r} I_{2q-2r}(A_{j} \tilde{\otimes}_{r}A_k)  + q!\langle A_{j},A_{k}\rangle\right )\right ] \\
     & = \sum_{i=1}^{d_1}  \lambda_i \sum_{r=0}^{q-1}c_{q,r}^{2} \langle A_{i} \tilde{\otimes}_{r} A_i, A_j \tilde{\otimes_{r}} A_k \rangle + \sum_{i=1}^{d} \lambda c_{q,q}^{2} (\| A_i \|^2 -1) \langle A_j, A_k \rangle \\
     & = \sum_{i=1, i \not \in \{ j,k\}}^{d_1}  \lambda_i c_{q,r}^{2} \sum_{r=0}^{q-1}\langle A_{i} \tilde{\otimes}_{r} A_i, A_j \tilde{\otimes_{r}} A_k \rangle + \sum_{i \in \{j,k \}}^{d_1}  \lambda_i \sum_{r=0}^{q-1} c_{q,r}^{2} \langle A_{i} \tilde{\otimes}_{r} A_i, A_j \tilde{\otimes_{r}} A_k \rangle + \sum_{i=1}^{d} \lambda c_{q,q}^{2} (\| A_i \|^2 -1) \langle A_j, A_k \rangle,
     \label{eq:mid_step_M_linear}
\end{align}
were in the last line we split the sum according to wether $i \in \{j,k\}$ or not. Recall we assume that $A_i \in (\RR^{d})^{\odot q}$ have independent, centered gaussian entries with variance $\frac{1}{d^{q}}$. Let: 
\begin{equation}
    T^{r}_{i,j,k}  =  \langle A_{i} \tilde{\otimes}_{r} A_i, A_j \tilde{\otimes_{r}} A_k \rangle. 
\end{equation}
For $u, v \in [d]^{q-r}$ and $r \in \{0\} \cup[q-1]$we have: 
\begin{equation}
     (A_{i} \tilde{\otimes}_{r} A_i)_{u,v} =  \sum_{\ell \in [d]^{r}} A_{i}[u, \ell] A_i [v, \ell ], \quad \text{ and } (A_{j} \tilde{\otimes}_{r} A_k)_{u,v} =  \sum_{\ell \in [d]^{r}} A_{j}[u, \ell] A_k [v, \ell ]. 
\end{equation}
Then: 
\begin{align}
    T^{r}_{i,j,k} & = \sum_{u,v \in [d]^{q-r}}  (A_{i} \tilde{\otimes}_{r} A_i)_{u,v} (A_{j} \tilde{\otimes}_{r} A_k)_{u,v} \\
    & = \sum_{u,v \in [d]^{q-r}}  \sum_{\ell_1 \in [d]^{r}} A_{i}[u, \ell_1] A_i [v, \ell_1 ] \sum_{\ell_2 \in [d]^{r}} A_{j}[u, \ell_2] A_k [v, \ell_2 ] \\
    & = \sum_{u,v \in [d]^{q-r}} \sum_{\ell_2 \in [d]^{r}} \left ( \sum_{\ell_1 \in [d]^{r}} A_{i}[u, \ell_1] A_i [v, \ell_1 ] \right )A_{j}[u, \ell_2] A_k [v, \ell_2 ]. 
\end{align}
If $i \not = j\not = k$, then we have: 
\begin{equation}
    \EE_{A_j} \left [ T_{r} | A_i, A_k\right ] = 0 . 
\end{equation}
On the other hand, the conditional variance equals: 
\begin{equation}
    \mathrm{Var} \left ( T_{r} | A_i, A_k\right) = \dfrac{1}{d^{q}} \sum_{u,v \in [d]^{q-r}} \left ( \sum_{\ell_2 \in [d]^{r}} \left ( \sum_{\ell_1 \in [d]^{r}} A_{i}[u, \ell_1] A_i [v, \ell_1 ] \right )A_{k}[u, \ell_2]\right )^{2}.
\end{equation}
Taking expectation with respect to $A_{k}$: 
\begin{align}
    \EE_{A_j} \left [ \mathrm{Var} \left ( T_{r} | A_i, A_k\right)\right] = \dfrac{d^{r}}{d^{2q}}\sum_{u,v \in [d]^{q-r}}\left ( \sum_{\ell_1 \in [d]^{r}} A_{i}[u, \ell_1] A_i [v, \ell_1 ] \right )^{2} = \dfrac{d^{r}}{d^{2q}} \| A_i \otimes_{r} A_i \|_{2}^{2}. 
\end{align}
And taking expectation with respect to $A_i$, we have: 
\begin{align}
    \EE_{A_i} \left [ \| A_i \otimes_{r} A_i \|_{2}^{2}\right] &=  \EE_{A_i} \left [ \sum_{u,v \in [d]^{q-r}}\left ( \sum_{\ell_1 \in [d]^{r}} A_{i}[u, \ell_1] A_i [v, \ell_1 ] \right )^{2}\right ] \\
    & = \sum_{u,v \in [d]^{q-r}} \EE_{A_i} \left [ \left (\sum_{\ell_1 \in [d]^{r}} A_{i}[u, \ell_1] A_i [v, \ell_1 ] \right )^{2} \right ] \\
    & =O\left ( \sum_{u \in [d]^{q-r}} \frac{1}{d^{q}}\right )  = O_{d}(d^{-r}).
    \label{eq:frobenius_norm_A_r_A}
\end{align}
Applying the Law of total variance, using the fact that for distinct $i,j,k$, $T_{i,j,k}$ is centered: 
\begin{equation}
    \mathrm{Var}_{A_i, A_j, A_k} (T_{i,j,k}) = O \left (\dfrac{d^{r}}{d^{2q}d^{r}}\right ) = O_{d}\left ( \dfrac{1}{d^{2q}}\right ).  
\end{equation}
Then, by applying Chebyshev Inequality we get that with high probability with respect to $A_i, A_j$ and $A_k$: 
\begin{equation}
    \left |T^{r}_{i,j,k} \right | = O_d \left ( \dfrac{1}{d^{q}}\right ), \text{ when } i \not = j \not = k. 
    \label{eq:order_t_i_j_k}
\end{equation}
We now move to the harder case where $i = j$ ( the case $i = k$ is analogous). First, by definition: 
\begin{equation}
    T^r_{i,i,j} = \sum_{u,v \in [d]^{q-r}} \sum_{\ell_2 \in [d]^{r}} \left ( \sum_{\ell_1 \in [d]^{r}} A_{i}[u, \ell_1] A_i [v, \ell_1 ] \right )A_{i}[u, \ell_2] A_j [v, \ell_2 ].
\end{equation}
Fixing $A_i$: 
\begin{equation}
    \EE_{A_j}[T^{r}_{i,i,j}] = 0. 
\end{equation}
The conditional variance given $A_i$ equals: 
\begin{equation}
    \mathrm{Var} (T^{r}_{i,i,j}|A_i)  = \dfrac{1}{d^{q}}\sum_{u,v \in [d]^{q-r}} \sum_{\ell_2 \in [d]^{r}}\left ( \left ( \sum_{\ell_1 \in [d]^{r}} A_{i}[u, \ell_1] A_i [v, \ell_1 ] \right )A_{i}[u, \ell_2] \right )^{2}. 
\end{equation}
Denote $w(u,v) = \sum_{\ell_1 \in [d]^{r}} A_{i}[u, \ell_1] A_i [v, \ell_1 ]$. Then:
\begin{align}
   \sum_{u,v \in [d]^{q-r}} \sum_{\ell_2 \in [d]^{r}}\left ( \left ( \sum_{\ell_1 \in [d]^{r}} A_{i}[u, \ell_1] A_i [v, \ell_1 ] \right )A_{i}[u, \ell_2] \right )^{2} & = \sum_{v \in [d]^{q-r}} \sum_{\ell_2 \in [d]^{r}}\sum_{u \in [d]^{q-r}}w(u,v)^{2}A_{i}[u, \ell_2]^{2} \\
   & \leq \sum_{v \in [d]^{q-r}} \sum_{\ell_2 \in [d]^{r}} \left ( \sum_{u \in [d]^{q-r}}w(u,v)^{2} \right ) \left (  \sum_{u \in [d]^{q-r}}A_{i}[u, \ell_2]^{2} \right ),
\end{align}
where in the last line we applied Cauchy-Schwarz. Now, taking expectation and then applying Cauchy-Schwarz again: 
\begin{align}
    \EE_{A_i} \left [ \mathrm{Var} (T^{r}_{i,i,j}|A_i)\right ] & \leq  \sum_{v \in [d]^{q-r}} \sum_{\ell_2 \in [d]^{r}} \EE_{A_i} \left [\left ( \sum_{u \in [d]^{q-r}}w(u,v)^{2} \right ) \left (  \sum_{u \in [d]^{q-r}}A_{i}[u, \ell_2]^{2} \right ) \right ] \\
    & \leq \sum_{v \in [d]^{q-r}} \sum_{\ell_2 \in [d]^{r}} \EE_{A_i} \left [\left ( \sum_{u \in [d]^{q-r}}w(u,v)^{2} \right )^{2}\right ]^{\frac{1}{2}} \EE \left [ \left (  \sum_{u \in [d]^{q-r}}A_{i}[u, \ell_2]^{2} \right )^{2} \right ]^{\frac{1}{2}}.
\end{align}
By the equivalence of norms for polynomials (\cite{janson1997gaussian}, Theorem 3.50): 
\begin{equation}
    \EE_{A_i} \left [\left ( \sum_{u \in [d]^{q-r}}w(u,v)^{2} \right )^{2}\right ] \leq C \EE_{A_i} \left [\left ( \sum_{u \in [d]^{q-r}}w(u,v)^{2} \right )\right ]^{2},
\end{equation}
and the same holds for $\EE \left [ \left (  \sum_{u \in [d]^{q-r}}A_{i}[u, \ell_2]^{2} \right )^{2} \right ]$. Hence: 
\begin{align}
    \EE_{A_i} \left [ \mathrm{Var} (T^{r}_{i,i,j}|A_i)\right ] & \leq C \sum_{v \in [d]^{q-r}} \sum_{\ell_2 \in [d]^{r}} \EE_{A_i} \left [\left ( \sum_{u \in [d]^{q-r}}w(u,v)^{2} \right )\right ] \EE \left [ \left (  \sum_{u \in [d]^{q-r}}A_{i}[u, \ell_2]^{2} \right )\right ] \\
    & \leq C  \sum_{\ell_2 \in [d]^{r}} \EE_{A_i} \left [\left ( \sum_{v \in [d]^{q-r}}\sum_{u \in [d]^{q-r}}w(u,v)^{2} \right )\right ] \EE \left [ \left (  \sum_{u \in [d]^{q-r}}A_{i}[u, \ell_2]^{2} \right )\right ]
\end{align}
By \Cref{eq:frobenius_norm_A_r_A}: 
\begin{align}
    \EE \left [\sum_{v \in [d]^{q-r}}\sum_{u \in [d]^{q-r}}w(u,v)^{2} \right ] &=   \sum_{u \in [d]^{q-r}} \EE \left ( \sum_{\ell_1 \in [d]^{r}} A_{i}[u, \ell_1] A_i [v, \ell_1 ]\right )^{2} \\ 
    & = O\left ( \dfrac{1}{d^{r}}\right ), 
\end{align}
and 
\begin{equation}
    \EE \left [ \left (  \sum_{u \in [d]^{q-r}}A_{i}[u, \ell_2]^{2} \right )\right ] = O\left ( \dfrac{d^{q-r}}{d^{q}}\right ) = O\left ( \dfrac{1}{d^{r}}\right ). 
\end{equation}
Then: 
\begin{align}
     \EE_{A_i} \left [ \mathrm{Var} (T^{r}_{i,i,j}|A_i)\right ] 
     \leq \dfrac{C}{d^{q}} \sum_{\ell_2 \in [d]^{r}}  \dfrac{1}{d^{r}}\dfrac{1}{d^{r}} = O\left (\dfrac{1}{d^{q+r}}\right)
\end{align}
Hence, by the Law of total variance: 
\begin{align}
    \mathrm{Var}_{A_i, A_j} \left (T^{r}_{i,i,j}\right ) & = O\left (  \dfrac{1}{d^{q+r}} \right ), 
\end{align}
and applying Chebyshev Inequality we get that with high probability over $A_i, A_j$:
\begin{equation}
    \left | T^r_{i,i,j} \right | = O \left (  \dfrac{1}{d^\frac{{q+r}}{2}}\right ). 
    \label{eq:order_t_i_i_j}
\end{equation}
Replacing \Cref{eq:order_t_i_j_k} and \Cref{eq:order_t_i_i_j} in 
\Cref{eq:mid_step_M_linear}: 
\begin{align}
    \left | E_\mathrm{lin} \right | &= \sum_{i=1, i \not \in\{ j,k\}}^{d_1}  \lambda_i \sum_{r=0}^{q-1}c_{q,r}^{2} T^{r}_{i,j,k}+ \sum_{i \in \{j,k \}}^{d_1}  \lambda_i \sum_{r=0}^{q-1}c_{q,r}^{2} T^{r}_{i,j,k} + \sum_{i=1}^{d} \lambda_i c_{q,q}^{2} (\| A_i \|^2 -1) \langle A_j, A_k \rangle \\
    & = O \left ( \dfrac{1}{d^{q}}\sum_{i=1, i \not \in \{ j,k\}}^{d_1}  \lambda_i \right ) + O \left ( \dfrac{1}{d^{q}}\sum_{i\in \{ j,k\}}^{d_1}  \lambda_i \right ) + \sum_{i=1}^{d} \lambda_i c_{q,q}^{2} (\| A_i \|^2 -1) \langle A_j, A_k \rangle. 
\end{align}
For the last term, applying Bernstein's inequality (\cite{vershynin2018high}, Theorem 2.9.1) for the cross inner product and Hanson-Wright (\cite{vershynin2018high}, Theorem 6.2.2) for the norm, we get that with high probability: 
\begin{equation}
    (\| A_i \|^2 -1) \langle A_j, A_k \rangle = O \left(\dfrac{1}{d^{q}} \right ).
\end{equation}
Finally, recalling that $\lambda_i = Z_\gamma z_i i^{-\gamma}$, with $z_i \sim \mathrm{Rad}(\frac{1}{2})$, we can apply Bernstein's inequality over the Radamacher variables. We get that, with high probability over the $z_i$: 
\begin{align}
    \left | E_\mathrm{lin} \right | =O \left ( \dfrac{Z_\gamma}{d^{q}}\sqrt{\sum_{i=1, i \not \in \{ j,k\}}^{d_1}  i^{-2\gamma}} \right ) + O \left ( \dfrac{Z_\gamma}{d^{\frac{q}{2}}}\sqrt{\sum_{i\in \{ j,k\}}^{d_1}  i^{-2\gamma}} \right ) + \dfrac{Z_\gamma}{d^{q}}\sqrt{\sum_{i=1}^{d} i^{-2\gamma} }. 
\end{align}
By definition $Z_\gamma = (\sum_{i=1}^{d_1} i^{-2\gamma})$, so we finally conclude that with very high probability: 
\begin{equation}
    \left | E_\mathrm{lin} \right | = O \left ( \dfrac{Z_\gamma}{d^{q}}(j^{\gamma} + k^{\gamma}))\right ). 
\end{equation}
Thus, we have proved the following Lemma. 
\begin{lemma}
    Let $S = \sum_{i=1}^{d_1} \lambda_i (I_{q}(A_i)^{2} -1)$. Then, given $i, j \in [d_1]$ with $i \not = j$: 
    \[
     \left | \EE \left [ S I_{q}(A_j) I_{q}(A_k)\right ] \right | = O \left ( \dfrac{Z_\gamma}{d^{\frac{q}{2}}}(j^{\gamma} + k^{\gamma}))\right ). 
    \]
    \label{lemma:linear_case_cross_expectations}
\end{lemma}

\subsection{The non-linear case}
\label{lemma:non_linear_part_expectations}

Let $g:\RR \to \RR$ be a polynomial of degree $m$. Following \Cref{thm:sample_complexity}, we now assume $\gamma > \frac{1}{2}$. 

We will now study the order in terms of $d$ of the expectation: 
\begin{equation}
    \EE \left [ g(h^{(2)}) I_{q}(A_j) I_{q}(A_k)\right ],
\end{equation}
where we recall that we denote $S = \sum_{i=1}^{d_1}\lambda_i (I_{q}(A_i)^{2}-1)$. Since $g$ is a polynomial, we can no longer apply the orthogonality of different chaos in the same way we did for the linear case in \Cref{eq:mid_step_M_linear}. To compute this, we will use \Cref{lemma:gaussian_ibp}. 

\subsubsection{First Step: Applying Gaussian Integration by parts} 
\label{section:int_by_parts}
This section follows the construction made in Chapter 8 in \cite{nourdin2012normal}, tailored to our setting. The objective is to derive \Cref{lemma:exact_expectation_polynomial}. The reader may skip this subsection on a first reading. 

Denote $F_1 = S$, $F_{2} = I_{q}(A_j)I_{q}(A_k)$, for $j \not =k$. Then, our expectation has the form: 
\begin{equation}
    \EE\left[ g(F_{1}) F_{2}\right ].
\end{equation}

Applying \Cref{lemma:gaussian_ibp} once, we get: 
\begin{align}
    \EE \left [ g(F_2) F_1\right ] = \underbrace{\EE \left [ g'(F_1) \langle D F_1, -DL^{-1}(F_2 - \EE[F_2] \rangle)\right ]}_{M:=} + \mathbb{E}[g(F_1)]\EE[F_2] . 
\end{align}
We focus on the first term. Denote 
\begin{equation}
    V_{1} = \langle DF_1, -DL^{-1}F_2 \rangle.
\end{equation}
Let $v_1 = \EE [V_{1}]$, and denote $\bar{V}_1 = V_1 - v_1$. Then: 
\begin{align}
    M = \EE \left [ g'(F_1) 
    \bar{V}_{1}\right ] + \EE[g'(F_1)]v_{1}. 
\end{align}
Then, applying \Cref{lemma:gaussian_ibp} again: 
\begin{align}
     M &=  \mathbb{E} \left [g^{(2)}(F_1) \underbrace{\langle DF_1, DL^{-1}\bar{V}_1\rangle}_{V_2} \right ] + \EE[g'(F_1)] v_1. 
\end{align}
We now denote $v_{2}  = \EE \left [ V_2\right ]$, and $\bar{V}_{2} = V_{2} - v_{2}$. Applying \Cref{lemma:gaussian_ibp} again: 
\begin{align}
    M & = \mathbb{E} \left [g^{(2)}(F_1) \bar{V}_{2} \right ] + \EE\left [ g^{(2)}(F_1)\right]v_{2} + \EE[g'(F_1)] v_1 \\ 
    & =  \mathbb{E} \left [g^{(3)}(F_1) \langle DF_1, DL^{-1}\bar{V}_{2} \rangle\right ] + \EE[g^{(2)}(F_1)]v_{2} + \sum_{i=1}^{d} \lambda_{i} v_1.
\end{align}
Iterating $\mathrm{deg}(g)-1$ times, we get: 
\begin{align}
    \EE \left [ g(h^{(2)}) I_{q}(A_j) I_{q}(A_k)\right ] &= \sum_{r = 0}^{\mathrm{deg}(g)-1} \EE[g^{(r)}(F_1)]v_{r}
    \label{eq:full_expression_expectation}
\end{align}
where we inductively defined 
\begin{equation}
    V_{r+1} = \EE \left [\langle DF_1, DL^{-1}(V_{r}-v_{r})) \rangle \right ], 
    \quad 
    v_{r} = \EE \left [ V_{r} \right],
\end{equation}
and $V_0 = I_{q}(A_j)I_{q}(A_j)$. Note that the objects $V_1, \dots V_{\mathrm{deg}(g)}$ are exactly the ones that appear in \Cref{def:gamma_l_1_l_N}. This can be made precise. We actually have:
\begin{equation}
    V_{r} = \Gamma_{F_2, \underbrace{F_1, \dots, F_1}_{r \text{ times}}}. 
    \label{eq:V_and_F}
\end{equation}
By \Cref{lemma:cumulants_and_Gamma}, one can relate the expectation of this variables to cumulants.  To be precise, we have that for $r \in \NN$: 
$$
\kappa_r(F_2, \underbrace{F_{1}, \dots, F_{1}}_{r \text{times}}) = \sum_{\sigma \in \mathfrak{S}_{\{2, \ldots, r\}}} \EE\left[\Gamma_{F_{2}, F_{1}, \ldots, F_{1}}(F)\right],
$$  
and since the expectation is repeated, \Cref{eq:V_and_F} allows us to conclude the relation: 
\begin{equation}
    \kappa_r(F_2, \underbrace{F_{1}, \dots, F_{1}}_{r \text{times}}) = r! \EE\left[\Gamma_{F_{2}, F_{1}, \ldots, F_{1}}(F)\right] = r! \EE\left[ V_{r}\right ] = r! v_{r}. 
\end{equation}
Then, we conclude: 
\begin{lemma} Let $S = \sum_{i=1}^{d_1} \lambda_{i} (I_{q}(h^{(2)})^{2} -1)$, and let $g: \RR \to \RR$ be a polynomial. Then, given $P, Q \in (\RR^{d})^{\odot q}$
\[
\EE \left[ g(h^{(2)}) I_{q}(P) I_{q}(Q)\right ] = \sum_{r=0}^{\mathrm{deg}(g)} \EE \left[ g^{(r)}(h^{(2)})\right ] \dfrac{1}{r!}\kappa_r(I_{q}(P)I_{q}(Q), \underbrace{S, \dots, S}_{r \text{times}}).
 \]
 \label{lemma:exact_expectation_polynomial}
\end{lemma}

\begin{remark}
    This type of result is a generalization of Eq. 8.5.1 in \cite{nourdin2012normal} to the multi-variate case. Note that it works for any symmetric tensor. 
\end{remark}

\subsubsection{Step 2: Computing the order of the Cumulants}
\label{section:cumulants}
By \Cref{lemma:exact_expectation_polynomial}, our problem is reduced to computing cumulants. In particular, we want to derive bounds for 
\begin{equation}
    v_{r} = \dfrac{1}{r!} \kappa \left (\underbrace{S, S, \dots, S}_{r}, (I_{q}(A_k)I_{q} (A_j) - q! \langle A_j, A_k \rangle ) \right ). 
\end{equation}
A key property of cumulants is that they are multi-linear. Since the variable $S= \sum_{i=1}^{d_1}\lambda_i (I_{q}(A_i)^{2} -1)$, this allows us to exchange this sums with the cumulants. Denote
\[
Y_{i_\ell} = (I_{q}(A_{i_\ell})^{2} -1), G_{i,j} = I_{q}(A_j) I_{q}(A_k) - \langle A_j, A_k\rangle.
\]
By the multi-linearity of the cumulants: 
\begin{equation}
    v_{r} = \dfrac{1}{r!} \sum_{i_{1}, \dots, i_{r}=1}^{d_1} \lambda_{i_1} \cdots \lambda_{i_r} \kappa \left (G_{j,k}, Y_{i_1}, Y_{i_2}, \dots, Y_{i_r}\right ). 
\end{equation}
By expanding each $Y_{\ell}$ into its chaos decomposition: 
\begin{equation}
    Y_{i_\ell} = I_{q}(A_{i_\ell})^{2} -1 = \sum_{L=0}^{q-1} c_{q,L} I_{2q-2L}(A_{i} \otimes_{L} A_i ) + (\| A_i\|^{2} - 1), 
    \label{eq:quadratic_chaos_expansion}
\end{equation}
and for $G_{j,k}$: 
\begin{equation}
    I_{q}(A_j)I_{q}(A_k) - \langle A_j, A_k\rangle= \sum_{L=0}^{q-1} c_{q,L} I_{2q-2L}(A_{j}\tilde{\otimes}_{L} A_k ). 
\end{equation}
With this, we can go further with the multi-linearity of the cumulants to obtain: 
\begin{equation}
     v_{r} = \dfrac{1}{r!} \sum_{i_{1}, \dots, i_{r}=1}^{d_1} \lambda_{i_1} \cdots \lambda_{i_r} \sum_{s_1, \dots, s_r = 0}^{q-1}\sum_{t = 0}^{q-1} c_{q,s,t }\kappa \left (I_{2q-2s_{1}}(A_{i_1} \tilde{\otimes}_{s_1}A_{i_1}) , \dots,I_{2q-2s_{r}}(A_{i_r} \tilde{\otimes}_{s_r}A_{i_r}), I_{2q-2t}(A_{j}\tilde{\otimes}_{t} A_k )\right ), 
\end{equation}
where we ignored the expectation terms in \Cref{eq:quadratic_chaos_expansion} since they become negligible.  Denote $f^{s}_i= A_{i} \tilde{\otimes}_{s}A_{i}, i \leq {r}$,  and $f^{s}_{r+1} = A_{j} \tilde{\otimes}_{s}A_{k}$ so that: 
\begin{equation}
     v_{r} = \dfrac{1}{r!} \sum_{i_{1}, \dots, i_{r}=1}^{d_1} \lambda_{i_1} \cdots \lambda_{i_r} \sum_{s_1, \dots, s_r = 0}^{q-1}\sum_{t = 0}^{q-1} c_{q,s,t }\kappa \left (I_{2q-2s_{1}}(f^{(s_1)}_{i_1}) , \dots,I_{2q-2s_{r}}(f^{(s_r)}_{i_r}), I_{2q-2t}(f^{t}_{r+1})\right ).
     \label{eq:v_r_in_terms_of_cumulants}
\end{equation}
Now, by \Cref{lemma:cummulants}, we can write: 
\begin{align}
    \kappa \left (I_{2q-2s_{1}}(f^{(s_1)}_{i_1}) , \dots,I_{2q-2s_{r}}(f^{(s_r)}_{i_r}), I_{2q-2t}(f^{t}_{r+1})\right ) & = \sum_{\sigma \in \mathfrak{S}_{\{2, \ldots, |m|\}}} (q_{\lambda_\sigma(|m|)})! \sum_* c_{q,l,\sigma}(a_2, \ldots, a_{|m|-1})  \\ 
    &  \langle (\ldots ((f_{i_{\lambda(1)}} \widetilde{\otimes}_{r_2} f_{i_{\lambda_\sigma(2)}}) \widetilde{\otimes}_{r_3} f_{\lambda_\sigma(3)}) \ldots) \widetilde{\otimes}_{r_{|m|-1}} f_{\lambda_\sigma(|m|-1)}; f_{\lambda_\sigma(|m|)} \rangle,
    \label{eq:cummulants_expression}
\end{align}
where the second sum runs over combinations of indices having technical conditions (all specified in \Cref{lemma:cummulants}. We ignored the upper-indices to avoid overloading the notation. The important message of \Cref{eq:cummulants_expression} is there is a finite set of possible contractions (in particular, a set of size $O_{d}(1)$), and we are summing over all of them and all possible permutations of indices.  Denote 
\begin{equation}
    T_{\lambda} = \langle (\ldots ((f_{i_{\lambda(1)}} \widetilde{\otimes}_{r_2} f_{i_{\lambda_\sigma(2)}}) \widetilde{\otimes}_{r_3} f_{\lambda_\sigma(3)}) \ldots) \widetilde{\otimes}_{r_{|m|-1}} f_{\lambda_\sigma(|m|-1)}; f_{\lambda_\sigma(|m|)} \rangle.
\end{equation}

Thus, we have concluded: 
\begin{lemma}
    Let $h^{(2)} = \sum_{i=1}^{d_1} \lambda_{i} (I_{q}(h^{(2)})^{2} -1)$, and let $g: \RR \to \RR$ be a polynomial. Then: 
\[
\EE \left[ g(h^{(2)}) I_{q}(A_{j}) I_{q}(A_{k})\right ] = \sum_{r=0}^{\mathrm{deg}(g)} \EE \left[ g^{(r)}(h^{(2)})\right ] \dfrac{1}{r!} \sum_{i_{1}, \dots, i_{r}=1}^{d_1} \lambda_{i_1} \cdots \lambda_{i_r} \sum_{s_1, \dots, s_r = 0}^{q-1}\sum_{t = 0}^{q-1} c_{q,s,t }\sum_{\lambda} T_{\lambda}.
 \]
 \label{lemma:expectation_and_cumulants}
\end{lemma}

We now make two observations: The first is that $v_{1}$ is exactly what we computed for \Cref{lemma:linear_case_cross_expectations}. The second one is that, so far, everything has been exact (i.e we have no error terms in our formula). Since we aim only to compute the order of the expectation, we will now proceed to bound each term, beginning with $T_{\lambda}$

\subsubsection{\texorpdfstring{Bounding $T_{\lambda}$}{Bounding T-lambda}}

Recall that we defined: 
\begin{equation}
    T_{\lambda} = \langle (\ldots ((f_{i_{\lambda(1)}} \widetilde{\otimes}_{r_2} f_{i_{\lambda_\sigma(2)}}) \widetilde{\otimes}_{r_3} f_{\lambda_\sigma(3)}) \ldots) \widetilde{\otimes}_{r_{|m|-1}} f_{\lambda_\sigma(|m|-1)}; f_{\lambda_\sigma(|m|)} \rangle.
    \label{def_T_lambda}
\end{equation}
In particular, the term $T_{\lambda}$ is computed from a particular set of appearances of $A_{1}, \dots A_{d_1}$. Let $A_{i_1},\dots, A_{i_r}$ denote the tensors involved in the computation of a particular $T_\lambda$. 
We claim the following: 
\begin{claim}
    If $\{ i_{1}, \dots i_{r}\}\cap \{ j,k\} = \emptyset$, then with high probability with respect to $A_{1}, \dots A_{d_{1}}$
    $$
    \left |T_{\lambda} \right | \lesssim \dfrac{1}{d^{q}}.
    $$
    \label{claim:disjoint_case}
\end{claim}

\begin{proof}
    Note that symmetrization only change the expectation up to constants, so we can compute the terms without symmetrization. If $\{ i_{1}, \dots i_{r}\}\cap \{ j,k\} = \emptyset$, then the contraction in \Cref{def_T_lambda} is linear in $A_j$ and $A_k$, and therefore we can write (without symmetrization): 
    \begin{equation}
        T_{\lambda} = \sum_{a,b \in [d]^{q}} w(a, b)A_{j,a} A_{k,b},
    \end{equation}
    where $w(a, b)$ sums over all the other contractions. 
    Then, fixing $A_{i_1}, \dots, A_{i_r}$, we note that the expectation of $T_{\lambda((i_1,s_1), \dots, (i_{r+1}, t))}$ w.r.t all $A$ is zero. We can also compute the conditional variance to get: 
    \begin{equation}
        \mathrm{Var}(T_{\lambda} |A_{i_1}, \dots, A_{i_r}) = \dfrac{1}{d^{2q}} \sum_{a,b \in [d]^{q}} w(a, b)^{2} = \dfrac{1}{d^{2q}} \| w(a, b)\|^{2}.
    \end{equation}
    Since $\EE_{A}\left [\| w(a, b)\|^{2} \right ]= O_{d}(1)$, we can proceed as with did in the linear case and conclude   the lemma by the law of total variance.  Having this, we get concentration by applying Chebyshev Inequality and Gaussian Hypercontractivity \cref{lemma:gaussian_hypercontractivity}). 
\end{proof}

Having dealt with the disjoint case in \Cref{claim:disjoint_case}, we now proceed with the harder case where indices $i_{1}, \dots i_{r}$ may have a non-empty intersection with $\{j,k\}$. Note that we don't want an sharp bound, but rather a bound that shows that this terms are negligible with respect to the linear part. 

Let $(i_{1},\dots, i_{r}$ be such that $\{ i_{1},\dots, i_{r}\} \cap \{j,k\} \not = \empty$. Assume that the tensor $A_{j}$ appears $m_{j}$ times in the sequence $f_{i_1}, \dots, f_{i_{r+1}}$. Then, necessarily, $m_{j}$ has to be odd: It appears once in $f_{i_{r+1}}$, and all other appearances will be tensor product of $A_j$ with itself. By the same argument, $m_{k}$ is also odd. Let $\mathcal{G}_{a_1, \dots, a_\ell}$  denote the $\sigma$-algebra generated by  $A_{a_1}, \dots, A_{a_l\ell}$. Then from this observation we conclude: 
\begin{equation}
    \EE_{A}[T_{\lambda }] = \EE [ \EE [T_\lambda | \mathcal{G}_{[r]\setminus \{j,k\}}]]= 0. 
\end{equation}
By computing the conditional variance of $T_\lambda$, we will have: 
\begin{equation}
    \mathrm{Var}(T_\lambda | \mathcal{G}_{[r]\setminus \{j,k\}}) = \EE \left [ T_\lambda^{2} | \mathcal{G}_{[r]\setminus \{j,k\}} \right ]. 
\end{equation}
Let $F(A_{j}) = T_\lambda$. Then, by \Cref{lemma:gaussian_poincare}:
\begin{equation}
    \mathrm{Var}(T_\lambda^{2} |\mathcal{G}_{[r]\setminus\{j\}}) \leq \dfrac{C}{d^{q}} \EE_{A_j} \left[ \| \nabla_{A_{j}} T_{\lambda}\|^{2}|\mathcal{G}_{[r]\setminus\{j\}}\right ].
\end{equation}
Now, note that $T_{\lambda}$ has the following form:
\begin{equation}
    T_{\lambda} = \mathrm{Contraction}(A_{i_1}, \dots, A_{i_r}, A_j, A_k).
\end{equation}
In particular, it is multi-linear in all of this arguments. Let $D_{H}$ be a directional derivative in direction $H$. If we denote by $\mathrm{Pos}(j)$  the set of positions such that $A_{i_1} = A_{j}$, then: 
\begin{equation}
    D_{H} T_{\lambda} = \sum_{p \in \mathrm{Pos}(j)} \mathrm{Contraction}(A_{i_1}, \dots, A_{i_{p-1}}, H, A_{i_{p+1}}, \dots).
\end{equation}
Then, by applying Cauchy-Schwarz: 
\begin{equation}
    \left | D_{H} T_{\lambda } \right | \leq C \sum_{p \in \mathrm{Pos}(j)} \|H\|_2 \prod_{\ell \not = p} \| K_{\ell}\|_2,
\end{equation}
where $K_{\ell}$ are all other tensors. Since all these tensors are already contractions, we can apply Cauchy-Schwarz again to get: 
\begin{equation}
    \left | D_{H} T_{\lambda } \right | \leq C \sum_{p \in \mathrm{Pos}(j)} \|H\|_2 \| A_{j}\|^{m_{j}-1}\prod_{\ell \not = j} \| A_{\ell}\|_2 \leq Cm_{j} \|H\|_2 \| A_{j}\|^{m_{j}-1}\prod_{\ell \not = j} \| A_{\ell}\|_2 .
\end{equation}
Taking supremum over the sphere, we conclude: 
\begin{equation}
    \| \nabla_{A_{j}} T_{\lambda}\| \leq C \| A_{j}\|^{m_{j}-1}\prod_{\ell \not = j} \| A_{\ell}\|_2^{m_\ell},
\end{equation}
and taking the square: 
\begin{equation}
     \| \nabla_{A_{j}} T_{\lambda}\| \leq C \| A_{j}\|^{2(m_{j}-1)}\prod_{\ell \not = j} \| A_{\ell}\|_2^{2m_{\ell}}.
\end{equation}
Finally, by Hanson-Wright we know that $\| A_{j}\|_2$ is $\Theta_d(1)$ with high probability. Therefore:
\begin{equation}
    \| \nabla_{A_{j}} T_{\lambda}\| \leq C \prod_{\ell \not = j} \| A_{\ell}\|_2^{2m_{\ell}},
\end{equation}
and in particular: 
\begin{equation}
    \mathrm{Var}(T_\lambda^{2} |\mathcal{G}_{[r]\setminus\{j\}}) \lesssim \dfrac{1}{d^{q}}\prod_{\ell \not = j} \| A_{\ell}\|_2^{2m_{\ell}}.
\end{equation}
By the Law of total variance, we can take expectation again to conclude: 
\begin{equation}
    \mathrm{Var}(T_{\lambda)}) \leq \dfrac{1}{d^{q}} . 
\end{equation}
Then, by applying Chebyshev Inequality and hypercontractivity, we conclude that with high probability over $A_{1}, \dots, A_{d_1}$
\begin{equation}
    |T_{\lambda}| \lesssim \dfrac{1}{d^{\frac{q}{2}}}. 
\end{equation}
Let's write all of this in a Lemma. 

\begin{lemma}
    Let $T_{\lambda}$ be defined as in \Cref{def_T_lambda}, for a
    set of tensors $A_{i_{1}}, \dots A_{i_{r}}$ . Then: 
    \begin{enumerate}
        \item If $\{ i_{1}, \dots, i_{r}\} \cap \{j,k\} = \emptyset$, then with high probability $|T_{\lambda}| \lesssim \dfrac{1}{d^{q}}$. 
        \item If $\{ i_{1}, \dots, i_{r}\} \cap \{j,k\} \not = \emptyset$, then with high probability $|T_{\lambda}| \lesssim \dfrac{1}{d^{\frac{q}{2}}}$.
    \end{enumerate}
    \label{lemma:control_T_lambda}
\end{lemma}

\subsubsection{Conclusion}

From \Cref{lemma:expectation_and_cumulants}, we had: 
\begin{equation}
    \EE \left[ g(h^{(2)}) I_{q}(A_{j}) I_{q}(A_{k})\right ] = \sum_{r=0}^{\mathrm{deg}(g)} \EE \left[ g^{(r)}(h^{(2)})\right ] \dfrac{1}{r!} \sum_{i_{1}, \dots, i_{r}=1}^{d_1} \lambda_{i_1} \cdots \lambda_{i_r} \sum_{s_1, \dots, s_r = 0}^{q-1}\sum_{t = 0}^{q-1} c_{q,s,t }\sum_{\lambda}T_{\lambda}.
\end{equation}
We will separate the linear part from the rest. We write: 
\begin{equation}
    \EE \left[ g(h^{(2)}) I_{q}(A_{j}) I_{q}(A_{k})\right ] = E_\mathrm{linear} + \underbrace{\sum_{r\geq 2}^{\mathrm{deg}(g)} \EE \left[ g^{(r)}(h^{(2)})\right ] \dfrac{1}{r!} \sum_{i_{1}, \dots, i_{r}=1}^{d_1} \lambda_{i_1} \cdots \lambda_{i_r} \sum_{s_1, \dots, s_r = 0}^{q-1}\sum_{t = 0}^{q-1} c_{q,s,t }\sum_{\lambda}T_{\lambda}}_{E_\mathrm{NL}}. 
\end{equation}
Let us focus on $E_\mathrm{NL}$. Replacing \Cref{lemma:control_T_lambda}, since the sum on the RHS concerns $O_{d}(1)$ terms: 
\begin{equation}
     E_\mathrm{NL}= O \left ( \sum_{r=2}^{\mathrm{deg}(g)} \EE \left[ g^{(r)}(h^{(2)})\right ] \dfrac{1}{d^{\frac{q}{2}}} \sum_{\substack{i_{1}, \dots, i_{r}=1\\ \{i_{1}, \dots, i_{r}\}\cap\{j,k\}\not = \emptyset  }}^{d_1} \lambda_{i_1} \cdots \lambda_{i_r}  + \sum_{r=2}^{\mathrm{deg}(g)} \EE \left[ g^{(r)}(h^{(2)})\right ] \dfrac{1}{d^{q}} \sum_{\substack{i_{1}, \dots, i_{r}=1\\ \{i_{1}, \dots, i_{r}\}\cap\{j,k\} = \emptyset  }}^{d_1} \lambda_{i_1} \cdots \lambda_{i_r}.
     \right )
\end{equation}
Since the sum of $\lambda_i$ is bounded  (by the same argument as the linear case), we get: 
\begin{equation}
    E_\mathrm{NL} = O \left ( \sum_{r=2}^{\mathrm{deg}(g)} \EE \left[ g^{(r)}(h^{(2)})\right ] \dfrac{1}{d^{\frac{q}{2}}} \sum_{\substack{i_{1}, \dots, i_{r}=1\\ \{i_{1}, \dots, i_{r}\}\cap\{j,k\}\not = \emptyset  }}^{d_1} \lambda_{i_1} \cdots \lambda_{i_r}\right ). 
    \label{eq:non_linear_almost}
\end{equation}
Now, from \Cref{lemma:linear_case_cross_expectations}, we already computed $v_{1}$, so: 
\begin{equation}
    E_\mathrm{linear} = \EE\left[ g(h^{(2)}) \right ] \langle A_j, A_k \rangle + \EE\left[ g'(h^{(2)}) \right ]  O\left ( \dfrac{\max(k,j)^{-\gamma}}{d^{\frac{q}{2}}}\right ).
\end{equation}
Analogously to the linear case, we can conclude: 
\begin{equation}
    E_\mathrm{linear} = \EE[g(h^{(2)})]\langle A_j, A_k \rangle + \EE\left[ g'(h^{(2)}) \right ]  O\left ( \dfrac{\max(k,j)^{-\gamma}}{d^{\frac{q}{2}}}\right ).
\end{equation}
From \Cref{eq:non_linear_almost} and he fact that all eigenvalues are in $[0,1]$, and $ r \geq 2$, we get that $E_\mathrm{NL}$ is sub-leading with respect to the linear term. To finish, we need the following Lemma, whose proof we postpone to \Cref{seciton:deferred_proofs}. 
\begin{lemma}
    Let $g:\RR \to \RR$ be a polynomial with information exponent $1$. Then:
    \begin{equation}
        \EE[g(h^{(2)})]= \dfrac{1}{\sqrt{d}} \quad \text{ and } \EE \left[ g'(h^{(2)}) \right ] = \nu_{1} + \dfrac{C}{\sqrt{d}},
    \end{equation}
    where $\nu_1$ is the first Hermite coefficient of $g$. 
    \label{lemma:information_exponent_g}
\end{lemma}
Combining this with \Cref{lemma:linear_case_cross_expectations}, we get that the first constant term is sub-leading and we conclude:  
\begin{lemma}
    Let $h^{(2)} = \sum_{i=1}^{d_1} \lambda_i (I_{q}(A_i)^{2} -1)$, and let $g$ be a polynomial with information exponent $1$. Then, given $i, j \in [d_1]$ with $i \not = j$: 
    \[
     \left | \EE \left [ g(h^{(2)}) I_{q}(A_j) I_{q}(A_k)\right ] \right | = O \left ( \dfrac{Z_\gamma}{d^{q}}(j^{-\gamma} + k^{-\gamma}))\right ). 
    \]\label{lemma:non_linear_case_cross_expectations}
\end{lemma}

\subsection{\texorpdfstring{Studying $\EE[\hat{C}]$}{Studying E[C]}}
\label{section:E[C]_proof}
The objective of this section is to prove the following Lemma:  
\begin{lemma}
Under the assumptions of \Cref{thm:sample_complexity}: 
    \[
    \EE \left [C^{(1)} \right ] = \dfrac{\EE\left[ g'(h^{(2)})\right ]}{\sqrt{2}}A^{(1)} D_{\gamma} (A^{(1)})^T + \Delta, 
    \]
    where $\| \Delta \|_\op = o_{d}(1)$,  $A^{(1)} = [u_1, \dots u_{d_1}] \in \mathbb{R}^{D \times d_1}$, and $D_\gamma = \mathrm{diag} ( \lambda_1, \dots, \lambda_{d_1}) \in \RR^{d_1 \times d_1}$. 
    \label{lemma:expectation_C_1_2}
\end{lemma}
The proof is similar to the one in \Cref{lemma:non_linear_part_expectations}. If $g$ is linear, the result is trivial, so we will focus on the non-linear setting, where we assume $0 \leq \gamma < \frac{1}{2}$. The linear setting will follow as a corollary. 

Recall that
\begin{equation}
    \hat{C} = \dfrac{1}{n} \sum_{\mu = 1}^{n} y_\mu He_2(F_\mu), 
\end{equation}
with $F_\mu = \mathcal{F}(\He_{q}(x_\mu))$, and $\He_{q}(x_\mu)$ is the degree $q$ Hermite tensor. Then the expectation is given by: 
\begin{equation}
    \EE \left [ \hat{C}\right ] = \EE \left [ y_\mu He_2(F_\mu)\right].
\end{equation} 
We want to prove that this expectation is close in operator norm to: 
\begin{equation}
    \EE[g'(h^2)] \sum_{i=1}^{d_1} \lambda_i A^{(1)}_i (A_i^{(1)})^T. 
\end{equation}
Let $B(d,q$ be the dimension of the vectors $A^{(1)}_i$. Then: 
\begin{align}
    \| \EE\left [\hat{C} \right ] -  \dfrac{\EE[g'(h^2)]}{\sqrt{2}} \sum_{i=1}^{d_1} \lambda_i A^{(1)}_i (A_i^{(1)})^T\|_\op  = \max_{B \in \RR^{B(d,q)}, \| B\|_2 = 1} \left | B^T \EE\left[ \hat{C} \right ] B - \dfrac{\EE[g'(h^2)]}{\sqrt{2}} \sum_{i=1}^{d} \lambda_i \langle A_i, B \rangle^{2}\right |. 
    \label{eq:operator_norm_subtraction_1}
\end{align}
Denote $Y_{B} : = I_{q}(B)^{2}-1$. Then: 
\begin{equation}
    B^T \He_{2}(F_{\mu}) B = \dfrac{1}{\sqrt{2}} (\langle B, F_\mu\rangle^{2}-\underbrace{1}_{=\| B\|^2}) = \dfrac{1}{\sqrt{2}}(I_{q}(B)^{2}-1) = \dfrac{1}{\sqrt{2}} F_{B}.
\end{equation}
Therefore:
\begin{equation}
    B^T \EE\left[ \hat{C} \right ] B = \dfrac{1}{\sqrt{2}}\EE \left [ g(h^{(2)})Y_{B}\right ].
\end{equation}
Then, going back to \Cref{eq:operator_norm_subtraction_1}:
\begin{equation}
    \| \EE\left [\hat{C} \right ] -  \EE[g'(h^2)] \sum_{i=1}^{d_1} \lambda_i A^{(1)}_i (A_i^{(1)})^T\|_\op  = \dfrac{1}{\sqrt{2}}\max_{B \in \RR^{B(d,q)}, \| B\|_2 = 1} \left | \EE \left [ g(h^{(2)})Y_{B}\right ]- \sum_{i=1}^{d} \lambda_i \langle A_i, B \rangle^{2}\right |.
\end{equation}
Let
\begin{equation}
    \Delta:=\max_{B \in \RR^{B(d,q)}, \| B\|_2 = 1} \left | \EE \left [ g(h^{(2)})Y_{B}\right ]- \sum_{i=1}^{d} \lambda_i \langle A_i, B \rangle^{2}\right |.
    \label{eq:operator_norm_subtraction}
\end{equation}
If we conclude that $\Delta$ is $o_{d}(1)$, we complete the proof. Our objective is hence to prove this. We will do this in a three steps.

In the following, we will extensively use the Wiener decomposition of $Y_{B}$ and similar random variables, which is computed by \Cref{lemma:product_formula}. For any symmetric tensor $C \in (\RR^{d})^{\odot q}$: 
\begin{equation}
    (I_{q}(C)^{2} -1)  = \sum_{r=0}^{q-1}c_{q,r} I_{2q-2r}(C\tilde{\otimes}_{r} C). 
    \label{eq:wiener_decomp_Y_B}
\end{equation}

\subsubsection{Step 1: Integration by Parts}

We will first try to write the term $\EE\left[ g(h^{(2)}) Y_{B}\right ]$ in \Cref{eq:operator_norm_subtraction} in the form:
\begin{equation}
    \EE\left[ g(h^{(2)}) Y_{B}\right ] = \sum_{i=1}^{d} \lambda_i \langle A_i, B \rangle^{2} + \text{other term}. 
\end{equation} 
The tool for this is \Cref{lemma:gaussian_ibp}, Integration by Parts. We will apply it twice. On a first iteration, since $Y_{B}$ is centered, we have: 
\begin{align}
    \EE\left[ g(h^{(2)}) Y_{B}\right ] =  \EE \left [ g'(h^{(2)}) \langle Dh^{(2)}, DL^{-1}Y_{B \rangle}\right ].
\end{align}
Now we apply integration by parts again and obtain:
\begin{align}
   \EE\left[ g(h^{(2)}) Y_{B}\right ]  &= \EE \left [ g'(h^{(2)}) \right ]\EE \left [ \langle Dh^{(2)}, DL^{-1}Y_{B \rangle}\right ] \\
    & +  \EE \left [ g^{(2)}(h^{(2)}) \langle Dh^{(2)}, D \left ( \langle Dh^{(2)}, DL^{-1}Y_{B}\rangle \right ) \rangle\right ].
    \label{eq:int_by_parts_twice}
\end{align}

Now, by the definition of $h^{(2)}$, and the linearity of the derivative: 
\begin{align}
    \EE \left [ \langle Dh^{(2)}, DL^{-1}Y_{B } \rangle \right ]  & = \sum_{i=1}^{d} \lambda_{i} \EE \left [  \langle D(I_{q}(A_{i})^{2}-1), DL^{-1}Y_{B } \rangle  \right ].
\end{align}
Define $Y_{i} = (I_{q}(A_{i})^{2}-1)$. Then, by doing inverse Gaussian integration by parts: 
\begin{equation}
    \EE \left [ \langle Dh^{(2)}, DL^{-1}Y_{B } \rangle \right ]   = \sum_{i=1}^{d} \lambda_{i} \EE \left [ Y_{i}Y_{B}\right ],
\end{equation}
and by applying \Cref{eq:wiener_decomp_Y_B}, and the orthogonality of Wiener Chaos \Cref{lemma:orthogonality_wiener_chaos}: 
\begin{align}
    \EE \left [ \langle Dh^{(2)}, DL^{-1}Y_{B } \rangle \right ] &= \sum_{i=1}^{d} \lambda_{i} \sum_{r=0}^{q-1} c_{q,r} \langle A_{i} \tilde{\otimes}_{r} A_{i}, B\tilde{\otimes}_{r} B \rangle \\
    & = \sum_{i=1}^{d} \lambda_{i} c_{q,1}\langle A_{i} \tilde{\otimes} A_{i}, B\tilde{\otimes} B \rangle+ \sum_{i=1}^{d} \lambda_{i} \sum_{r=1}^{q-1} c_{q,r} \langle A_{i} \tilde{\otimes}_{r} A_{i}, B\tilde{\otimes}_{r} B \rangle.
    \label{eq:intermediate_step_2}
\end{align}
By \Cref{lemma:simple_self_and_cross_contractions}, we have that for all $r \in [q-1]$:
\begin{equation}
    \EE_{A_i} \left [\|A_i^{(1)} \otimes_{r}A_i^{(1)}\|^{2}_{F} \right ] = O\left(\dfrac{1}{d^{-r}}\right ).
    \label{eq:intermediate_step_1}
\end{equation}
Then with high probability: 
\begin{equation}
    \|A_i^{(1)} \otimes_{r}A_i^{(1)}\|^{2}_{F} \lesssim  \dfrac{1}{d^{-r}}
    \label{eq:high_prob_bound_a_i_a_i}
\end{equation}
On the other hand, using the fact that $\lambda_{i}=z_{i}Z_{\gamma} i^{-\gamma}$, with $z_{i}\sim \mathrm{Rad}(\frac{1}{2})$, for fixed $A's$, we can apply Bernstein's Inequality \cite{vershynin2018high} to get: 
\begin{align}
    \left |\sum_{i=1}^{d} \lambda_{i} \sum_{r=1}^{q-1} c_{q,r} \langle A_{i} \tilde{\otimes}_{r} A_{i}, B\tilde{\otimes}_{r} B \rangle \right | & \lesssim \sqrt{\sum_{i=1}^{d} Z_{\gamma}^{2}i^{-2\gamma} \left ( \sum_{r=1}^{q-1} c_{q,r} \langle A_{i} \tilde{\otimes}_{r} A_{i}, B\tilde{\otimes}_{r} B \rangle  \right )^{2}} \\
    & \leq \sqrt{\sum_{i=1}^{d} Z_{\gamma}^{2}i^{-2\gamma}\sum_{r=1}^{q-1} c_{q,r} \langle A_{i} \tilde{\otimes}_{r} A_{i}, B\tilde{\otimes}_{r} B \rangle^{2}  } \\
    & \leq \sqrt{\sum_{i=1}^{d} Z_{\gamma}^{2}i^{-2\gamma}\sum_{r=1}^{q-1} c_{q,r} \| A_{i} \tilde{\otimes}_{r} A_{i}\|^{2}_{F}\| B\tilde{\otimes}_{r} B \|_{F}^{2}  },
\end{align}
where in the last line we applied Cauchy Schwartz. Since $\| B\|^{2}_2$ by definition, we have $\| B\tilde{\otimes}_{r} B \|_{F}^{2} \lesssim 1$. Then: 
\begin{equation}
     \left |\sum_{i=1}^{d} \lambda_{i} \sum_{r=1}^{q-1} c_{q,r} \langle A_{i} \tilde{\otimes}_{r} A_{i}, B\tilde{\otimes}_{r} B \rangle \right |  \lesssim \sqrt{\sum_{i=1}^{d} Z_{\gamma}^{2}i^{-2\gamma}\sum_{r=1}^{q-1} c_{q,r} \| A_{i} \tilde{\otimes}_{r} A_{i}\|^{2}_{F}},
\end{equation}
and replacing \Cref{eq:high_prob_bound_a_i_a_i}, we conclude: 

\begin{equation}
    \left |\sum_{i=1}^{d} \lambda_{i} \sum_{r=1}^{q-1} c_{q,r} \langle A_{i} \tilde{\otimes}_{r} A_{i}, B\tilde{\otimes}_{r} B \rangle \right | \lesssim \dfrac{1}{\sqrt{d}}. 
\end{equation}
Replacing in \Cref{eq:intermediate_step_2}, we conclude: 
\begin{equation}
    \EE \left [ \langle Dh^{(2)}, DL^{-1}Y_{B } \rangle \right ]   =\sum_{i=1}^{d} \lambda_{i} c_{q,1}\langle A_{i} \tilde{\otimes} A_{i}, B\tilde{\otimes} B \rangle+ \Delta_{1},
\end{equation}
with $| \Delta_{1}|\lesssim \dfrac{1}{d}$. Replacing in \Cref{eq:int_by_parts_twice}:
\begin{align}
     \EE\left[ g(h^{(2)}) Y_{B}\right ]   & =\sum_{i=1}^{d} \lambda_{i} c_{q,1}\langle A_{i} \tilde{\otimes} A_{i}, B\tilde{\otimes} B \rangle+ \Delta_{1} \\
    & +  \EE \left [ g^{(2)}(h^{(2)}) \langle Dh^{(2)}, D \left ( \langle Dh^{(2)}, DL^{-1}Y_{B}\rangle \right ) \rangle\right ] + \Delta_{1},
\end{align}
with $| \Delta_{1}|\lesssim \dfrac{1}{d}$. Then, by \Cref{eq:operator_norm_subtraction}, we conclude:
\begin{equation}
     \Delta:=\max_{B \in \RR^{B(d,q)}, \| B\|_2 = 1} \left | \EE \left [ g^{(2)}(h^{(2)}) \langle Dh^{(2)}, D \left ( \langle Dh^{(2)}, DL^{-1}Y_{B}\rangle \right ) \rangle\right ]\right | +o_{d}(1).
     \label{eq:operator_norm_subtraction_2}
\end{equation}
We can now proceed to step 2.

\subsubsection{Step 2: Computation of the Kernels}
So far, we have reduced our problem to bounding 
\begin{equation}
    \left | \EE \left [ g^{(2)}(h^{(2)}) \langle Dh^{(2)}, D \left ( \langle Dh^{(2)}, DL^{-1}Y_{B}\rangle \right ) \rangle\right ]\right | 
\end{equation}
uniformly for $\|B||=1$. Since $g$ is a polynomial, $g^{(2)}$ is also a polynomial. At the same time, $h^{(2)}$ has finite variance. Then we have: 
\begin{equation}
    \EE \left [ g^{(2)}(h^{(2)})^{2}\right ]^{\frac{1}{2}}\lesssim 1. 
\end{equation}
Then, by Cauchy Schwartz: 
\begin{align}
    \Delta &\lesssim \max_{B \in \RR^{B(d,q)}, \| B\|_2 = 1} \EE \left [ \left \langle Dh^{(2)}, D \left ( \langle Dh^{(2)}, DL^{-1}Y_{B} \rangle \right ) \right \rangle^{2}\right ]^{\frac{1}{2}}+o_{d}(1) \\
    & \lesssim \max_{B \in \RR^{B(d,q)}, \| B\|_2 = 1} \EE \left [ \left (\sum_{i,j=1}^{d_1} \lambda_i \lambda_j \left \langle DY_{i}, D \left ( \langle DY_{j}, DL^{-1}Y_{B} \rangle \right ) \right \rangle \right )^{2}\right ]^{\frac{1}{2}}+o_{d}(1)
    \label{eq:operator_norm_subtraction_3}
\end{align} 
Let 
\begin{equation}
    V^{i_1,i_2} =\langle DY_{i_2}, D \left ( \langle DY_{i_1}, DL^{-1}Y_{B}\rangle \right \rangle
    \label{def:V_nested}
\end{equation}
We will now compute the Chaos expansion of $V$. We begin with the nested derivative. Define $T^{r}_{i} = A_{i} \tilde{\otimes}_{r}A_{i}$, and $T^{r}_{B} = B \tilde{\otimes}_{r}B$. From \Cref{eq:quadratic_chaos_expansion} and the derivative computation rule \Cref{lemma:computation_derivatives}: 
\begin{align}
    \langle DY_{i_1}, DL^{-1}Y_{B}\rangle & = \sum_{r_1,r_2 = 0}^{q-1} c_{q,r_{1},r_{2}} \langle DI_{2q-2r_{1}}(T^{r_2}_{i_1}), DI_{2q-2r_{2}}(T^{r_2}_{B})\rangle \\
    & = \sum_{r_1,r_2 = 0}^{q-1} c_{q,r_{1},r_{2}} \sum_{r_3=1}^{2q-2\max(r_1, r_2)}I_{4q-2(r_1 + r_2 + r_3)}(T^{r_2}_{i_1} \tilde{\otimes}_{r_3}T^{r_1}_{B}).  
\end{align}
Replacing in \Cref{def:V_nested}:
\begin{align}
    V^{i_1,i_2} & =\langle DY_{i_2}, DL^{-1} \left ( \langle DY_{i_1}, DL^{-1}Y_{B}\rangle \right \rangle \\
    & = \sum_{r_{4} = 0}^{q-1} c_{q,r_{4}}\langle D I_{2q-2r_{4}}(T_{i_2}^{r_4}), D \left ( \sum_{r_1,r_2 = 0}^{q-1} c_{q,r_{1},r_{2}} \sum_{r_3=1}^{2q-2\max(r_1, r_2)}I_{4q-2(r_1 + r_2 + r_3)}(T^{r_2}_{i_1} \tilde{\otimes}_{r_3}T^{r_1}_{B})\right ) \\
    & = \sum_{r_1, r_2, r_4 = 0}^{q-1} \sum_{r_3=1}^{2q-2\max(r_1, r_2)}c_{q,r} \langle D I_{2q-2r_{4}}(T_{i_2}^{^{r_4}}), D I_{4q-2(r_1 + r_2 + r_3)}(T^{r_2}_{i_1} \tilde{\otimes}_{r_3}T^{r_1}_{B})\rangle \\
    & =  \sum_{r_1, r_2, r_4 = 0}^{q-1} \sum_{r_3=1}^{2q-2\max(r_1, r_2)}\sum_{r_5=1}^{6q-2\max(r_4,(r_1 +r_2 + r_4)}c_{q,r}I_{6q-2(\sum_{\ell =1}^{5} r_i)} \left ( T_{i_2}^{r_4} \tilde{\otimes}_{r_5} \left ( T^{r_2}_{i_1} \tilde{\otimes}_{r_3}T^{r_1}_{B}\right )\right ).
\end{align}
We can now go to Step 3.
\subsubsection{Step 3: Contraction Bounds}
Ignoring the sum for the moment. The tensors involved in this computation are of the form:
\begin{equation}
    T_{i_2}^{r_4} \tilde{\otimes}_{r_5} \left ( T^{r_2}_{i_1} \tilde{\otimes}_{r_3}T^{r_1}_{B} \right )
\end{equation}
Now, there are three possible cases:
\begin{enumerate}
    \item \textbf{Case 1:} The contraction $r_5$ contains a contraction of size greater than $1$ between $T_{i_2}^{r_4}$ and $T^{r_1}_{i_1}$. 
    \item \textbf{Case 2}: The contraction $r_5$ only contracts $T_{i_2}^{r_4}$ and $T^{r_1}_{B}$, but $\max(r_4, r_1) \geq 1$. 
    \item \textbf{Case 2}: The contraction $r_5$ only contracts $T_{i_2}^{r_4}$ and $T^{r_1}_{B}$, but $r_4= r_1 = 0$.
\end{enumerate}
We write: 
\begin{equation}
    V^{i_1,i_1} = V_1^{i_1,i_1}+ V_2^{i_1,i_1} + V_3^{i_1,i_1},
\end{equation}
 where $V_{1}^{i_1,i_1}$ counts only the indices in \textbf{Case 1}, $V_2^{i_1,i_1}$ counts only the indices in \textbf{Case 2}, and $V_{3}^{i_1,i_1}$ counts only the indices in \textbf{Case 3}. We now study each term separately. \\

\textbf{Case 1: }Assume there is a contraction between $T_{i_2}^{r_4}$ and $T^{r_1}_{i_1}$ of $t_{i_{1},i_{2}}$ indices. Then, by Cauchy-Schwarz: 
\begin{align}
    \| T_{i_2}^{r_4} \tilde{\otimes}_{r_5} \left ( T^{r_2}_{i_1} \tilde{\otimes}_{r_3}T^{r_1}_{B} \right )\|_{F} \leq \| T_{i_2}^{r_4} \tilde{\otimes}_{t_{i_{1},i_{2}}} T^{r_2}_{i_1} \| \|T^{r_1}_{B}\|_{F}.
    \label{eq:CS_contractions}
\end{align}
Since $\| B\|^2 = 1$, we can bound $\|T^{r_1}_{B}\|_{F}$ by a constant and obtain: 
\begin{equation}
    \| T_{i_2}^{r_4} \tilde{\otimes}_{r_5} \left ( T^{r_2}_{i_1} \tilde{\otimes}_{r_3}T^{r_1}_{B} \right )\|_{F}\leq C\| T_{i_2}^{r_4} \tilde{\otimes}_{t_{i_{1},i_{2}}} T^{r_2}_{i_1} \|_F.
\end{equation}
By \Cref{lemma:double_contraction_expectation_cross} and \Cref{lemma:double_contraction_expectation}, we have: 
\begin{equation}
    \EE\left [ \| T_{i_2}^{r_4} \tilde{\otimes}_{t_{i_{1},i_{2}}} T^{r_2}_{i_1} \|^{2} \right ] = O(\dfrac{1}{d}).
\end{equation} 
Then, since $\| T_{i_2}^{r_4} \tilde{\otimes}_{t_{i_{1},i_{2}}} T^{r_2}_{i_1} \|^{2} $ is a polynomial, we can apply Chebyshev Inequality to obtain that with high probability: 
\begin{equation}
    \| T_{i_2}^{r_4} \tilde{\otimes}_{t_{i_{1},i_{2}}} T^{r_2}_{i_1} \| \lesssim \dfrac{1}{\sqrt{d}}.
\end{equation}
Then: 
\begin{equation}
    \EE[(V^{i_1,i_1}_1)^2]^{\frac{1}{2}} \lesssim \dfrac{1}{\sqrt{d}}. 
    \label{eq:bound_V_1_squared}
\end{equation}

\textbf{Case 2: }If the contraction $r_5$ only contracts $T_{i_2}^{r_4}$ and $T^{r_1}_{B}$, but $\max(r_4, r_1) \geq 1$. Then,  applying \Cref{eq:CS_contractions} we get:
\begin{equation}
    \| T_{i_2}^{r_4} \tilde{\otimes}_{r_5} \left ( T^{r_2}_{i_1} \tilde{\otimes}_{r_3}T^{r_1}_{B} \right )\|_{F} \leq \| T_{i_2}^{r_4}\|_F \| T^{r_2}_{i_1} \|_F \|T^{r_1}_{B}\|_{F},
\end{equation}
By Lemma A.3 from \cite{tabanelli2026deep}, we have:
\begin{equation}
    \EE \left [\| T_i^{r}\|^{2} \right ] = O\left( \dfrac{1}{d^{-r}}\right ),
\end{equation}
for $r \in [q-1]$. Then, by applying hypercontractivity and Chebyshev inequality, with high probability:  
\begin{equation}
    \| T_{i}^{r}\|_F \lesssim d^{-\frac{r}{2}},
\end{equation}
with high probability, so from the fact that $\max(r_2, r_4) \geq 1$, we conclude that with high probability: 
\begin{equation}
    \| T_{i_2}^{r_4} \tilde{\otimes}_{r_5} \left ( T^{r_2}_{i_1} \tilde{\otimes}_{r_3}T^{r_1}_{B} \right )\|_{F} \lesssim \dfrac{1}{\sqrt{d}}. 
\end{equation}
Then: 
\begin{equation}
    \EE[(V^{i_1,i_1}_2)^2]^{\frac{1}{2}} \lesssim \dfrac{1}{\sqrt{d}}. 
    \label{eq:bound_V_2_squared}
\end{equation}
Note that we can write \Cref{eq:operator_norm_subtraction_3} as:
\begin{align}
    \Delta &  \lesssim \max_{B \in \RR^{B(d,q)}, \| B\|_2 = 1} \EE \left [ \left (\sum_{i,j=1}^{d_1} \lambda_i \lambda_j V^{i_1,i_2} \right )^{2}\right ]^{\frac{1}{2}}+o_{d}(1) \\
    &\lesssim \max_{B \in \RR^{B(d,q)}, \| B\|_2 = 1} \EE \left [ \left (\sum_{i,j=1}^{d_1} \lambda_i \lambda_j (V^{i,j}_{1} + V^{i,j}_2 + V^{i,j}_{3} )\right )^{2}\right ]^{\frac{1}{2}}+o_{d}(1).
    \label{eq:operator_norm_subtraction_4}
\end{align} 
Replacing \Cref{eq:bound_V_2_squared} and \Cref{eq:bound_V_1_squared} in \Cref{eq:operator_norm_subtraction_4}, and using the fact that for $\gamma <\frac{1}{2}$,
\begin{equation}
    \left (\sum_{i,j=1}^{d_1} \lambda_{i}^2\lambda_{j}^2\right ) = \Theta(1),
\end{equation}
we get: 
\begin{align}
    \Delta &\lesssim \max_{B \in \RR^{B(d,q)}, \| B\|_2 = 1} \EE \left [ \left (\sum_{i,j=1}^{d_1} \lambda_i \lambda_j V^{i,j}_{3} )\right )^{2}\right ]^{\frac{1}{2}}+o_{d}(1).
    \label{eq:operator_norm_subtraction_5}
\end{align} 

We can now proceed with the last step.
\subsubsection{Step 4: Studying the last term}

In \textbf{Case 3},  the contraction $r_5$ only contracts $T_{i_2}^{r_4}$ and $T^{r_1}_{B}$, and moreover $r_4= r_3 =0$, so we cannot apply the results for controlling the norms of tensors. Note that if $r=0$, we have:
\begin{equation}
    T_{i}^{r} = A_i^{(1)}\tilde{\otimes} A_i^{(1)}.
\end{equation}
Then; 
\begin{equation}
    V^{i_1,i_2}_{3} =\sum_{r_1=1}^{q-1} \sum_{r_3=1}^{2q-2r_1}\sum_{r_5=1}^{6q-2\max(r_4,r_1)}c_{q,r} I_{6q-2(\sum_{\ell =1}^{5} r_i)} \left (T_{i_2}^{0} \tilde{\otimes}_{r_5}^{B} \left ( T^{0}_{i_1} \tilde{\otimes}_{r_3}T^{r_1}_{B}\right )\right ),
\end{equation}
where we used the upper index $\otimes^{B}$ to denotes that the $r_{5}$ contractions are only  between elements of $T_{i_2}^{0}$ and $T^{r_1}_{B}$. 
Then:
\begin{align}
    \sum_{i,j=1}^{d_1} \lambda_i \lambda_j V^{i,j}_{3}  & =  \sum_{i,j=1}^{d_1} \lambda_i \lambda_j  \sum_{r_1=1}^{q-1} \sum_{r_3=1}^{2q-2r_1}\sum_{r_5=1}^{6q-2\max(r_4,r_1)}c_{q,r} I_{6q-2(\sum_{\ell =1}^{5} r_i)} \left (T_{i}^{0} \tilde{\otimes}_{r_5}^{B} \left ( T^{0}_{j} \tilde{\otimes}_{r_3}T^{r_1}_{B}\right )\right ) \\
    & =   \sum_{r_1=1}^{q-1} \sum_{r_3=1}^{2q-2r_1}\sum_{r_5=1}^{6q-2\max(r_4,r_1)}c_{q,r} I_{6q-2(\sum_{\ell =1}^{5} r_i)} \left (\sum_{i,j=1}^{d_1}\lambda_i \lambda_jT_{i}^{0} \tilde{\otimes}_{r_5}^{B} \left (  T^{0}_{j} \tilde{\otimes}_{r_3}T^{r_1}_{B}\right )\right ).
\end{align}
Then, going to \Cref{eq:operator_norm_subtraction_5}, we can apply Cauchy Schwartz to get: 
\begin{align}
    \Delta &\lesssim \max_{B \in \RR^{B(d,q)}, \| B\|_2 = 1} \EE \left [ \left (\sum_{r_1=1}^{q-1} \sum_{r_3=1}^{2q-2r_1}\sum_{r_5=1}^{6q-2\max(r_4,r_1)}c_{q,r} I_{6q-2(\sum_{\ell =1}^{5} r_i)} \left (\sum_{i,j=1}^{d_1}\lambda_i \lambda_jT_{i}^{0} \tilde{\otimes}_{r_5}^{B} \left (  T^{0}_{j} \tilde{\otimes}_{r_3}T^{r_1}_{B}\right )\right ) \right )^{2}\right ]^{\frac{1}{2}}+o_{d}(1) \\
    & \lesssim \max_{B \in \RR^{B(d,q)}, \| B\|_2 = 1}  \sum_{r_1=1}^{q-1} \sum_{r_3=1}^{2q-2r_1}\sum_{r_5=1}^{6q-2\max(r_4,r_1)}c_{q,r}\EE \left [I_{6q-2(\sum_{\ell =1}^{5} r_i)} \left (\sum_{i,j=1}^{d_1}\lambda_i \lambda_jT_{i}^{0} \tilde{\otimes}_{r_5}^{B} \left (  T^{0}_{j} \tilde{\otimes}_{r_3}T^{r_1}_{B}\right )\right )^{2}\right ]^{\frac{1}{2}}+o_{d}(1) \\
    & \lesssim \max_{B \in \RR^{B(d,q)}, \| B\|_2 = 1}  \sum_{r_1=1}^{q-1} \sum_{r_3=1}^{2q-2r_1}\sum_{r_5=1}^{6q-2\max(r_4,r_1)}c_{q,r}\EE \left [I_{6q-2(\sum_{\ell =1}^{5} r_i)} \left (\sum_{i,j=1}^{d_1}\lambda_i \lambda_jT_{i}^{0} \tilde{\otimes}_{r_5}^{B} \left (  T^{0}_{j} \tilde{\otimes}_{r_3}T^{r_1}_{B}\right )\right )^{2}\right ]^{\frac{1}{2}}+o_{d}(1)
    \label{eq:operator_norm_subtraction_6}
\end{align}
For a fixed \textbf{Case 3} pattern of contractions \(\alpha\), define
\begin{equation}
K_{B,\alpha} =\sum_{i,j=1}^{d_1}
\lambda_i\lambda_j\,T_i^0\tilde\otimes_{r_5}^{B,\alpha}\left(T_j^0\tilde\otimes_{r_3}^{\alpha}
T_B^{r_1}
\right).    
\end{equation}
Where we recall that the upper index \(B,\alpha\) means that the \(r_5\) contractions are only between \(T_i^0\) and the \(T_B^{r_1}\). Then, computing the expectation
\begin{equation}
    \EE \left [I_{\alpha} \left (\sum_{i,j=1}^{d_1}\lambda_i \lambda_jT_{i}^{0} \tilde{\otimes}_{r_5}^{B} \left (  T^{0}_{j} \tilde{\otimes}_{r_3}T^{r_1}_{B}\right )\right )^{2}\right ]^{\frac{1}{2}} \leq C \|K_{B,\alpha}\|_F
\end{equation}

Recall that $T_i^0 = A_{i}^{(1)} \otimes A_{i}^{(1)}$. Then, each kernel $K_{B,\alpha}$ has the form: 
\begin{equation}
    K_{B,\alpha} =\sum_{i,j=1}^{d_1}
\lambda_i\lambda_j\,(A_{i}^{(1)} \otimes A_{i}^{(1)})\tilde\otimes_{r_5}^{B,\alpha}\left((A_{i}^{(1)} \otimes A_{i}^{(1)})^0\tilde\otimes_{r_3}^{\alpha}
T_B^{r_1}
\right) =  (\sum_{i=1}^{d_1}
\lambda_iA_{i}^{(1)} \otimes A_{i}^{(1)})\tilde\otimes_{r_5}^{B,\alpha}\left(\sum_{j=1}^{d_1}
\lambda_j(A_{j}^{(1)} \otimes A_{j}^{(1)})^0\tilde\otimes_{r_3}^{\alpha}
T_B^{r_1}
\right),
\end{equation}
where in the last equality we used bi-linearity of the contractions. Now, let $\alpha$ be an admissible contraction pattern from \textbf{Case 3}. Note that $\alpha$ has to specify which $r_3$ indices of $T_{B}$ are contracted with $T_{j}$, and which $r_5$ indices are contracted with the $T_i$ (out of the remaining $2q-r_3$ indices of $T_{B}$). Then we can write: 
\begin{equation}
     K_{B,\alpha} = \mathcal{F}_{\alpha, r_5}\left [\sum_{i=1}^{d_1}
\lambda_i A_{i}^{(1)} \otimes A_{i}^{(1)} \right ] \circ  \mathcal{F}_{\alpha, r_3}\left [\sum_{j=1}^{d_1}
\lambda_j A_{j}^{(1)} \otimes A_{j}^{(1)} \right ] \mathcal{F}_{\alpha, r_1} \left [ 
T_B^{r_1} \right ],
\end{equation}
where $\mathcal{F}_{\alpha, r_5}$, $\mathcal{F}_{\alpha, r_3}$ and $\mathcal{F}_{\alpha, r_5}$ are deterministic flattenings specified by the particular choice of $\alpha$. Then, applying Cauchy-Schwarz:
Then:
\begin{equation}
    \|K_{B,\alpha}\|_F \lesssim \left \| \mathcal{F}_{\alpha, r_5}\left [\sum_{i=1}^{d_1}
\lambda_i A_{i}^{(1)} \otimes A_{i}^{(1)} \right ]\right \|_\op \left \|  \mathcal{F}_{\alpha, r_3}\left [\sum_{j=1}^{d_1}
\lambda_j A_{j}^{(1)} \otimes A_{j}^{(1)} \right ]\right \|_\op \| \mathcal{F}_{\alpha, r_1} \left [ 
T_B^{r_1} \right ] \|_{F},
\end{equation}
and since $\|B\|=1$, we conclude: 
\begin{equation}
    \|K_{B,\alpha}\|_F \lesssim \left \| \mathcal{F}_{\alpha, r_5}\left [\sum_{i=1}^{d_1}
\lambda_i A_{i}^{(1)} \otimes A_{i}^{(1)} \right ]\right \|_\op \left \|  \mathcal{F}_{\alpha, r_3}\left [\sum_{j=1}^{d_1}
\lambda_j A_{j}^{(1)} \otimes A_{j}^{(1)} \right ]\right \|_\op.
\end{equation}
Let $T = \sum_{j=1}^{d_1}
\lambda_j A_{j}^{(1)} \otimes A_{j}^{(1)} \in (\RR^{d})^{\odot 2q}$. The map $\mathcal{F}_{\alpha, r}$ transforms $T \to \mathcal{F}_{\alpha, r}[T] \in (\RR^d){R} \times (\RR^{d})^{C}$, for some sets $R, C$ (denoted like that for rows and columns, respectively). Denote $A, B$ for the row and column indices of $T$. There are to possible cases: $\mathcal{F}_{\alpha,r}$ can have $A \cap R = A, B \cap C= C$ when the map preserves columns and rows, or $A \cap R \subset A$ strictly, when it doesn't. In the first case, we have:
\begin{equation}
    \|\mathcal{F}_{\alpha,r}[T] \|_\op \leq \|T\|_\op, 
\end{equation}
as $\mathcal{F}_{\alpha,r}$ is just a permutation matrix. By re-writing $T = (A^{(1)})^T D A^{(1)}$, and using the fact that the vectors in $A^{(1)}$ are almost orthonormal, we conclude that with high probability
\begin{equation}
    \| T\|_\op = \left \| \sum_{i=1}^{d_1}
\lambda_i A_{i}^{(1)} \otimes A_{i}^{(1)} \right \|_\op \lesssim \max |\lambda_{i}| ,
\label{eq:flattening_type_1}
\end{equation}
and since $\gamma < \frac{1}{2}$, $\max |\lambda_{i}| = |\lambda_1 |= Z_\gamma = \Theta_{d}(d_1^{\frac{1-2\alpha}{2}})$. \\

Now, for the second type of flattening, the map $\mathcal{F}$ gives: 
\begin{equation}
    \mathcal{F}_{\alpha, t}[T] = \sum_{i=1}^{d_1} \lambda_i \mathcal{F}_{\alpha, t}[ A_{i}^{(1)} \otimes A_i^{(1)}] = \sum_{i=1}^{d_1} \lambda_i \mathrm{Perm} \left ( (A_{i}^{(1)})^{s} \otimes (A_{i}^{(1)})^{t}  \right ),
\end{equation}
where $\mathrm{Perm}$ denotes a permutation of indices, and $A_{i}^{(1)})^{s}$ denotes a flattening of $A_i^{(1)}$ as a matrix of size $d^{s} \times d^{q-s}$, for $s \geq 1$. Then we will get: 
\begin{equation}
    \| (A_{i}^{(1)})^{s} \otimes (A_{i}^{(1)})^{t} \|_\op \leq \|(A_{i}^{(1)})^{s}\|_\op \| (A_{i}^{(1)})^{t} \|_\op. 
\end{equation}
Since $(A_{i}^{(1)})$ has Gaussian entries, we can apply Theorem 5.39 in \cite{V10}, to get, with high probability, 
\begin{equation}
    \| (A_{i}^{(1)})^{t} \|_\op \lesssim d^{-\frac{\min(q-t, t)}{2}} \leq d^{-\frac{1}{2}}, 
    \label{eq:flattened_A_bound}
\end{equation}
if $t \in [q-1]$. Then $\| (A_{i}^{(1)})^{s} \otimes (A_{i}^{(1)})^{t} \|_\op \lesssim \frac{1}{\sqrt{d}}$ for our ranges of $r$. We now apply this on $ \|\mathcal{F}_{\alpha, t}[T]\|_\op$. Since $\lambda_i$ have random sings, we can first apply Matrix Bernstein with the tensors $A^{(1)}_i$ fixed to get, and then triangular inequality: 
\begin{align}
    \| \mathcal{F}_{\alpha,r}[T]\|_\op &\lesssim \|Z_{\gamma}^2\sum_{i=1}^{d} i^{-2\gamma}  \mathrm{Perm} \left ( (A_{i}^{(1)})^{s} \otimes (A_{i}^{(1)})^{t} \right ) \| \\
    & \lesssim \sum_{i=1}^{d}i^{-2\gamma} \| (A_{i}^{(1)})^{s} \otimes (A_{i}^{(1)})^{t} \|_\op \lesssim \dfrac{1}{\sqrt{d}},
    \label{eq:flattening_type_2}
\end{align}
where we used that $Z_{\gamma}^2\sum_{i=1}^{d} i^{-2\gamma}  = 1$, and \Cref{eq:flattened_A_bound}. Putting both cases (\Cref{eq:flattening_type_1} and \Cref{eq:flattening_type_2}) together: 
\begin{equation}
    \|K_{B,\alpha}\|_F \lesssim Z_{\gamma}^2 = \dfrac{1}{d^{1-2\gamma}}.  
\end{equation}
for all patterns $\alpha$ in \textbf{Case 3}. Then, replacing in \Cref{eq:operator_norm_subtraction_6}, we get: 
\begin{equation}
    \|\Delta\|_\op \lesssim \dfrac{1}{d_1^{1-2\gamma}},
\end{equation}
with high probability. Therefore, we conclude that: 
\begin{equation}
    \| \EE\left [\hat{C} \right ] -  \EE[g'(h^2)] \sum_{i=1}^{d_1} \lambda_i A^{(1)}_i (A 
_i^{(1)})^T\|_\op \lesssim  \Delta \lesssim \dfrac{1}{d_1^{1-2\gamma}},
\end{equation}
which concludes the proof. 

\subsubsection{\texorpdfstring{Proof of \Cref{lemma:expectation_C_1}}{Proof of Lemma}}

Putting \Cref{lemma:expectation_C_1_2} and \Cref{lemma:information_exponent_g} together, we can conclude:
\begin{corollary}
Under the assumptions of \Cref{thm:sample_complexity}: 
\[
\EE \left [C^{(1)} \right ] = \dfrac{\nu_{1}}{\sqrt{2}}A^{(1)} D_{\gamma} (A^{(1)})^T + \Delta, 
\]
where $\| \Delta \|_\op = o_{d}(1)$,  $A^{(1)} = [u_1, \dots u_{d_1}] \in \mathbb{R}^{D \times d_1}$,  $D_\gamma = \mathrm{diag} ( \lambda_1, \dots, \lambda_{d_1}) \in \RR^{d_1 \times d_1}$, and $\nu_{1}$ is the first Hermite coefficient of $g$.
\end{corollary}

\section{Deferred Proofs}
\label{seciton:deferred_proofs}
\begin{lemma}
    Let $k,j \in [d_1]$, with $k \not =j$. Assume $n = \omega_{d}(\frac{\min(k,j)^{2\gamma}}{Z_\gamma ^{2}}d^{q} )$. Then with high probability 
    $$
        \left ( \dfrac{Z_\gamma }{n} y_\mu h_{\mu,k} h_{\mu,j} \right )^{2} \lesssim \dfrac{Z_\gamma ^2}{d^{q}} \min(k,j)^{-2\gamma}. 
    $$
    \label{lemma:bound_square_y_s}
\end{lemma}

\begin{proof}
    We want to concentrate
    \begin{align}
         E^{2} =\left ( \dfrac{Z_\gamma }{n} y_\mu h_{\mu,k} h_{\mu,j} \right )^{2}. 
    \end{align}
     For this, we begin by decomposing: 
    \begin{align}
        E^{2} & = \left (  \dfrac{1}{n} \sum_{\mu=1}^{n} \left ( y_\mu h_{\mu,k} h_{\mu,j} -\EE\left [y_\mu h_{\mu,k} h_{\mu,j}\right ]\right ) + \EE \left [ y_\mu h_{\mu,k}h_{\mu,k} \right ] \right )^{2} \\ 
        & \lesssim \underbrace{\left ( \dfrac{1}{n} \sum_{\mu=1}^{n} \left ( y_\mu h_{\mu,k} h_{\mu,j} -\EE\left [y_\mu h_{\mu,k} h_{\mu,j}\right ]\right )\right )^{2}}_{(I)} + \underbrace{\EE \left [ y_\mu h_{\mu,k}h_{\mu,k} \right ] ^{2}}_{(II)}. 
    \end{align}
    We begin by concentrating $(I)$. For this, we note that since all terms are centered, independent random variables: 
    \begin{align}
        \EE \left [ (I) \right ] & = \dfrac{1}{n^{2}} \sum_{\mu =1 }^{n}  \EE \left [ \left ( y_\mu h_{\mu,k} h_{\mu,j} -\EE\left [y_\mu h_{\mu,k} h_{\mu,j}\right ]\right )^{2}\right ].
    \end{align}
    Note that, since $\mathrm{Var}(y_\mu) \lesssim 1$, we have: 
    \begin{align}
        \EE \left [ \left ( y_\mu h_{\mu,k} h_{\mu,j} -\EE\left [y_\mu h_{\mu,k} h_{\mu,j}\right ]\right )^{2}\right ] \lesssim \EE \left [ \left ( y_\mu h_{\mu,k} h_{\mu,j} \right )^{2} \right ] \lesssim 1. 
    \end{align}
    Then
    \begin{equation}
        \EE \left [ (I) \right ] \lesssim \dfrac{1}{n},
    \end{equation}
    and since $(I)$ is positive, by Markov inequality we conclude that with high probability: 
    \begin{equation}
        \left ( \dfrac{Z_\gamma }{n} y_\mu h_{\mu,k} h_{\mu,j} \right )^{2} \lesssim \dfrac{1}{n}  + \EE \left [ y_\mu h_{\mu,k}h_{\mu,k} \right ] ^{2}. 
        \label{eq:bound_I_sum_y}
    \end{equation}
    We now turn to $(II)= \EE \left [ y_\mu h_{\mu,k}h_{\mu,k} \right ] ^{2}$. Recall that, by definition: 
    \begin{equation}
        y_\mu = g\left ( \sum_{p=1}^{d_1}  \lambda_p ((h^{(1)}_{\mu, p})^{2} - 1) \right ),
    \end{equation}
    for some polynomial $g:\RR \to \RR$, with $\EE[g'(h^{(2)})] \not = 0$. Then, by \Cref{lemma:non_linear_case_cross_expectations}
    \begin{align}
         \left | \EE \left [ y_\mu h_{\mu,k}h_{\mu,k} \right ] \right | &= O \left ( \dfrac{Z_\gamma }{d^{\frac{q}{2}}} \min(k,j)^{-\gamma} \right ). 
    \end{align}
    Taking the square:
    \begin{equation}
        \EE \left [ y_\mu h_{\mu,k}h_{\mu,k} \right ]^{2} \lesssim \dfrac{Z_\gamma ^2}{d^{q}} \min(k,j)^{-2\gamma}. 
        \label{eq:bound_II_sum_y}
    \end{equation}
    By putting together \Cref{eq:bound_I_sum_y} and \Cref{eq:bound_II_sum_y}: 
    \begin{equation}
         \left ( \dfrac{Z_\gamma }{n} y_\mu h_{\mu,k} h_{\mu,j} \right )^{2} \lesssim \dfrac{1}{n} + \dfrac{Z_\gamma ^2}{d^{q}} \min(k,j)^{-2\gamma}.
    \end{equation}
    and since $n = \omega_{d}(\frac{\min(k,j)^{2\gamma}}{Z_\gamma ^{2}}d^{q} )$,  we conclude that with high probability: 
    \begin{equation}
        \left ( \dfrac{Z_\gamma }{n} y_\mu h_{\mu,k} h_{\mu,j} \right )^{2} \lesssim \dfrac{Z_\gamma ^2}{d^{q}} \min(k,j)^{-2\gamma}. 
    \end{equation}
\end{proof}

\subsection{Gaussian Universality}

In the following, let $\mathcal{W}_{1}$ denote the $1$-Wasserstein distance on $\mathcal{P}_{1}(\RR^{r})$. That is,  for $\mu, \nu \in \mathcal{P}_{1}(\RR^{r})$, $\mathcal{W}_{1}$ defines the metric: 
\begin{equation}
    \mathcal{W}_{1}(\mu, \nu ) = \sup_{h: \RR^{r} \to \RR, \mathrm{Lip}(h) \leq 1} \left | \EE_{G \sim \mu} [h(G)] - \EE_{Z \sim \nu} [h(Z)]\right |. 
\end{equation}

We will need the following result bounding the Wasserstein distance to Gaussians. 
\begin{lemma}[Theorem 5.1.3 in \cite{nourdin2012normal}] Let $F \in \mathbb{D}^{1,2}$ with $E[F] = 0$ and $E[F^2] = \sigma^2 > 0$ , and let $N \sim \mathcal{N}(0, \sigma^2)$ . Then
$$
d_{\mathrm{W}}(F, N) \leq \frac{\sqrt{2}}{\sigma \sqrt{\pi}} \, E\left[\left|\sigma^2 - \langle DF, -DL^{-1}F \rangle_{\mathfrak{F}}\right|\right]. 
$$
\label{lemma:one_dimensional_CLT}
\end{lemma}

\begin{lemma}
Let $r \in \NN$, and let $A^{(1)}_{1}, \dots, A^{(1)}_{r} \in (\RR^{d})^{\otimes q}$ be symmetric tensors of order $q$ such as the ones specified in \Cref{section:setting}. Let $x \sim \mathcal{N}(0,I_{d})$, and let $H_{k}(x)$ denote the degree $k$ Hermite tensor of $x$. Define $h^{(2)}$ as in \Cref{section:setting}. Then, then there exists a constant $C_{k} < \infty$, depending only on $k$, such that: 
\begin{equation}
    \mathcal{W}_1 \left ( h^{(2)} , N \right ) \leq C_{k}\dfrac{1}{\sqrt{d}},
\end{equation}
where $N \sim \mathcal{N}(\EE[h^{(2)}], \mathrm{Var(h^{(2)})})$.
\label{lemma:multivariate_gaussian_approx_h_2}
\end{lemma}
\begin{proof}The proof is similar to Lemma A.1 in \cite{tabanelli2026deep}, but instead of using the multi-variate Gaussian approximation Lemma, we use the one-dimensional version \Cref{lemma:one_dimensional_CLT}. We begin by computing the Wiener expansion of $h^{(2)}$. We have: 
\begin{align}
    h^{(2)} = \sum_{i =1}^{d_1} \lambda_{i} \left ( I_{q}(A^{(1)}_i)^{2} - 1\right ) = \sum_{i=1}^{d_1} \lambda_{i} (I_{q}(A_{i})^{2} - \| A_i\|^{2}) + \sum_{i=1}^{d} \lambda_{i} (\| A_i\|^{2} - 1).
\end{align}
By the product formula \Cref{lemma:product_formula}, and the fact that the first term is centered:  
\begin{align}
    h^{(2)} &= \sum_{i=1}^{d_1} \lambda_{i} \sum_{r= 0}^{q-1}c_{q,r}I_{2q-2r}(A_i \tilde{\otimes}_{r} A_i) + \sum_{i=1}^{d} \lambda_{i} (\| A_i\|^{2} - 1) \\
    & =  \sum_{r= 0}^{q-1}I_{2q-2r}(c_{q,r}\sum_{i=1}^{d_1} \lambda_{i}A_i \tilde{\otimes}_{r} A_i) + \sum_{i=1}^{d} \lambda_{i} (\| A_i\|^{2} - 1).
\end{align}
From here, we have: 
\begin{equation}
    \EE \left[ h^{(2)} \right ] = \sum_{i=1}^{d} \lambda_{i} (\| A_i\|^{2} - 1), \text{ and} \EE\left[ (h^{(2)})^{2} \right ] = \sum_{r= 0}^{q-1}c_{q,r}^{2}\| \sum_{i=1}^{d_1} \lambda_{i}A_i \tilde{\otimes}_{r} A_i\|^{2}_{F}. 
\end{equation}
Now, we want to apply \Cref{lemma:one_dimensional_CLT}. By centering, we have to control  $\langle D\bar{h^{(2)}}, DL^{-1}\bar{h^{(2)}}\rangle$, for $\bar{h^{(2)}} = h^{(2)} - \EE[h^{(2)}]$. Denote $B_{r} = \sum_{i=1}^{d_1} c_{q,r}\lambda_i A_i \tilde{\otimes}_{r}A_i$. By the linearity of the derivative:
\begin{align}
    \langle D\bar{h^{(2)}}, DL^{-1}\bar{h^{(2)}}\rangle & = \sum_{r,r'= 0}^{q-1}\langle D I_{2q-2r}(B_{r}), DI_{2q-2r'}(B_{r'}) \rangle, 
\end{align}
and by applying \Cref{lemma:computation_derivatives}: 
\begin{align}
     \langle D\bar{h^{(2)}}, DL^{-1}\bar{h^{(2)}}\rangle & = \sum_{r,r'= 0}^{q-1} \sum_{p=1}^{2q-2\max(r_1, r_2)}c_{q,r,r',p}I_{4q-2r-2r'-2p}(B_{r} \tilde{\otimes}_{p} B_{r'}). 
\end{align}
Now, let 
\begin{equation}
    K_{s} = \sum_{r,r'= 0}^{q-1} \sum_{p=1}^{2q-2\max(r_1, r_2)}c_{q,r,r',p} \mathbf{1}_{4q-2r-2r'-2p = s} B_{r} \tilde{\otimes}_{p} B_{r'}. 
\end{equation}
Then: 
\begin{equation}
    \langle D\bar{h^{(2)}}, DL^{-1}\bar{h^{(2)}}\rangle = \sum_{s \geq 0} I_{s} (K_{s}).
\end{equation}
Then, we have: 
\begin{align}
    \EE \left[ \left (\mathrm{Var}(h^{(2)}) - \langle Dh^{(2)}, DL^{-1}h^{(2)}\rangle \right )^{2}\right ] & = \sum_{s \geq 1} \|K_{s}\|^{2}_{F}. 
\end{align}
Then to conclude, we need to show that this norms are negligible for large $d$. Let $s \geq 0$. We have: 
\begin{align}
    \| K_{s}\|_F^{2} &= \|  \sum_{r,r'= 0}^{q-1} \sum_{p=1}^{2q-2\max(r_1, r_2)}c_{q,r,r',p} \mathbf{1}_{4q-2r-2r'-2p = s} \|B_{r} \tilde{\otimes}_{p} B_{r'} \|_F^{2} \\
    & \leq C \sum_{r,r'= 0}^{q-1} \sum_{p=1}^{2q-2\max(r_1, r_2)}\mathbf{1}_{4q-2r-2r'-2p = s} \left \|B_{r} \tilde{\otimes}_{p} B_{r'} \right \|_F^{2}.
\end{align}
Recall $B_{r} = \sum_{i=1}^{d_1} c_{q,r}\lambda_i A_i \tilde{\otimes}_{r}A_i$. Denote $T^r_{i} = A_i \tilde{\otimes}{r}A_i$. Then: 
\begin{align}
    B_{r} \tilde{\otimes}_{p} B_{r'} = \sum_{i,j} \lambda_i \lambda_j T^r_i \tilde{\otimes}_{p} T^{r}_j,
\end{align}
and: 
\begin{align}
    \| K_{s}\|_{F}^{2} & \leq C \sum_{r,r'= 0}^{q-1} \sum_{p=1}^{2q-2\max(r_1, r_2)}\mathbf{1}_{4q-2r-2r'-2p = s} \sum_{i,j} |\lambda_i|^2 |\lambda_j|^2 \|T^r_i \tilde{\otimes}_{p} T^{r}_j\|_{F}^2,
    \label{eq:bound_frobenius_K_s}
\end{align}
so everything reduces to estimating the norms $\|T^r_i \tilde{\otimes}_{p} T^{r}_j\|_{F}$. 

Ignoring symmetrization, which only changes things up to constants, given $a, b \in [d]^{q-r}$, we can apply \Cref{lemma:double_contraction_expectation} for $i=j$ and \Cref{lemma:double_contraction_expectation_cross} for $i \not = j$. From which we will obtain: 
\begin{equation}
    \EE \left [\|T^r_i \tilde{\otimes}_{p} T^{r}_j\|_{F}^{2} \right ] = O\left (\dfrac{1}{d} \right ) \forall i,j \in [d_1]
\end{equation}
Since $\|T^r_i \tilde{\otimes}_{p} T^{r}_j\|_{F}^{2}$ is a polynomial of Gaussians, we can use \Cref{lemma:gaussian_hypercontractivity} to control its moments, and applying Chebyshev we will get: 
\begin{equation}
    \|T^r_i \tilde{\otimes}_{p} T^{r}_j\|_{F}^2 \lesssim \dfrac{1}{\sqrt{d}},
\end{equation}
with high probability over $A_1^{(1)}, \dots, A^{(1)}_{d_1}$. Replacing in \Cref{eq:bound_frobenius_K_s}:
\begin{align}
    \| K_{s}\|_{F}^{2} & \lesssim \dfrac{1}{\sqrt{d}}C \sum_{r,r'= 0}^{q-1} \sum_{p=1}^{2q-2\max(r_1, r_2)}\mathbf{1}_{4q-2r-2r'-2p = s} \sum_{i,j} |\lambda_i|^2 |\lambda_j|^2.
\end{align}
Note that, since $\gamma < \frac{1}{2}$
\begin{equation}
    \sum_{i,j} |\lambda_i|^2 |\lambda_j|^2 = Z_{\gamma}^{4} \left (\sum_{i =1}^{d_1} i^{-2\alpha} \right )^2 = \Theta_{d}\left ( \dfrac{d^{2(1-2\alpha)}}{d^{4\frac{(1-2\alpha)}{2}}}\right ) = \Theta_{d}(1). 
\end{equation}
From this, we can conclude: 
\begin{equation}
    \| K_{s}\|_{F}^{2} \lesssim \dfrac{1}{\sqrt{d}} \forall s,
\end{equation}
with high probability. Finally, this allows us to conclude:
\begin{equation}
    \EE \left[ \left (\mathrm{Var}(h^{(2)}) - \langle Dh^{(2)}, DL^{-1}h^{(2)}\rangle \right )^{2}\right ] = O \left ( \dfrac{1}{\sqrt{d}}\right ),
\end{equation}
with high probability over $A_1^{(1)}, \dots, A^{(1)}_{d_1}$. 
\end{proof}

With this, we conclude the following Corollary, stated as \Cref{lemma:information_exponent_g} in \Cref{section:explicit_computation_hermite}. 
\begin{corollary}
    Let $g:\RR \to \RR$ be a polynomial with information exponent $1$. Assume $\gamma < \frac{1}{2}$. Then:
    \begin{equation}
        \EE[g(h^{(2)})]= \dfrac{1}{\sqrt{d}} \quad \text{ and } \EE \left[ g'(h^{(2)}) \right ] = \nu_{1} + \dfrac{C}{\sqrt{d}},
    \end{equation}
    where $\nu_1$ is the first Hermite coefficient of $g$. 
    \label{Cor:information_exponent_g}
\end{corollary}

\begin{proof}[Sketch of the Proof]
The idea of the proof is to approximate $g$ by a Lipschitz function which has bounded support and apply \Cref{lemma:multivariate_gaussian_approx_h_2}.
\end{proof}

\begin{lemma}[Lemma A.3 in \cite{tabanelli2026deep}]
Let $A, B$ be independent tensors in $(\mathbb{R}^d)^{\otimes k}$ with i.i.d. entries $\mathcal{N}(0, d^{-k})$ . Then for each $s \in \{1, \dots, k\}$,

$$
\mathbb{E}\|A \otimes_s B\|_F^2 = \Theta(d^{-s}).
$$

While for the self-contractions, $E\|A \otimes_s A\|_F^2 = \Theta(d^{-s})$ for $s \in \{1, \dots, k-1\}$ and $E\|A \otimes_k A\|_F^2 = 1$.
\label{lemma:simple_self_and_cross_contractions}
\end{lemma}

\section{Technical Lemmas}

\subsection{Cumulants}

\begin{definition}[Definition 8.2.1 in \cite{nourdin2012normal}]
    Let $F = (F_1, \ldots, F_N)$ be an $\mathbb{R}^N$-valued random vector with $F_i \in \mathbb{D}^{1,2}$ for each $i$. 
    Let $l_1, l_2, \ldots$  be a sequence taking values in the multi-index set  $\{e_1, \ldots, e_N\}$. We set $\Gamma_{l_1}(F) = F^{l_1} = F_j$, where $j$  is such that $l_1 = e_j$. If the random variable $\Gamma_{l_1, \ldots, l_k}(F)$ is a well-defined element of $L^2(\Omega)$ for some $k \geq 1$,  we set
    $$
    \Gamma_{l_1, \ldots, l_{k+1}}(F) = \langle DF^{l_{k+1}}, -DL^{-1}\Gamma_{l_1, \ldots, l_k}(F) \rangle_{\mathfrak{H}}.
    $$
    \label{def:gamma_l_1_l_N}
\end{definition}

\begin{lemma}[Theorem 8.2.5 in \cite{nourdin2012normal}]
Let $m = (m_1, \ldots, m_d) \in \mathbb{N}^d \setminus \{0\}$ be a multi-index. Write $m = l_1 + \ldots + l_{|m|}$ where the multi-indices $l_i \in \{e_1, \ldots, e_d\}$ , $i = 1, \ldots, |m|$ , are unique in the sense of Lemma 8.1.1. Suppose that the random vector $F = (F_1, \ldots, F_d)$ is such that $F_i \in \mathbb{D}^{|m|, 2^{|m|}}$ for each $i$ . Then

$$
\kappa_m(F) = \sum_{\sigma \in \mathfrak{S}_{\{2, \ldots, |m|\}}} \EE\left[\Gamma_{l_1, l_{\sigma(2)}, \ldots, l_{\sigma(|m|)}}(F)\right].
$$    
\label{lemma:cumulants_and_Gamma}
\end{lemma}

\begin{lemma}[Theorem 8.3.1 in \cite{nourdin2012normal}]
Let $m \in \mathbb{N}^d \setminus \{0\}$ be a multi-index such that $|m| \geq 3$ . Write $m = l_1 + \ldots + l_{|m|}$ with $l_i \in \{e_1, \ldots, e_d\}$ for each $i$ (see Lemma 8.1.1). Consider an $\mathbb{R}^d$ -valued random vector of the form

$$
F = (F_1, \ldots, F_d) = (I_{q_1}(f_1), \ldots, I_{q_d}(f_d)),
$$

where each $f_i$ belongs to $\mathfrak{H}^{\odot q_i}$ . When $l_k = e_j$ , we set $\lambda_k = j$ , so that $F^{l_k} = F_{\lambda_k}$ for all $k = 1, \ldots, |m|$ . Then

$$
\kappa_m(F) = \sum_{\sigma \in \mathfrak{S}_{\{2, \ldots, |m|\}}} (q_{\lambda_\sigma(|m|)})! \sum_* c_{q,l,\sigma}(r_2, \ldots, r_{|m|-1})
$$

$$
\times \langle (\ldots ((f_{\lambda_1} \widetilde{\otimes}_{r_2} f_{\lambda_\sigma(2)}) \widetilde{\otimes}_{r_3} f_{\lambda_\sigma(3)}) \ldots) \widetilde{\otimes}_{r_{|m|-1}} f_{\lambda_\sigma(|m|-1)}; f_{\lambda_\sigma(|m|)} \rangle_{\mathfrak{H}^{\otimes q_{\lambda_\sigma(|m|)}}},
$$

where the second sum $\sum_*$ runs over all collections of integers $r_2, \ldots, r_{|m|-1}$ such that:

(i) $1 \leq r_i \leq q_{\lambda_\sigma(i)}$ for all $i = 2, \ldots, |m|-1$ ;

(ii) $r_2 + \ldots + r_{|m|-1} = \frac{q_{\lambda_1} + q_{\lambda_\sigma(2)} + \ldots + q_{\lambda_\sigma(|m|-1)} - q_{\lambda_\sigma(|m|)}}{2}$ ;

(iii) $r_2 < \frac{q_{\lambda_1} + q_{\lambda_\sigma(2)}}{2}, \ldots, r_2 + \ldots + r_{|m|-2} < \frac{q_{\lambda_1} + q_{\lambda_\sigma(2)} + \ldots + q_{\lambda_\sigma(|m|-2)}}{2}$ ;

(iv) $r_2 \leq q_{\lambda_1}, r_3 \leq q_{\lambda_1} + q_{\lambda_\sigma(2)} - 2r_2, \ldots, r_{|m|-1} \leq q_{\lambda_1} + q_{\lambda_\sigma(2)} + \ldots + q_{\lambda_\sigma(|m|-2)} - 2r_2 - \ldots - 2r_{|m|-2}$ ;
\label{lemma:cummulants}
\end{lemma}

\subsection{Gaussian Tensors}

\begin{lemma}[Theorem 3.1.1 in \cite{de2012decoupling}] For natural numbers $n \ge m$, let $\{X_i\}_{i=1}^n$ be $n$ independent random variables with values in a measurable space $(S, \delta)$, and let $\{X_i^k\}_{i=1}^n$, $k = 1, \ldots, m$, be $m$ independent copies of this sequence. Let $B$ be a separable Banach space and, for each $(i_1, \ldots, i_m) \in I_n^m$, let $h_{i_1 \ldots i_m} : S^m \mapsto B$ be measurable functions such that $\mathbb{E}\big(\|h_{i_1 \ldots i_m}(X_{i_1}, \ldots, X_{i_m})\|\big) < \infty$. Let $\Phi : [0, \infty) \to [0, \infty)$ be a convex non-decreasing function such that $\mathbb{E}\Phi\big(\|h_{i_1 \ldots i_m}(X_{i_1}, \ldots, X_{i_m})\|\big) < \infty$ for all $(i_1, \ldots, i_m) \in I_n^m$. Then,

$$
\mathbb{E}\Phi\Big(\Big\|\sum_{I_n^m} h_{i_1 \ldots i_m}(X_{i_1}, \ldots, X_{i_m})\Big\|\Big) \le \mathbb{E}\Phi\Big(C_m \Big\|\sum_{I_n^m} h_{i_1 \ldots i_m}(X_{i_1}^1, \ldots, X_{i_m}^m)\Big\|\Big).
$$
 \label{lemma:decoupling}
\end{lemma}

\begin{lemma}
    Let $A \in (\RR^{d})^{\odot q}$ be a symmetric tensor with independent gaussian centered entries with variance $d^{-q}$. Let $r,r' \in [q-1]\cup\{0\}$, and let $p \in [2q - 2\max(r_1,r_2)]$. 
    Then:
    \[
    \EE \left [ \| (A \otimes_{r_1} A) \otimes_{r_3} (A \otimes_{r_2} A)\|^{2}\right ] = O \left( \dfrac{1}{d}\right ).
    \]
\label{lemma:double_contraction_expectation}
\end{lemma}

\begin{proof}
Let $M = (r_1 + r_2 + r_3)$, and $N = 4q - 2(r_1 + r_2 + r_3 ) = 4q-2M$. Let $F : (\RR^{d})^{\otimes q} \to (\RR^{d})^{ \otimes N}$ be defines by: 
\begin{equation}
    F(A)_{x} = \sum_{z \in \RR^{N}}  \prod_{\ell = 1}^{4} A_{I_{\ell}(x,z)},
\end{equation}
where each $I_{\ell}(x,z) \in [d]^{q}$ is a $q$-tuple. Given a pair $(x,z)$, let $\pi(x,z)$ define a partition of set $[4]$ by
\begin{equation}
    a \sim_{\pi(x,z)} b \iff I_{a}(x,z) =I_{b}(x,z). 
\end{equation}
Denote by $\hat{\pi}$ the minimal partition, that is $\hat{\pi} = \{ \{1 \}, \{ 2\}, \{ 3\},\{ 4\}\}$. In this partition, we have $I_{1}(x,z) = I_{2}(x,z) = I_{3}(x,z)=I_{4}(x,z)$. Let 
\begin{equation}
    F^{\mathrm{distinct}}(A) = \sum_{z \in \RR^{N}} \mathbf{1}_{\pi(x,z) = \hat{\pi}} \prod_{\ell = 1}^{4} A_{I_{\ell}(x,z)},
\end{equation}
and let 
\begin{equation}
    F^{\mathrm{equal}}(A) = F(A)  - F^{\mathrm{distinct}}(A) = \sum_{z \in [d]^{N}} \mathbf{1}_{\pi(x,z) \not = \hat{\pi}} \prod_{\ell = 1}^{4} A_{I_{\ell}(x,z)}.
\end{equation}
Let $\Phi (x) = \| x\|^{2}$. Then we have: 
\begin{equation}
    \phi(F(A)) \leq \Phi(F^{\mathrm{equal}}(A)) + \phi(F^{\mathrm{distinct}}(A))
\end{equation}

Since $\Phi$ is convex, we can apply \Cref{lemma:decoupling} to obtain:
\begin{equation}
    \EE \left [ \phi(F^{\mathrm{distinct}}(A))\right] \leq C \EE \left [\phi(F^{\mathrm{distinct}}(A^{1}, A^2, A^3,A^4))\right ],
\end{equation}
where $A^\ell, \ell \in [4]$ are independent copies of $A$. Then: 
\begin{align}
    \EE \left [\phi(F^{\mathrm{distinct}}(A, B, C,D))\right ] 
    & = \sum_{x \in [d]^M} \sum_{z, z' \in [d]^N} \prod_{a=1}^4 \mathbb{E} \left[ A_{I_a(x,z)}^{(a)} A_{I_a(x,z')}^{(a)} \right] \\
    & =\lesssim \dfrac{1}{d^{4q}} d^{M + N} = \dfrac{1}{d^{r_1 + r_2 + r_3}}. 
\end{align}
Since $r_3 >0$, this term is at most $O(\frac{1}{d})$.
We now move to $F^{\mathrm{equal}}$. Given $\ell, \ell' \in [4]$, define the sets: 
\begin{equation}
    J_{ab} = \left \{ (x,z): I_{\ell}(x,z) = I_{\ell'}(x,z) \right \},
\end{equation}
and 
\begin{equation}
    J_\mathrm{equal} = \left \{ (x,z): \pi(x,z) \not = \hat{\pi}\right \}. 
\end{equation}
Then, by definition of $\hat{\pi}$: 
\begin{equation}
    J_\mathrm{equal} \subseteq \bigcup_{a < b} J_{ab}.
\end{equation}
We claim that $|J_{a,b}|\leq Cd^{M+N-1}$. To see this, note that a pair $(x,z)$ has $M+N$ degrees of freedom ($M$ from $x \in [d]^{M}$, and $N$ from $z \in [d]^{N}$). 
Moreover, denote by $m_{a,b}$ the number of shared coordinates between $I_a(x,z)$ and $I_b(x,z)$. Then, we have $M +N -m_{a,b}$ degrees of freedom. Let's bound this quantity. 

If we look at the function $F$ as $F(A^1,A^2, A^3, A^4)$, then we have that $m_{1,2} = r_1$ and $m_{3,4} = r_2$. The final $r_3$ contraction further identifies coordinates as a partition: 
\begin{equation}
    r_3 = s_{13} + s_{14} +s_{23} + s_{24},
\end{equation}
where $s_{ab}$ denotes the number of coordinates that are shared after the contraction $r_3$. Since $r_3 >1$, we have that one of this terms has to be at least $1$. 

Note that, since initially the different blocks of the $r_3$ contraction were not sharing coordinates, we have that $m_{13}=s_{1,3}$, $m_{1,4}=s_{1,4}$, $m_{2,3}=s_{2,3}$ and $m_{2,4} = s_{2,4}$. Since $r_1, r_2 \in [q-1]\cup \{ 0\}$, we have that $\max(m_{12}, m_{23})\leq q-1$. Moreover, a cross contraction $s_{a,b}$ cannot be more than the indices that were already contracted in $a$ or $b$, so
\begin{equation}
    s_{a,b} \leq q- \min(r_1,r_2) \leq q-1.
\end{equation}
Then we have that $m_{a,b} \leq q-1$ for all $a,b \in [4]$, and consequently, $q-m_{a,b} \geq 1$, and therefore, we cannot have more that $N+ M - 1$ degrees of freedom, which by definition is at most $4q-1$ degrees of freedom. As a consequence, $|J_{a,b}|\leq Cd^{M+N-1}$. 

With this, we can finally bound $\| F^\mathrm{equal} \|^2$. We have:
\begin{align}
    \EE \left [  \| F^\mathrm{equal}\|^{2}\right ] & = \sum_{x\in[d]^{M}}\sum_{z,z \in [d]^{N}} \mathbf{1}_{\pi(x,z) \not = \hat{\pi}} \EE\left [\prod_{\ell = 1}^{4} A_{I_{\ell}(x,z)} \prod_{\ell = 1}^{4} A_{I_{\ell}(x,z')} \right ] \\
    & \leq \dfrac{1}{d^{4q}} \left | J_\mathrm{equal}\right | \\
    & \leq \dfrac{1}{d^{4q}} \sum_{a<b, a,b \in[4]} |J_{a,b} |,
\end{align}
but as we just saw, $|J_{a,b} | \leq d^{M+N-1}$, so we conclude:
\begin{equation}
    EE \left [  \| F^\mathrm{equal}\|^{2}\right ] \leq \dfrac{C}{d},
\end{equation}
and we conclude the proof. 
\end{proof}

\begin{lemma}
    Let $A,B \in (\RR^{d})^{\odot q}$ be two distinct symmetric tensors with independent gaussian centered entries with variance $d^{-q}$. Let $r,r' \in [q-1]\cup\{0\}$, and let $p \in [2q - 2\max(r_1,r_2)]$. 
    Then:
    \[
    \EE \left [ \| (A \otimes_{r_1} A) \otimes_{r_3} (A \otimes_{r_2} A)\|^{2}\right ] = O \left( \dfrac{1}{d}\right ).
    \]
\label{lemma:double_contraction_expectation_cross}
\end{lemma}

\begin{proof}
    The proof is analogous to one of \Cref{lemma:double_contraction_expectation}.
\end{proof}

\end{document}